\documentclass[journal]{IEEEtran}
\ifCLASSINFOpdf
\else
\fi
\newif\ifextended
\extendedfalse

\usepackage{clipboard}

\usepackage{algorithmic}
\usepackage{multirow}
\usepackage{soul}
\usepackage{subfig}
\usepackage{makecell}

\usepackage[utf8]{inputenc} 
\usepackage[T1]{fontenc}    
\usepackage{hyperref}       
\hypersetup{
 colorlinks,
 linkcolor={blue!100!black},
 citecolor={blue!100!black},
 urlcolor={blue!80!black}
}
\usepackage{url}            
\usepackage{booktabs}       
\usepackage{amsfonts}       
\usepackage{nicefrac}       
\usepackage{microtype}      
\usepackage{xcolor}         

\usepackage{amssymb,amsmath}

\usepackage{mathtools,amsthm}
\usepackage{commath}
\usepackage[ruled,norelsize]{algorithm2e}

\usepackage{enumitem}

\usepackage{tikz}

\usepackage{bbm}

\newtheorem{theorem}{Theorem}
\newtheorem{algenvironment}[algocf]{Algorithm}

\newtheorem{remark}{Remark}
\newtheorem{lemma}{Lemma}
\newtheorem{corollary}{Corollary}

\newtheorem{definition}{Definition}
\newtheorem{assumption}{Assumption}

\newtheorem{example}{Example}

\newcommand{\red}[1]{\textcolor{red}{#1}}

\newcommand{\bE}{\mathbb{E}}

\newcommand{\bR}{\mathbb{R}}

\newcommand{\cB}{\mathcal{B}}
\newcommand{\cC}{\mathcal{C}}

\newcommand{\cE}{\mathcal{E}}
\newcommand{\cF}{\mathcal{F}}

\newcommand{\cH}{\mathcal{H}}

\newcommand{\cL}{\mathcal{L}}

\newcommand{\cN}{\mathcal{N}}

\newcommand{\cR}{\mathcal{R}}
\newcommand{\cS}{\mathcal{S}}
\newcommand{\cT}{\mathcal{T}}

\newcommand{\cV}{\mathcal{V}}

\newcommand{\cX}{\mathcal{X}}

\newcommand{\cZ}{\mathcal{Z}}

\newcommand{\bolde}{\mathbf{e}}

\newcommand{\boldx}{\mathbf{x}}
\newcommand{\boldy}{\mathbf{y}}
\newcommand{\boldz}{\mathbf{z}}

\newcommand{\boldK}{\mathbf{K}}

\newcommand{\boldZ}{\mathbf{Z}}

\newcommand{\boldalpha}{\boldsymbol{\alpha}}
\newcommand{\boldbeta}{\boldsymbol{\beta}}
\newcommand{\boldgamma}{\boldsymbol{\gamma}}

\newcommand{\boldnu}{\boldsymbol{\nu}}

\newcommand{\boldrho}{\boldsymbol{\rho}}
\newcommand{\boldsigma}{\boldsymbol{\sigma}}

\DeclareMathOperator{\var}{Var}
\DeclareMathOperator{\Span}{span}

\DeclareMathOperator{\proj}{proj}

\DeclareSymbolFont{bbold}{U}{bbold}{m}{n}
\DeclareSymbolFontAlphabet{\mathbbold}{bbold}

\usepackage{colortbl}
\usepackage{makecell}  
\usepackage{booktabs}  

\usepackage{tikz}
\usepackage{amsmath}
\usetikzlibrary{shapes.geometric, arrows.meta, positioning, fit}

\tikzstyle{startstop} = [rectangle, rounded corners, minimum width=3cm, minimum height=1cm,text centered, draw=black, fill=gray!20]
\tikzstyle{io} = [rectangle, minimum width=3cm, minimum height=1cm, text centered, draw=black]
\tikzstyle{process} = [rectangle, minimum width=3cm, minimum height=1cm, text centered, draw=black]
\tikzstyle{decision} = [diamond, minimum width=3cm, minimum height=1cm, text centered, draw=black]
\tikzstyle{arrow} = [thick, -{Latex[length=2mm]}]

\begin{document}
%
\title{Gram-Schmidt Methods for Unsupervised\\Feature Extraction and Selection}
%
%
%

\author{Bahram~Yaghooti,~\IEEEmembership{Student Member,~IEEE,}
        Netanel~Raviv,~\IEEEmembership{Senior Member,~IEEE,}
        and~Bruno~Sinopoli,~\IEEEmembership{Fellow,~IEEE}
\thanks{B. Yaghooti and B. Sinopoli are with the Department of Electrical and Systems Engineering, Washington University in St. Louis, St. Louis, MO, USA (Email: byaghooti@wustl.edu; bsinopoli@wustl.edu). N. Raviv is with the Department of Computer Science and Engineering, Washington University in St. Louis, St. Louis, MO, USA (Email: netanel.raviv@wustl.edu). Parts of this work were presented in the 2024 IEEE Conference on Decision and Control (CDC)~\cite{yaghooti2024gram}, and the 2024 Annual Allerton Conference on Communication, Control, and Computing~\cite{yaghooti2024beyond}.}
}

\maketitle

\begin{abstract}
Feature extraction and selection in the presence of nonlinear dependencies among the data is a fundamental challenge in unsupervised learning. We propose using a Gram-Schmidt (GS) type orthogonalization process over function spaces to detect and map out such dependencies. Specifically, by applying the GS process over some family of functions, we construct a series of covariance matrices that can either be used to identify new large-variance directions, or to remove those dependencies from known directions. In the former case, we provide information-theoretic guarantees in terms of entropy reduction. In the latter, we provide precise conditions by which the chosen function family eliminates existing redundancy in the data. Each approach provides both a feature extraction and a feature selection algorithm. Our feature extraction methods are linear, and can be seen as natural generalization of principal component analysis (PCA). We provide experimental results for synthetic and real-world benchmark datasets which show superior performance over state-of-the-art (linear) feature extraction and selection algorithms. Surprisingly, our linear feature extraction algorithms are comparable and often outperform several important nonlinear feature extraction methods such as autoencoders, kernel PCA, and UMAP. Furthermore, one of our feature selection algorithms strictly generalizes a recent Fourier-based feature selection mechanism (Heidari \textit{et al.}, IEEE Transactions on Information Theory, 2022), yet at significantly reduced complexity.
\end{abstract}


%
\IEEEpeerreviewmaketitle

\section{Introduction}\label{sec:introduction}
\IEEEPARstart{D}{imensionality} reduction is a set of techniques used to reduce the number of features (or dimensions) in a dataset, while retaining the important information. 
In the context of unsupervised learning, the goal is to simplify the data and make it easier to analyze or visualize, while reducing the computational complexity and memory requirements of any further processing. 
The two main approaches to dimensionality reduction are \textit{feature selection}, which involves selecting a subset of the original features based on some criteria, and \textit{feature extraction}, which involves transforming the original features into a new set of features, preferably in a linear fashion, so that the extracted features capture the most important information~\cite{khalid2014survey}.

In a sense, the ultimate goal of feature extraction is to identify and remove \textit{redundancy}, that is, parts of the data which do not carry intrinsic information but rather can be determined by other parts of the data. 
It is well known that whenever such redundancy is linear, i.e., when certain features are linear functions of other ones, then PCA identifies and removes those redundant features. 
The paper at hand advances the state-of-the-art by presenting new linear feature extraction and selection techniques which are capable of removing \textit{nonlinear} redundancy from the data.
We focus on \textit{unsupervised} feature extraction, i.e., that does not take into account any label information.

The well-known PCA technique works by analyzing the covariance matrix of the data (or an empirical approximation thereof), and extracts new features called \textit{principal components}; the principal components are uncorrelated and capture the variance in the original data in descending order. 
Each principal components is a projection of the data on an eigenvector of its covariance matrix, referred to as \textit{principal direction}.
In contrast, we propose a family of Gram-Schmidt based algorithms which rely on eigenvector analysis of \textit{alternative} covariance matrices, introduced next. 

Suppose our data is sampled i.i.d from some unknown random variable~$X=(X_1,\ldots,X_d)^\intercal$ over~$\bR^d$ with distribution~$P_X$. Our algorithms begin by fixing a finite family~$\cF(\boldz)$ of linearly independent functions in some (non-random) variables~$\boldz=(z_1,z_2,\ldots)$, e.g., low-degree polynomials.
In an iterative process, the variables~$z_i$ are substituted one-by-one with linear projections of~$X$ (i.e., of the form~$ X^\intercal \boldnu$ for some~$\boldnu\in\bR^d$), effectively turning those (ordinary) functions into random ones. 
Then, after applying the \textit{Gram-Schmidt} (GS) process over the resulting random functions, (Section~\ref{section:GS}), we subtract from~$X$ its projections on those functions, 
thereby creating new data distributions with less redundancy.
The covariance matrices of these new data distributions are the \textit{alternative covariance matrices} mentioned earlier. 
Note that this approach requires no particular assumption regarding~$\cF(\boldz)$ or~$P_X$, other than the assumption that the functions in the GS process belong to~$\cL^2(P_X)$, i.e., have finite variance when computed according to~$P_X$~\cite{luenberger1997optimization}.

In the context of feature extraction, eigen-analysis of these alternative covariance matrices enables to either define new high variance directions (Section~\ref{section:GFR}), or remove nonlinear redundancy from the principal components (Section~\ref{section:GCA}).
In the context of feature selection, similar analysis of the main diagonal of those matrices enables two different feature selection algorithms which either select high-variance features (Section~\ref{section:GFS}) or remove nonlinear redundancy from the given features (Section~\ref{section:GFA}).

In more detail, the Gram-Schmidt process is a classic algebraic technique to convert a given basis of a subspace to an orthonormal basis of the same subspace, i.e., where all vectors are of length~$1$ and any pair of vectors are orthogonal. 
We employ a well-known generalization of GS to function spaces, where vectors are replaced by $\cL^2(P_X)$ functions, i.e., where the inner product between two functions is the expectation of their product, computed according to $P_X$. 

The use of GS in our algorithms results in a new set of orthonormal $\cL^2(P_X)$ functions. 
The orthonormality of these functions enables to define new random variables~$d_j(X)$---one at each iteration~$j$ of the algorithm---by subtracting from~$X$ its projections on those orthonormal functions. 
This amounts to ``nonlinear redundancy removal,'' in the sense that these~$d_j$'s capture whatever remains in~$X$ should one set to zero the part of the redundancy which can be described by the orthonormal functions.

As mentioned earlier, those~$d_j$'s give rise to alternative covariance matrices~$\Sigma_{j+1}=\bE[d_j(X)d_j(X)^\intercal]$. 
Standard eigenvector analysis over~$\Sigma_{j+1}$ provides insights into choosing the next feature to be extracted or selected.
The extraction/selection process stops once sufficiently many features have been extracted/selected, or a maximum variance threshold has been reached. 
This general framework, which is illustrated in Figure~\ref{figure:scheme}, is manifested in two complementing approaches that are described next.

In the first approach, we provide algorithms with entropy reduction guarantees that apply to any data distribution over a finite-alphabet.
To this end, at step~$j$ we extract the largest eigenvector\footnote{I.e., the eigenvector of largest eigenvalue in absolute value.} of~$\Sigma_j$ (in the feature extraction case), or select the most variant feature according to~$\Sigma_j$ (in the feature selection case).
To prove these entropy reduction guarantees, we borrow ideas from a recently proposed Fourier-based feature selection method due to~\cite{heidari2022sufficiently}, and show that the conditional entropy~$H(X | Z)$, where~$Z$ is the resulting extracted/selected random variable, is bounded by a function of the dimension and a threshold parameter. 
We term these algorithms \textit{Gram-Schmidt Functional Reduction} (GFR, feature extraction), and~\textit{Gram-Schmidt Functional Selection} (GFS, feature selection).

In the second approach, we provide algorithms with clear guarantees of their ability to \textit{eliminate redundancy}. 
Roughly speaking, a ``redundancy'' is as an equation that is satisfied, or approximately satisfied, by every datapoint drawn from~$X$. 
Each such redundancy implies a functional connection between the features/components of the data, and by ``eliminating redundancy'' we mean that one of those features/components is not extracted/selected since it can be determined by other features/components. 
This paradigm should be preferred in cases where prior knowledge regarding the structure of the redundancy is available, even if not exact. 
In particular, knowing a priori that some linear mapping of the data (e.g., the principal components or the features themselves) has redundancy, and that the redundancy can be approximated by some function family, can provide strong guarantees on the ability to remove that redundancy.

To implement the second approach, we employ similar GS-type orthogonalization, but instead of choosing new directions/features according to the~$\Sigma_j$'s as in the previous approach, we first subtract all low-variance directions/features from the data, and then choose variance maximizing directions/features. 
This subtraction process enables a clearer analysis of the effect the algorithms have on the redundancy structures in the data. 
We term these algorithms~\textit{Gram-Schmidt Functional Selection} (GFS, for eliminating redundancy from the principal components) and \textit{Gram-Schmidt Feature Analysis} (GFA, for eliminating redundancy from the features themselves).

Finally, the complexity of all algorithms is roughly quadratic in the size of the chosen function family (which can be determined by the user), linear in the size of the dataset, and at most cubic in the data dimension. We also comment that GFS strictly generalizes the recently proposed Unsupervised Fourier Feature Selection (UFFS) method \cite{heidari2022sufficiently,heidari2021finding} at significantly reduced complexity, and the details are given in the appendix. We summarize all our contributions in the following subsection.

\begin{remark}
    Roughly speaking, our approach falls under the broad category of ``maximum variance pursuit.'' 
    Algorithms in this family, which includes PCA and kernel PCA, scan the feature space in search of projections which maximize some form of variance.
    Indeed, our GFR algorithm can be seen as a strict generalization of PCA, in the sense it specifies to PCA if one chooses~$\cF(\boldz)=\{z_1,\ldots,z_d\}$.
    Kernel PCA, however, is a nonlinear method, whereas all our methods are linear. 
    A comprehensive literature review, including a detailed comparison to PCA and its variants, as well as to other state-of-the-art linear and nonlinear methods, is given in Section~\ref{section:literature}. In Appendix~\ref{section:GFR-vs-KPCA}, a detailed explanation of the underlying mechanisms and procedural distinctions between GFR and kernel PCA is provided, with particular emphasis on their differences during both the training and inference phases.
\end{remark}

\tikzset{every picture/.style={line width=0.75pt}} 
\begin{figure*}
    \centering
    \begin{tikzpicture}[x=0.75pt,y=0.75pt,yscale=-1,xscale=1]
\draw    (20,20) .. controls (59.42,20.28) and (25.63,56.14) .. (67.35,59.82) ;
\draw [shift={(70,60)}, rotate = 182.98] [fill={rgb, 255:red, 0; green, 0; blue, 0 }  ][line width=0.08]  [draw opacity=0] (8.93,-4.29) -- (0,0) -- (8.93,4.29) -- cycle    ;
\draw    (20,100) .. controls (59.42,100.28) and (25.63,59.31) .. (67.35,59.91) ;
\draw [shift={(70,60)}, rotate = 182.98] [fill={rgb, 255:red, 0; green, 0; blue, 0 }  ][line width=0.08]  [draw opacity=0] (8.93,-4.29) -- (0,0) -- (8.93,4.29) -- cycle    ;
\draw    (90,60) -- (167,60) ;
\draw [shift={(170,60)}, rotate = 180] [fill={rgb, 255:red, 0; green, 0; blue, 0 }  ][line width=0.08]  [draw opacity=0] (8.93,-4.29) -- (0,0) -- (8.93,4.29) -- cycle    ;
\draw    (90,70) .. controls (109.16,99.68) and (139.64,99.63) .. (168.25,71.74) ;
\draw [shift={(170,70)}, rotate = 134.51] [fill={rgb, 255:red, 0; green, 0; blue, 0 }  ][line width=0.08]  [draw opacity=0] (8.93,-4.29) -- (0,0) -- (8.93,4.29) -- cycle    ;
\draw    (190,60) -- (267,60) ;
\draw [shift={(270,60)}, rotate = 180] [fill={rgb, 255:red, 0; green, 0; blue, 0 }  ][line width=0.08]  [draw opacity=0] (8.93,-4.29) -- (0,0) -- (8.93,4.29) -- cycle    ;
\draw    (190,70) .. controls (209.16,99.68) and (239.64,99.63) .. (268.25,71.74) ;
\draw [shift={(270,70)}, rotate = 134.51] [fill={rgb, 255:red, 0; green, 0; blue, 0 }  ][line width=0.08]  [draw opacity=0] (8.93,-4.29) -- (0,0) -- (8.93,4.29) -- cycle    ;
\draw    (80,70) -- (80,147) ;
\draw [shift={(80,150)}, rotate = 270] [fill={rgb, 255:red, 0; green, 0; blue, 0 }  ][line width=0.08]  [draw opacity=0] (8.93,-4.29) -- (0,0) -- (8.93,4.29) -- cycle    ;
\draw    (180,70) -- (180,147) ;
\draw [shift={(180,150)}, rotate = 270] [fill={rgb, 255:red, 0; green, 0; blue, 0 }  ][line width=0.08]  [draw opacity=0] (8.93,-4.29) -- (0,0) -- (8.93,4.29) -- cycle    ;
\draw  [dash pattern={on 4.5pt off 4.5pt}]  (290,60) -- (367,60) ;
\draw [shift={(370,60)}, rotate = 180] [fill={rgb, 255:red, 0; green, 0; blue, 0 }  ][line width=0.08]  [draw opacity=0] (8.93,-4.29) -- (0,0) -- (8.93,4.29) -- cycle    ;
\draw  [dash pattern={on 4.5pt off 4.5pt}]  (290,70) .. controls (309.16,99.68) and (339.64,99.63) .. (368.25,71.74) ;
\draw [shift={(370,70)}, rotate = 134.51] [fill={rgb, 255:red, 0; green, 0; blue, 0 }  ][line width=0.08]  [draw opacity=0] (8.93,-4.29) -- (0,0) -- (8.93,4.29) -- cycle    ;
\draw    (280,70) -- (280,147) ;
\draw [shift={(280,150)}, rotate = 270] [fill={rgb, 255:red, 0; green, 0; blue, 0 }  ][line width=0.08]  [draw opacity=0] (8.93,-4.29) -- (0,0) -- (8.93,4.29) -- cycle    ;
\draw    (380,70) -- (380,147) ;
\draw [shift={(380,150)}, rotate = 270] [fill={rgb, 255:red, 0; green, 0; blue, 0 }  ][line width=0.08]  [draw opacity=0] (8.93,-4.29) -- (0,0) -- (8.93,4.29) -- cycle    ;
\draw    (446.17,99.85) .. controls (411.48,97.12) and (445.38,62.26) .. (400,59.49) ;
\draw [shift={(450,100)}, rotate = 180.99] [fill={rgb, 255:red, 0; green, 0; blue, 0 }  ][line width=0.08]  [draw opacity=0] (8.93,-4.29) -- (0,0) -- (8.93,4.29) -- cycle    ;
\draw    (446.99,20.06) .. controls (412.25,21.85) and (444.38,62.26) .. (400,59.49) ;
\draw [shift={(450,20.01)}, rotate = 180.99] [fill={rgb, 255:red, 0; green, 0; blue, 0 }  ][line width=0.08]  [draw opacity=0] (8.93,-4.29) -- (0,0) -- (8.93,4.29) -- cycle    ;

\draw (3.58,20.2) node    {$\cF(\boldz)$};
\draw (13.22,100.2) node    {$X$};
\draw (80.98,61.19) node    {$d_{0}$};
\draw (127.13,41.99) node    {$Z_{1} =X^{\intercal} \boldnu_{1}$};
\draw (122.35,113.94) node  [font=\footnotesize] [align=left] {\begin{minipage}[lt]{33.11pt}\setlength\topsep{0pt}
Orthogonalize\\\& Subtract

\end{minipage}};
\draw (179.98,61.19) node    {$d_{1}$};
\draw (227.13,41.99) node    {$Z_{2} =X^{\intercal} \boldnu_{2}$};
\draw (217.45,113.94) node  [font=\footnotesize] [align=left] {\begin{minipage}[lt]{35.98pt}\setlength\topsep{0pt}
Orthogonalize\\\& Subtract
\end{minipage}
};
\draw (279.98,61.19) node    {$d_{2}$};
\draw (89.75,161.35) node  [font=\footnotesize]  {$\Sigma _{1} =\mathbb{E}[ XX^\intercal]$};
\draw (386.63,61.19) node    {$d_{m-1}$};
\draw (464.63,56.77) node    {$ \begin{array}{l}
\boldnu_1\\
\boldnu_2\\
~\vdots\\
\boldnu_m
\end{array}$};
\draw (182.74,162.85) node  [font=\footnotesize]  {$\Sigma _{2} =\mathbb{E}[ d_{1} d_{1}^{\intercal}]$};
\draw (283.74,162.85) node  [font=\footnotesize]  {$\Sigma _{3} =\mathbb{E}[ d_{2} d_{2}^{\intercal}]$};
\draw (386.99,162.35) node  [font=\footnotesize]  {$\Sigma _{m} =\mathbb{E}[ d_{m-1} d_{m-1}^\intercal]$};
\end{tikzpicture}
    \caption{A schematic description of all algorithms in this paper.
    The algorithm receives a function family~$\cF(\boldz)$ in variables~$\boldz=(z_1,z_2,\ldots)$ and data sampled i.i.d from an unknown random variable~$X$.
    The algorithm iteratively substitutes the (non-random) variables~$z_i$ in~$f(\boldz)\in\cF (\boldz)$ with random variables~$Z_i$, which are linear functions of~$X$, orthogonalizes the resulting $\cL^2(P_X)$ functions, and subtracts from~$X$ its projections on them. 
    This creates new random variables~$d_j(X)$, and the~$\Sigma_j$'s are their covariance matrices. 
    Different algorithms vary in the way the~$\boldnu_i$'s are specified: 
    either as the largest eigenvectors of the~$\Sigma_i$'s (Gram-Schmidt Functional Reduction, GFR, Section~\ref{section:GFR}) or as high-variance principal directions (Gram-Schmidt Component Analysis, GCA, Section~\ref{section:GCA})) in the feature extraction case, or their unit-vector counterparts in Gram-Schmidt Functional Selection (GFS, Section~\ref{section:GFS}) and 
    Gram-Schmidt Feature Analysis (GFA, Section~\ref{section:GFA}) in the feature selection case. The output of the algorithm is~$\boldnu_1,\ldots,\boldnu_m$.}\label{figure:scheme}
\end{figure*}

\subsection{Our contributions:}
\begin{itemize}
    \item We present \textit{Gram-Schmidt Functional Reduction} (GFR), a \textbf{linear feature extraction} technique that identifies directions of high-variance (different from the principal ones). We provide information-theoretic guarantees of bounded conditional entropy for the case of discrete data distributions (Section~\ref{section:GFR}). 
    \item We present \textit{Gram-Schmidt Component Analysis} (GCA), a \textbf{linear feature extraction} technique that identifies and removes nonlinear dependencies among the principal components. We provide conditions on the chosen function family by which GCA can eliminate redundancy (Section~\ref{section:GCA}).
    \item We present \textit{Gram-Schmidt Functional Selection} (GFS), a \textbf{feature selection} technique which inherits similar structure and similar information-theoretic guarantees from GFR (Section~\ref{section:GFS}). 
    \item We present \textit{Gram-Schmidt Feature Analysis} (GFA), a \textbf{feature selection} technique which inherits similar formal guarantees from GCA (Section~\ref{section:GFA}).
    \item We prove that GFS generalizes the Unsupervised Fourier Feature Selection (UFFS) algorithm due to~\cite{heidari2022sufficiently}, yet at significantly lower complexity. Specifically, UFFS arises as a special case of GFS when a certain function family is chosen (Appendix~\ref{section:previousFourier}).
    \item We support our theoretical findings and show the applicability of our techniques by performing broad experiments. The experiments validate our findings in several aspects:
    \begin{itemize}
        \item GFR achieves similar residual variance\footnote{I.e., the maximum amount of variance that is left in the data distribution after extracting said features. This is captured in PCA by the~$(m+1)$'th largest eigenvalue of the covariance matrix when extracting~$m$ features.} using significantly fewer extracted features in comparison to PCA, by almost an order of magnitude in certain cases. 
        \item GFR preserves classification accuracy in both synthetic and real-world datasets, in comparison to PCA and state-of-the-art linear methods. It also performs comparatively well or better than some powerful nonlinear methods such as autoencoders, kernel PCA, and Uniform Manifold Approximation and Projection (UMAP). 
        \item GFS shows superior performance in classification accuracy on real-world and synthetic datasets in comparison to state-of-the-art feature selection algorithms, including UFFS~\cite{heidari2022sufficiently}. 
        \item The theoretical guarantees of GCA and GFA are shown to hold over synthetic datasets.        
    \end{itemize}
    The implementations of our algorithms are simple, and source code is available at~\url{https://github.com/byaghooti/Gram_schmidt_feature_extraction}.
\end{itemize}

\section{Notations and Preliminaries}
\subsection{Notations}
We denote vector random variables using capital letters~$X,Y,Z$, etc., and in particular we let~$X=(X_1,\ldots,X_d)$ be the (unknown, zero mean) random variable over~$\bR^d$ (i.e., each~$X_i$ is over~$\bR$) from which the data is sampled i.i.d. We use~$\boldx=(x_1,\ldots,x_d), \boldy=(y_1,\ldots,y_d)$, etc. to denote ordinary (that is, non-random) variables, and~$\boldalpha=(\alpha_1,\ldots,\alpha_d), \boldbeta=(\beta_1,\ldots,\beta_d)$ etc. to denote vectors of scalars (i.e., $\alpha_i,\beta_i\in\bR$). Further, for~$\cS\subseteq[d]$ (where~$[d]\triangleq\{1,2,\ldots,d\}$), we use~$X_\cS$ (resp.~$\boldx_\cS$, etc.) to denote a vector of length~$|\cS|$ which contains the entries of~$X$ (resp.~$\boldx$, etc.) that are indexed by~$\cS$.

We denote by~$\cF(\boldz)$ a generic function family in the (ordinary) variables~$z_1,\ldots,z_d$, which are linearly independent over~$\bR$. For~$\cS\subseteq[d]$ we let~$\cF(\boldz_\cS)$ be the subset of~$\cF$ which contains all functions in~$\cF$ which depend only on a subset of the variables in~$\boldz_\cS$. Our Gram-Schmidt approach strongly relies on substituting the ordinary variables~$z_i$ by random ones~$Z_i$, and hence for~$\cS=\{\sigma_1,\ldots,\sigma_{|\cS|}\}\subseteq[d]$ (with~$\sigma_1<\ldots<\sigma_{|\cS|}$) and a vector of random variables~$Z=(Z_1,\ldots,Z_d)$, we denote by~$\cF(Z_{\cS})$ the result of substituting~$z_i\leftarrow Z_{\sigma_i}$ for every~$i\in[|\cS|]$ in~$\cF(\boldz_{|\cS|})$.

\begin{example}
    Let~$d=3$ and~$\cF(\boldz)=\{z_1, z_2, z_3, z_1z_2, z_1z_2z_3\}$. For~$\cS=[2]$ we have that~$\cF(\boldz_{\cS})=\{z_1,z_2,z_1z_2\}$, and for~$\cT=\{1,3\}$ we have~$\cF(Z_{\cT})=\{Z_1,Z_3,Z_1Z_3\}$.
\end{example}

All function families~$\cF(\boldz)$ we use in this paper contain all the ``singleton'' functions, i.e.,~$\{f_i(\boldz)=z_i\}_{i=1}^d\subseteq\cF(\boldz)$. Our algorithms require no particular assumption on~$\cF(\boldz)$ or on~$P_X$ other than~$\cF(Z_\cS)\subseteq\cL^2(P_X)$ for every~$\cS\subseteq [d]$, i.e., that the functions which result from the substitution have finite variance,
and as such belongs to a Hilbert space in which the inner product between two functions~$f$ and~$g$ is~$\bE[fg]$. We note that all random variables in this paper are some functions of~$X$, the random variable from which the data is sampled i.i.d, and all expectations are computed with respect to $P_X$. 

\begin{remark}\label{remark:stupidReviewerHere}
    The random variables we use in this paper (including~$X$) need not be known. We manipulate random variables as symbolic expressions, and only use them in the context of computing expectations. In practice, these expectations are approximated using empirical means. For example, we may manipulate~$X$ to obtain a function~$\ell(X)=X_1+X_2^2$, without the need to know how~$X$ is distributed. To avoid notational clutter, expressions of the form~$\bE[\ell(X)]$ are used in lieu of the empirical approximation~$\bE[\ell(X)]\approx \sum_{i=1}^N \ell(\boldalpha_i)=\sum_{i=1}^N(\alpha_{i,1}+\alpha_{i,2}^2)$, where~$\{\boldalpha_i\}_{i=1}^N$ is the dataset at hand. 
\end{remark}

As mentioned previously, our algorithms rely on the Gram-Schmidt process in a Hilbert space. To this end, we denote by~$\hat{\cF}(Z_{\cS})$ the result of applying the Gram-Schmidt process over a function family~$\cF(Z_{\cS})$; notice that~$\Span\hat{\cF}(Z_{\cS})=\Span\cF(Z_{\cS})$. The full details of this Gram-Schmidt process are given in the following section. Finally, we use the algorithmic notation~$a\leftarrow b$ to denote that the expression~$b$ is computed, and then assigned into the variable~$a$.

\begin{remark}
    All our proposed algorithms can be implemented using purely algebraic operations. However, we chose to present them within a probabilistic framework in order to convey the core ideas more clearly, as well as to establish the theoretical guarantees more easily. Nevertheless, to support practical implementation, algebraic representations of all proposed algorithms, including orthogonalization, GFR, GCA, GFS, and GFA, are provided in Appendix~\ref{section:algebraic-representations}.
\end{remark}

\subsection{Gram-Schmidt Orthogonalization over Function Spaces}\label{section:GS}
In this section we describe the core GS orthogonalization process in a function space on which all our algorithms are based. 
As mentioned earlier, we begin by fixing a function family~$\cF(\boldz)$ in variables~$\boldz=(z_1,\ldots,z_d)$; other than being square integrable according to the data distribution we do not make any particular assumptions on those functions, but in our experimental sections we focus on various polynomial functions.

Our algorithms (Sections~\ref{section:GFR}-\ref{section:GFA}) apply in an iterative manner, where in iteration~$j$ we choose some new feature~$Z_{\ell_j}$ to select/extract, substitute~$z_j\leftarrow Z_{\ell_j}$ in the functions of~$\cF(\boldz_{[j]})$, and perform GS orthogonalization. 
To this end, Algorithm~\ref{algorithm:orthogonalize} operates at each iteration~$j$ of either of our algorithms by receiving an already orthogonalized~$\hat{\cF}(Z_{\ell_1},\ldots,Z_{\ell_{j-1}})$, new functions in~$\cF(\boldz_{[j]})\setminus\cF(\boldz_{[j-1]})$, and a new random variable~$Z_{\ell_j}$. 

Then, for each member of~$\cF(\boldz_{[j]})\setminus\cF(\boldz_{[j-1]})$, the algorithm substitutes~$z_i\leftarrow Z_{\ell_i}$ for every~$i\in[j]$ (Line~\ref{line:substitute_nonrandom_variables}), subtracts its projection on already-orthogonalized functions (Line~\ref{line:subtraction_of_projections}), normalizes (Line~\ref{line:normalization}), and adds the result to~$\cT$ (Line~\ref{line:add}). 
After the subtraction process we define a new (vector) random variable~$d_j(X)$ from~$X$ (Line~\ref{line:subtraction_of_data_distribution}), whose covariance matrix is computed (Line~\ref{line:covariance_matrix}) and returned as output. An equivalent explicit algebraic representation of this orthogonalization process, referred to as \textit{Algebraic-Orthogonalize}, is provided in Algorithm~\ref{algorithm:orthogonalize_algebraic} in Appendix~\ref{section:algebraic-orthogonalize}.

\begin{algorithm}[!h]
\caption{$\text{Orthogonalize}(\hat{\cF}(Z_{\ell_1},\ldots,Z_{\ell_{j-1}}),\cF(\boldz_{[j]})\setminus\cF(\boldz_{[j-1]}),{\{Z_{\ell_i}\}_{i=1}^j})$} \label{algorithm:orthogonalize}
\begin{algorithmic}[1]
\STATE {\bfseries Input:} 
\begin{itemize}
    \item $\hat{\cF}(Z_{\ell_1},\ldots,Z_{\ell_{j-1}})$: an orthogonalized function family in random variables~$Z_{\ell_1},\ldots,Z_{\ell_{j-1}}$.
    \item $\cF(\boldz_{[j]})\setminus\cF(\boldz_{[j-1]})$: all the functions in~$\cF(\boldz)$ \\ which depend on~$\boldz_{[j]}$ but not only on~$\boldz_{[j-1]}$.
    \item ${\{Z_{\ell_i}\}_{i=1}^j}$: the random variables~$Z_{\ell_1},\ldots,Z_{\ell_{j-1}}$ which appear in $\hat{\cF}(Z_{\ell_1},\ldots,Z_{\ell_{j-1}})$, and a new random variable~$Z_{\ell_j}$.
\end{itemize}
\STATE {\bfseries Output:} Orthogonalized function family \\
$\hat{\cF}(Z_{\ell_1},\ldots,Z_{\ell_j})$, and a covariance matrix~$\Sigma_{j+1}$.
\STATE {\bfseries Initialize:} $\cT\leftarrow\hat{\cF}(Z_{\ell_1},\ldots,Z_{\ell_{j-1}})$.
\STATE Denote $\cF(\boldz_{[j]})\setminus\cF(\boldz_{[j-1]}))\triangleq\{f_{1}(\boldz),\ldots,f_{\ell}(\boldz)\}$, where~$f_1(\boldz)=z_j$.\label{line:orthogonalization_order}
\FOR{$k\gets 1$ {\bfseries to} $\ell$}
\STATE $g(Z_{\ell_1},\ldots,Z_{\ell_j})\leftarrow f_k(Z_{\ell_1},\ldots,Z_{\ell_j})$ \\\quad(create~$g$ from~$f_k$ by substituting~$z_i\leftarrow Z_{\ell_i}$ for~$i\in[j]$.)\label{line:substitute_nonrandom_variables}
\STATE $g(Z_{\ell_1},\ldots,Z_{\ell_j})\leftarrow g(Z_{\ell_1},\ldots,Z_{\ell_j})-\sum_{f\in \cT}\bE[g f]f$\label{line:subtraction_of_projections}\\\quad(orthogonalize~$g$ with respect to~$\cT$.)
\STATE $g(Z_{\ell_1},\ldots,Z_{\ell_j})\leftarrow\frac{g(Z_{\ell_1},\ldots,Z_{\ell_j})}{\sqrt{\bE[g(Z_{\ell_1},\ldots,Z_{\ell_j})^2]}}$\\\quad (normalize~$g$.) \label{line:normalization}
\STATE $\cT\leftarrow\cT\cup\{g(Z_{\ell_1},\ldots,Z_{\ell_j})\}$\\\quad(add~$g$ to~$\cT$)\label{line:add}
\ENDFOR 
\STATE Define~$\hat{\cF}(Z_{\ell_1},\ldots,Z_{\ell_j})= \cT$.
\STATE Define~$d_j(X) = X - \sum_{f\in \cT}\bE[X f]f$.\label{line:subtraction_of_data_distribution}
\STATE Define~$\Sigma_{j+1} = \mathbb{E}[d_j(X) d_j^\intercal(X)]$.\label{line:covariance_matrix}
\STATE Return~$\Sigma_{j+1}$, $\cT$.\label{line:return}
\end{algorithmic}
\end{algorithm}

\section{Gram-Schmidt Functional Reduction (GFR)}\label{section:GFR}
In Algorithm~\ref{algorithm:GFR} below we propose \textit{Gram-Schmidt Functional Reduction} (GFR), which uses the alternative covariance matrices~$\Sigma_j$ in order to extract new high variance directions. 
In each iteration~$j$, GFR begins by identifying the largest eigenvector~$\boldnu_j$ of~$\Sigma_j$ (this would be the highest variance principal direction of~$d_{j-1}(X)$, computed in Line~\ref{line:subtraction_of_data_distribution} of Algorithm~\ref{algorithm:orthogonalize}). 
If the variance of this direction is less than~$\epsilon^2$, the algorithm stops, and otherwise it continues by specifying the next random variable~$Z_j$ as~$Z_j=X^\intercal\boldnu_j$, which enables the next call to Orthogonalize (Algorithm~\ref{algorithm:orthogonalize}). 
An equivalent and fully algebraic version of this algorithm, referred to as \textit{Algebraic-GFR}, is provided in Algorithm~\ref{algorithm:GFR_algebraic} in Appendix~\ref{section:algebraic-GFR-GCA}.

\begin{algorithm}[!h]
\caption{Gram-Schmidt Functional Reduction (GFR)}\label{algorithm:GFR}
\begin{algorithmic}[1]
\STATE {\bfseries Input:} A function family~$\cF(\boldz)$ and a threshold $\epsilon>0$.
\STATE {\bfseries Output:} Extracted directions~$\boldnu_1,\ldots,\boldnu_m$ ($m$ is a \\
varying number that depends on~$\cF(\boldz),\epsilon$, and~$X$).
\STATE {\bfseries Initialize:} $\Sigma_1=\bE[XX^\intercal]$.
\FOR{$j\gets1$ {\bfseries to} $d$}
\STATE Let~$\boldnu_j$ be the largest unit-norm eigenvector of $\Sigma_j$, \\ 
i.e., $\boldnu_j=\arg\max_{\Vert\boldnu\Vert=1}\boldnu^\intercal\Sigma_j\boldnu$. \label{line:eigenvector_maximization}
\IF{$\boldnu_j^\intercal\Sigma_j\boldnu_j\le\epsilon^2$}
\STATE break.
\ELSE
\STATE Define~$Z_j=X^\intercal\boldnu_j$.
\ENDIF
\STATE $\Sigma_{j+1},\hat{\cF}(Z_1,\ldots,Z_j)=\text{Orthogonalize}(\hat{\cF}(Z_1,\ldots,Z_{j-1}),\cF(\boldz_{[j]})\setminus\cF(\boldz_{[j-1]}),\{Z_i\}_{i=1}^j)$.
\ENDFOR
\end{algorithmic}
\end{algorithm}

Next, we state the information-theoretic guarantees of GFR. 
While GFR shows significant gains over state-of-the-art linear (and some nonlinear) methods with real-world datasets, its information-theoretic guarantees require the data distribution~$X$ to be over a finite alphabet. 
In our experiments (Section~\ref{section:experiments}) it is shown that setting~$\cF(\boldz)$ to be multilinear polynomials of low degree reduces the residual variance in real-world datasets by up to an order of magnitude in comparison with PCA.

\begin{theorem}\label{theorem:GFR}
    Suppose~$X$ is over a discrete domain~$\cX^d$, for some~$\cX\subseteq\bR$, and let~$Z=(Z_1,\ldots,Z_m)=(X^\intercal\boldnu_1,\ldots,X^\intercal\boldnu_m)$, where~$\boldnu_1,\ldots,\boldnu_m$ is the output of GFR whose input is a given~$\epsilon$ and~$\cF(\boldz)$. Then~$H(X\vert Z)\le dO(\epsilon)$.
\end{theorem}

The proof of Theorem~\ref{theorem:GFR} is similar to a proof from~\cite{heidari2022sufficiently}, but requires several additional statements.
These statements essentially show that once GFR stops, the variance of the data in \textit{all potential subsequent directions} is smaller than~$\epsilon^2$.
As a result, a tradeoff between~$\cF(\boldz)$ and~$\epsilon$ is revealed---taking a more powerful~$\cF(\boldz)$ (e.g., higher degree polynomials) or a smaller~$\epsilon$ would decrease the number~$m$ of extracted features.
To prove this statement, we complete the vectors~$\boldnu_1,\ldots,\boldnu_m$ with~$\boldnu_{m+1},\ldots,\boldnu_d$ as follows.
\begin{definition}\label{definition:orthogonalCompletion}
    Let~$\boldnu_{1},\ldots,\boldnu_m$ be the output of GFR with input~$\epsilon>0$ and~$\cF(\boldz)$. Let~$\boldnu_1,\ldots,\boldnu_{m'}$ be the output of GFR with the same~$\cF(\boldz)$ and~$\epsilon=0$, and let~$\Sigma_1,\ldots,\Sigma_{m'}$ be the respective covariance matrices. 
    In cases where~$m'<d$ we let~$\boldnu_{m'+1},\ldots,\boldnu_d$ be an arbitrary orthonormal completion of~$\boldnu_1,\ldots,\boldnu_{m'}$.
    In these cases we also artificially complete the orthogonalization process to produce~$\hat{\cF}(Z_1,\ldots,Z_d)$, with~$Z_j=X^\intercal\boldnu_j$ for all~$j\in[d]$, and let~$\{\Sigma_j\}_{j=1}^d$ be the respective covariance matrices.
    Finally, for~$j\in[d]$ we denote $\lambda_j=\boldnu_j^\intercal\Sigma_j\boldnu_j$.
\end{definition} 

\begin{lemma}\label{lemma:alldirections}
    For every~$j\in\{m+1,\ldots,d\}$ we have that~$\boldnu_j^\intercal\Sigma_j\boldnu_j\le\epsilon^2$.
\end{lemma}

The following technical lemma is required for the subsequent proof of Lemma~\ref{lemma:alldirections}. For brevity, in the subsequent lemmas for every~$j\in[d]$ we denote~$\hat{\cF}_j\triangleq\hat{\cF}(Z_1,\ldots,Z_j)$.

\begin{lemma}\label{lemma:GFRaux}
    In GFR, we have the following. 
    \begin{enumerate}
        \item[(a)] For~$i<j$ in~$[d]$ we have \\
        $\Sigma_j=\Sigma_i-\sum_{f\in\hat{\cF}_{j-1}\setminus\hat{\cF}_{i-1}}\bE[fX]\bE[X^\intercal f]$.
        \item[(b)] The vectors~$\boldnu_1,\ldots,\boldnu_{d}$ are orthonormal.
    \end{enumerate}
\end{lemma}
\begin{proof}[Proof of Lemma~\ref{lemma:GFRaux}.$(a)$]
    It is readily verified that~$d_{j-1}(X)$ (see Algorithm~\ref{algorithm:orthogonalize}) can be written as
    \begin{align}\label{equation:dj-1}
        d_{j-1}(X)=d_{i-1}(X)-\sum_{f\in\hat{\cF}_{j-1}\setminus\hat{\cF}_{i-1}}\bE[X f]f,
    \end{align}
    and that~$d_{i-1}(X)$ can be written as
    \begin{align}\label{equation:di-1}
        d_{i-1}(X)=X-\textstyle\sum_{f\in\hat{\cF}_{i-1}}\bE[X f]f.
    \end{align}
    Hence, it follows from~\eqref{equation:dj-1} that
    \begin{align}\label{equation:sigmajassigmai}
        \Sigma_j&=\bE[d_{j-1}(X)d_{j-1}(X)^\intercal]\nonumber\\
        &=\textstyle\bE\bigg[ \left(d_{i-1}(X)-\sum_{f\in\hat{\cF}_{j-1}\setminus\hat{\cF}_{i-1}}\bE[X f]f\right)\nonumber\\
        &\phantom{=\bE\bigg[~}\textstyle\left(d_{i-1}(X)-\sum_{f\in\hat{\cF}_{j-1}\setminus\hat{\cF}_{i-1}}\bE[Xf]f\right)^\intercal \bigg]\nonumber\\
        &= \Sigma_i-\textstyle\sum_{f\in\hat{\cF}_{j-1}\setminus\hat{\cF}_{i-1}}\bE[fd_{i-1}(X)]\bE[X^\intercal f]\nonumber\\
        &\phantom{= \Sigma_i}-\textstyle\sum_{f\in\hat{\cF}_{j-1}\setminus\hat{\cF}_{i-1}}\bE[fX]\bE[ d_{i-1}(X)^\intercal f]\nonumber\\
        &\phantom{= \Sigma_i}+\textstyle\sum_{f\in \hat{\cF}_{j-1}\setminus\hat{\cF}_{i-1}}  \bE[Xf]\bE[X^\intercal f],
    \end{align}
    where the last summand follows from the orthonormality of~$\hat{\cF}_{j-1}$. Further, it follows from~\eqref{equation:di-1} that every~$f\in\hat{\cF}_{j-1}\setminus\hat{\cF}_{i-1}$ satisfies
    \begin{align*}
        \bE[fd_{i-1}(X)]&=\bE[f\cdot(X-\textstyle\sum_{g\in\hat{\cF}_{i-1}}\bE[Xg]g)]=\bE[fX], \\
        \bE[d_{i-1}(X)^\intercal f]&=\bE[(X-\textstyle\sum_{g\in\hat{\cF}_{i-1}}\bE[Xg]g)^\intercal\cdot f]=\bE[X^\intercal f],
    \end{align*}
    and therefore~$\eqref{equation:sigmajassigmai}=\Sigma_i-\sum_{f\in\hat{\cF}_{j-1}\setminus\hat{\cF}_{i-1}}\bE[Xf]\bE[X^\intercal f]$ as required.
\end{proof}

\begin{proof}[Proof of Lemma~\ref{lemma:GFRaux}.$(b)$]
Since all~$\boldnu_i$'s are normalized it suffices to show that for any~$j\in[d]$, the vector $\boldnu_j$ is orthogonal to~$\boldnu_1,\ldots,\boldnu_{j-1}$. 
If~$j>m'$, however, the latter statement follows from Definition~\ref{definition:orthogonalCompletion}, since~$\boldnu_{m'+1},\ldots,\boldnu_{d}$ is an orthogonal completion of~$\boldnu_{1},\ldots,\boldnu_{m'}$.
If~$j\le m'$ let~$i<j$, and observe that $\boldnu_i^\intercal\Sigma_j\boldnu_j=\lambda_j\boldnu_i^\intercal\boldnu_j$. 
It follows from part~$(a)$ that 
\begin{align*}
    \Sigma_j = \Sigma_i - \textstyle\sum_{f\in\hat{\cF}_{j-1}\setminus\hat{\cF}_{i-1}}\mathbb{E}[ fX] \mathbb{E}[X^\intercal f],
\end{align*}
and hence
\begin{align}\label{equation:vtsigmaju}
    \boldnu_i^\intercal \Sigma_j \boldnu_j & = \boldnu_i^\intercal \left(\Sigma_i - \textstyle\sum_{ f\in\hat{\cF}_{j-1}\setminus\hat{\cF}_{i-1}}\mathbb{E}[  fX] \mathbb{E}[X^\intercal  f]\right) \boldnu_j \nonumber \\
    & = \lambda_i\boldnu_i^\intercal \boldnu_j - \textstyle\sum_{ f\in\hat{\cF}_{j-1}\setminus \hat{\cF}_{i-1}}\bE[  f\cdot\boldnu_i^\intercal X] \mathbb{E}[X^\intercal\boldnu_j \cdot f].
\end{align}
Let~$\Tilde{Z}_i$ and~$\hat{Z}_i$ be the orthogonalized and orthonormalized versions of~$Z_i$, respectively. By definition, we have that $\boldnu_i^\intercal X = Z_i = \Tilde{Z}_i + \sum_{ f\in\hat{\cF}_{i-1}} \mathbb{E}[Z_i  f] f$, and hence every~$ f\in\hat{\cF}_{j-1}\setminus\hat{\cF}_{i-1}$ satisfies 
\begin{align}\label{equation:fhatztilde}
\bE[ f\cdot\boldnu_i^\intercal X]&=\bE[ f\cdot(\Tilde{Z}_i+\textstyle\sum_{ g\in\hat{\cF}_{i-1}}\bE[Z_i g] g)] \nonumber\\
&=\bE[ f\Tilde{Z}_i]=
    \begin{cases}
        \Vert \Tilde{Z}_i\Vert & \mbox{if } f=\hat{Z}_i\\
        0 & \mbox{ otherwise.}
    \end{cases}
\end{align}
Therefore, 
\begin{align}\label{equation:lambdavivj}
    \eqref{equation:vtsigmaju}&=\lambda_i\boldnu_i^\intercal \boldnu_j - \textstyle\sum_{ f\in\hat{\cF}_{j-1}\setminus \hat{\cF}_{i-1}}\bE[  f\Tilde{Z}_i] \bE[X^\intercal\boldnu_j \cdot f] \nonumber\\
    &=\lambda_i\boldnu_i^\intercal\boldnu_j-\Vert\Tilde{Z}_i\Vert\bE[X^\intercal\boldnu_j\cdot\hat{Z}_i]\nonumber\\
    &=(\lambda_i\boldnu_i^\intercal-\bE[X^\intercal\Tilde{Z}_i])\boldnu_j.
\end{align}
Now observe that
\begin{align*}
    \bE[X\Tilde{Z}_i] &=  \bE[X(Z_i - \textstyle\sum_{ f\in\hat{\cF}_{i-1}} \bE[Z_i  f] f)] \nonumber\\
    &= \bE[X(\boldnu_i^\intercal X - \textstyle\sum_{ f\in\hat{\cF}_{i-1}} \bE[\boldnu_i^\intercal X  f] f)] 
    \\
    &= \bE[ (XX^\intercal) \boldnu_i] - \textstyle\sum_{ f\in\hat{\cF}_{i-1}} \bE[X  f]\bE[ X^\intercal\boldnu_i  f ]  \nonumber \\
    &= \left(\bE[ XX^\intercal ] - \textstyle\sum_{ f\in\hat{\cF}_{i-1}}  \bE[ fX ]\bE[ X^\intercal f  ]\right)\boldnu_i\nonumber\\
    &\overset{(\star)}{=}\Sigma_i\boldnu_i=\lambda_i\boldnu_i,
\end{align*}
where~$(\star)$ follows by applying part~$(a)$ on~$\Sigma_i$ and~$\Sigma_1$, and hence~$\eqref{equation:lambdavivj}=0$. Subsequently, we have~$\lambda_j\boldnu_i^\intercal\boldnu_j=\boldnu_i^\intercal\Sigma_j\boldnu_j=0$. Since~$\lambda_j\ne 0$ by Definition~\ref{definition:orthogonalCompletion}, it follows that~$\boldnu_j$ is orthogonal to~$\boldnu_i$ for all~$i\in[j-1]$, and the claim follows.
\end{proof}

\begin{proof}[Proof of Lemma~\ref{lemma:alldirections}]

    According to the stopping criterion in GFR, it follows that
\begin{align}\label{equation:conclusionsFromStopping_}
    \boldnu_j^\intercal\Sigma_j\boldnu_j&>\epsilon^2\mbox{ for all }j\in[m]\nonumber\\
    \max_{\Vert\boldnu\Vert=1}\boldnu^\intercal\Sigma_{m+1}\boldnu&\le\epsilon^2.
\end{align}
Assume for contradiction that there exists~$j\in\{m+1,\ldots,d\}$ such that~$\boldnu_j^\intercal\Sigma_j\boldnu_j>\epsilon^2$, and observe that~$j>m+1$, since otherwise the existence of~$\boldnu_{m+1}$ contradicts~\eqref{equation:conclusionsFromStopping_}. It follows from Lemma~\ref{lemma:GFRaux}.(a) that
\begin{align*}
    \Sigma_j=\Sigma_{m+1}-\sum_{ f\in\hat{\cF}_{j-1}\setminus\hat{\cF}_{m}}\bE[ fX]\bE[X^\intercal f],
\end{align*}
and therefore
\begin{align*}
    \boldnu_j^\intercal\Sigma_j\boldnu_j&\textstyle=\boldnu_j^\intercal\Sigma_{m+1}\boldnu_j-\sum_{ f\in\hat{\cF}_{j-1}\setminus\hat{\cF}_{m}}\bE[ f\cdot\boldnu_j^\intercal X]^2 \\
    &\overset{\eqref{equation:conclusionsFromStopping_}}{\le}\textstyle\epsilon^2-\sum_{ f\in\hat{\cF}_{j-1}\setminus\hat{\cF}_{m}}\bE[ f\cdot\boldnu_j^\intercal X]^2.
\end{align*}
Therefore, since~$\boldnu_j^\intercal\Sigma_j\boldnu_j>\epsilon^2$, it follows that 
\begin{align*}
    -\sum_{ f\in\hat{\cF}_{j-1}\setminus\hat{\cF}_{m}}\bE[ f\cdot\boldnu_j^\intercal X]^2>0,
\end{align*}
which is a contradiction since $\bE[ f\cdot\boldnu_j^\intercal X]^2\ge 0$ for every~$ f$.
\end{proof}

\begin{proof}[Proof of Theorem~\ref{theorem:GFR}]
    
Since the matrix whose rows are the~$\boldnu_i$'s has determinant~$1$, we have:
\begin{align}\label{equation:Hzx}
    H(X\vert Z)&=H(\{X^\intercal\boldnu_j\}_{j=1}^d\vert \{X^\intercal\boldnu_j\}_{j=1}^m)\nonumber\\
    &=H(\{X^\intercal\boldnu_j\}_{j=m+1}^d\vert \{X^\intercal\boldnu_j\}_{j=1}^m)\nonumber\\
    &\textstyle= H(Z^\bot\vert Z)= \sum_{i=m+1}^d H(Z_i\vert (Z_j)_{j=1}^{i-1}),
\end{align}
where the last equality follows from the chain rule for information entropy. Since
\begin{align}\label{equation:ziTilde}
    \Tilde{Z}_i=Z_i-\sum_{ f\in \hat{\cF}_{i-1}}\bE[Z_i f] f,
\end{align}
it follows that
\begin{align}\label{equation:Hztilde}
    \textstyle H(Z_i\vert (Z_j)_{j=1}^{i-1})&\textstyle =H(\Tilde{Z}_i+\sum_{ f\in \hat{\cF}_{i-1}}\bE[Z_i f] f\vert (Z_j)_{j=1}^{i-1})\nonumber\\
    &\overset{(a)}{=}H(\Tilde{Z}_i\vert (Z_j)_{j=1}^i)\textstyle\overset{(b)}{\le}H(\Tilde{Z}_i),
\end{align}
where~$(a)$ follows since the expression $\sum_{ f\in \hat{\cF}_{i-1}}\bE[Z_i f] f$ is uniquely determined by the variables~$Z_1,\ldots,Z_{i-1}$ (every~$ f\in\hat{\cF}_{i-1}$ is a deterministic function of~$Z_1,\ldots,Z_{i-1}$, and the coefficients~$\bE[Z_i f]$ are constants), and~$(b)$ follows since conditioning reduces entropy. Combining~\eqref{equation:Hzx} with~\eqref{equation:Hztilde} we have that~$H(X\vert Z)\le\sum_{i=m+1}^dH(\Tilde{Z}_i)$, and hence it remains to bound~$H(\Tilde{Z}_i)$ for all~$i\in\{m+1,\ldots,d\}$.

To this end let~$i\in\{m+1,\ldots,d\}$, define\footnote{To clarify the notation~$\tilde{Z}_i(\boldalpha)$ for~$\boldalpha\in\cX^d$, recall that every~$Z_i$ is a function of~$X$, and therefore so is its orthogonalized version~$\tilde{Z}_i=Z_i-\sum_{f\in\hat{\cF}_{i-1}}\bE[Z_i f]f$. Therefore, the notation~$\tilde{Z}_i(\boldalpha)$ stands for the scalar given by the substitution~$X\leftarrow\boldalpha$ in~$\tilde{Z}_i$.}~$\alpha_i\triangleq\min\{|\Tilde{Z}_i(\boldalpha)|:\boldalpha\in\cX^d, \Tilde{Z}_i(\boldalpha)\ne 0\}$ and~$\alpha_{\min}\triangleq\min\{\alpha_i\}_{i=m+1}^d$, and observe that by Markov's inequality we have
\begin{align}\label{equation:PrziTile}
    \Pr(\Tilde{Z}_i\ne 0)&=\Pr(|\Tilde{Z}_i|\ge \alpha_i)\nonumber\\
    &=\Pr(\Tilde{Z}_i^2\ge \alpha_i^2)\le \frac{\bE[\Tilde{Z}_i^2]}{\alpha_i^2}\le \frac{\bE[\Tilde{Z}_i^2]}{\alpha_{\min}^2}.
\end{align}
Moreover, recall that
\begin{align*}
    d_{i-1}(X)&=X-\textstyle\sum_{ f\in\hat{\cF}_{i-1}}\bE[X f] f\mbox{, and}\\
    \Sigma_i&=\bE[d_{i-1}(X)d_{i-1}(X)^\intercal],
\end{align*}
and therefore
\begin{align*}
    \boldnu_i^\intercal\Sigma_i\boldnu_i&=\boldnu_i^\intercal\bE[d_{i-1}(X)d_{i-1}(X)^\intercal]\boldnu_i\\
    &= \boldnu_i^\intercal\bE[(X-\textstyle\sum_{ f\in\hat{\cF}_{i-1}}\bE[X f] f)\\
    &\phantom{= \boldnu_i^\intercal\bE[(}(X-\textstyle\sum_{ f\in\hat{\cF}_{i-1}}\bE[X f] f)^\intercal]\boldnu_i\\
    &= \bE[(\boldnu_i^\intercal X-\textstyle\sum_{ f\in\hat{\cF}_{i-1}}\bE[\boldnu_i^\intercal X f] f)\\
    &\phantom{= \bE[(}(X^\intercal\boldnu_i-\textstyle\sum_{ f\in\hat{\cF}_{i-1}}\bE[X^\intercal\boldnu_i f] f)]\\
    &= \bE[(Z_i-\textstyle\sum_{ f\in\hat{\cF}_{i-1}}\bE[Z_i f] f)\\
    &\phantom{= \bE[(}(Z_i-\textstyle\sum_{ f\in\hat{\cF}_{i-1}}\bE[Z_i f] f)]\\
    &\phantom{=}\overset{\eqref{equation:ziTilde}}{=}\bE[\Tilde{Z}_i^2].
\end{align*}
Consequently, it follows that
\begin{align*}
    \eqref{equation:PrziTile}=\frac{\boldnu_i^\intercal\Sigma_i\boldnu_i}{\alpha_{\min}^2}\overset{\text{Lemma}~\ref{lemma:alldirections}}{\le}\frac{\epsilon^2}{\alpha_{\min}^2}.
\end{align*}
Now, by using the grouping rule \cite[Ex.~2.27]{cover1999elements} we have
\begin{align*}
    H(\Tilde{Z}_i)\le h_b(\epsilon^2/\alpha_{\min}^2)+\frac{\epsilon^2}{\alpha_{\min}^2}\log|\cX|,
\end{align*}
where~$h_b$ is the binary entropy function, defined as $h_b(p)=-p\log(p)-(1-p)\log(1-p)$. Finally, using the bound~$h_b(p)\le2\sqrt{p(1-p)}\le2\sqrt{p}$, it follows that
\begin{align*}
    H(X\vert Z)&\le \textstyle\sum_{i=m+1}^d H(\Tilde{Z}_i)\\
    &\le(d-m)(h_b(\epsilon^2/\alpha_{\min}^2)+(\epsilon^2/\alpha_{\min}^2)\log|\cX|)\\
    &\le(d-m)(2\epsilon/\alpha_{\min}+(\epsilon/\alpha_{\min})\log|\cX|)\\
    &=dO(\epsilon).\qedhere
\end{align*}
\end{proof}

\section{Gram-Schmidt Component Analysis (GCA)}\label{section:GCA}
In this section we present \textit{Gram-Schmidt Component Analysis} (GCA) in Algorithm~\ref{algorithm:GCA}, which demonstrates how  orthogonalization techniques similar to those used in GFR (Section~\ref{section:GFR}) can eliminate redundancy among the principal components. 
Specifically, it is shown that if the principal components (i.e., the projections of the data on its principal directions) satisfy some nonlinear equations, and if these equations are well approximated by the function family~$\cF$ in a sense that will be made formal shortly, then GCA identifies that redundancy, i.e., it eliminates the dependent components among the principal ones.

Recall that~$X=(X_1,\ldots,X_d)$ is the random variable from which the data is sampled, and let~$Z=(Z_1,\ldots,Z_d)$ be the principal components of~$X$ in decreasing order of variance. 
That is,~$Z_i=X^\intercal\boldrho_i$ where~$\boldrho_i$ is the~$i$'th largest principal direction of~$X$, and hence~$\var(Z_1)\ge\var(Z_2)\ge\ldots\ge \var(Z_d)$\footnote{The $\boldrho_1,\boldrho_2,\ldots$ notations in this section, which denote the principal directions (output of PCA), are not to be confused with~$\boldnu_1,\boldnu_2,\ldots$ in Section~\ref{section:GFR}, which denote the outputs of~GFR.}.
We comment that GCA can be applied using any orthonormal basis (other than the principal directions), and yet we focus on the principal components for their prominent role in data analysis. Applying GCA for the special case of the standard basis results in a (slightly simplified) feature selection algorithm, and the full details are given in Section~\ref{section:GFA}.
We now turn to formalize the notion of redundancy.
\begin{definition}\label{definition:epsilon_redundancy}
    For a function~$R:\bR^d\to\bR$ and~$\epsilon>0$, we say that~$X$ \textit{contains $\epsilon$-redundancy~$R$} among its principal components if~$\bE[R(Z)^2]\le \epsilon$.
\end{definition}

Intuitively, data which contains~$\epsilon$-redundancy~$R$ lies close to the set of zeros of~$R(\boldz)$ in some portion of the space.
Roughly speaking, each $\epsilon$-redundancy implies that some variable~$z_j$ which participates in~$R$ is \textit{redundant}, as it can be ``isolated'' from~$R(\boldz)=0$ in the form~$z_j=r(\boldz^{-j})$, where~$\boldz^{-j}\triangleq (z_1,\ldots,z_{j-1},z_{j+1},\ldots,z_d)$, for some function~$r$. 
This separation of~$z_j$ requires some additional conditions that will be discussed shortly.

The purpose of GCA is to eliminate such redundancies, i.e., for each~$\epsilon$-redundancy~$R$ we wish to skip the component which is made redundant by it, and output to the user only the non-redundant components. 
GCA can be seen as an additional redundancy removal algorithm that can be applied on top of PCA for further dimension reduction; PCA identifies if the data lies in a subspace, and GCA identifies if further (nonlinear) dependencies exist among the principal components. 
In Section~\ref{section:GFA}, we apply similar principals to feature \textit{selection}.

\begin{example}
    Suppose the data has a~$0$-redundancy~$R(z_1,z_2,z_3)=z_3-z_1z_2$, meaning that~$R(Z)=Z_3-Z_1Z_2=0$. It is clear that extracting~$Z_1$ and~$Z_2$ suffices in order to determine~$Z_3$, and hence extracting~$Z_1,Z_2$ and skipping~$Z_3$ can be understood as ``eliminating'' the redundancy~$R$.
\end{example}

Gram-Schmidt Component Analysis (GCA) is given in Algorithm~\ref{algorithm:GCA}. 
In iteration~$j$, it begins by identifying the principal components~~$\{\boldrho_i\}_{i\in\cE_j}$ in which the reduced data distribution~$d_j(X)$ (see Algorithm~\ref{algorithm:orthogonalize}) has less than~$\epsilon$ variance, and terminates if~$\cE_j=[d]$. 
The set~$\cE_j$ contains the indices of principal components that were either selected in previous iterations of GCA, or can be described as functions from~$\cF(\boldz)$ of the selected principal components. 
In  Line~\ref{line:GCA_new_data_distribution} GCA subtracts the components in~$\cE_j$ from~$X$ to define~$\bar{X}_j$, and the most variant component of~$\bar{X}_j$ is selected (the selected components are gathered in a set~$\cL_j$, which is later returned as output).
Finally, a call to Orthogonalize (Algorithm~\ref{algorithm:orthogonalize}) is made, in order to construct an orthogonal basis~$\hat{\cF}(Z_{\cL_j})$ to the function family~$\Span(\cF(Z_{\cL_j}))$, and in order to obtain the covariance matrix~$\Sigma_{j+1}$ of $d_{j}(X)$.
The corresponding equivalent algebraic implementation, denoted \textit{Algebraic-GCA}, is provided in Algorithm~\ref{algorithm:GCA_algebraic} in Appendix~\ref{section:algebraic-GFR-GCA}.

\begin{algorithm}[!h] 
\caption{Gram-Schmidt Component Analysis (GCA)} \label{algorithm:GCA}
\begin{algorithmic}[1]
\STATE {\bfseries Input:} 
A function family~$\cF(\boldz)$ and a threshold~$\epsilon>0$.
\STATE {\bfseries Output:} {A set~$\cL\subseteq[d]$ of indices of extracted principal directions.}
\STATE {\bfseries Initialize:} $\Sigma_1=\bE[XX^\intercal]$ and $\cL_0=\varnothing$.
\FOR{$j\gets1$ {\bfseries to} $d$}
\STATE Let $\cE_j=\{ i\vert \boldrho_i^\intercal \Sigma_j\boldrho_i <\epsilon \}$.\label{line:GCA_E_j}
\IF{$\cE_j=[d]$}
\RETURN{}\hspace{-1mm}$\cL_{j-1}$.
\ENDIF
\STATE Define~$\bar{X}_j = X - \sum_{i\in\cE_j}X^\intercal\boldrho_i\cdot\boldrho_i$.\label{line:GCA_new_data_distribution}
\STATE Let $\ell_j=\arg\max_{i\in [d]\setminus\cE_j}\boldrho_i^\intercal \bE[\bar{X}_j\bar{X}_j^\intercal] \boldrho_i$ (breaking \\ ties arbitrarily), and $\cL_j=\cL_{j-1}\cup\{\ell_j\}$.\label{line:GCA_most_variant_component_selection}
\STATE Define $Z_{\ell_j}=\boldrho_{\ell_j}^\intercal X$.
\STATE $\Sigma_{j+1},\hat{\cF}(Z_{\cL_j})=\text{Orthogonalize}(\hat{\cF}(Z_{\cL_{j-1}}),\cF(\boldz_{[j]})\!\setminus\cF(\boldz_{[j-1]}),\{Z_{\ell_i}\}_{i=1}^j)$. 
\ENDFOR
\RETURN \hspace{-1mm}$\cL_d$.
\end{algorithmic}
\end{algorithm}

We now turn to formally describe and explain the ability of GCA to eliminate redundancy in the data. For some~$\epsilon>0$, suppose~$X$ contains the $\epsilon$-redundancies~$R^{(1)},\ldots,R^{(t)}$ among its principal components (i.e.,~$\bE[R^{(1)}(Z)^2]\le \epsilon$, $\ldots$ ,$\bE[R^{(t)}(Z)^2]\le \epsilon$), and for every~$i\in[t]$ let~$Z_{j_i}$ be the least variant component among those participating in~$R^{(i)}$. We consider the components~$\{Z_{j_i}\}_{i=1}^t$ as \textit{redundant}, and show that they can be eliminated by GCA if the function family~$\cF(\boldz)$ is in some sense sufficiently descriptive. We make the following simplifying assumptions.
\begin{assumption}\label{assumtopn:distinctANDindependent}
    The integers~$j_1,\ldots,j_t$ are distinct, and there is no ``inter-dependence'' among~$Z_{j_1},\ldots,Z_{j_t}$. That is, for every~$i$ the variable~$z_{j_i}$ does not appear in~$R^{(j)}(\boldz)$ for every~$j\ne i$. A justification for this assumption using the well-known implicit-function theorem is given in the sequel.
\end{assumption}
We also assume without loss of generality that~$j_1<\ldots<j_t$, and denote the redundant index-set by~$\cR^\epsilon\triangleq\{j_1,\ldots,j_t\}$ and the non-redundant one by~$\cN^\epsilon\triangleq [d]\setminus\cR^\epsilon$. Further, we denote~$\cN_i^\epsilon=[j_i-1]\setminus\{j_\ell\}_{\ell=1}^{i-1}$ for all~$i\in[t]$, i.e.,~$\cN_i^\epsilon$ is the index-set of non-redundant components up to~$j_i-1$.
Recall that for a set~$\cS=\{\sigma_1,\ldots,\sigma_{|\cS|}\}\subseteq[d]$, with~$\sigma_1<\ldots<\sigma_{|\cS|}$, we denote by~$\cF(Z_\cS)$ the result of taking~$\cF(z_1,\ldots,z_{|\cS|})$ and substituting~$z_i\leftarrow Z_{\sigma_i}$ for every~$f\in \cF(z_1,\ldots,z_{|\cS|})$ and every~$i\in[|\cS|]$. 
Also recall that $\hat{\cF}(Z_\cS)$ is the result of applying the Gram-Schmidt process over~$\cF(Z_\cS)$, and let $\proj_{\cF(Z_\cS)}Z_j$ be the projection of~$Z_j$ on $\Span\cF(Z_\cS)$, given by $\sum_{f\in \hat{\cF}(Z_\cS)}\bE[Z_j f]f$.

Using these notations, we proceed to define what it means for a function family~$\cF(\boldz)$ to be sufficiently descriptive in order for GCA to eliminate said redundancies. For a given~$\epsilon>0$ we say that the $\epsilon$-redundancy~$R^{(i)}$ is $\epsilon$-\textit{captured} by~$\cF(\boldz)$ if~$\bE[(Z_{j_i}-\proj_{\cF(Z_{\cN_i^\epsilon})}Z_{j_i})^2]\le \epsilon$. 
In simple terms, $R^{(i)}$ is captured by~$\cF(\boldz)$ if the least variant component of~$R^{(i)}$ can be ``well approximated'' by the closest function to~$Z_{j_i}$ in~$\cF(\boldz)$ (this function take the remaining components in~$R^{(i)}$ as inputs). 
Both the closeness and the approximation are measured by the expected~$\ell_2$-deviation.

\begin{theorem}\label{theorem:epsCapture}
    Under Assumption~\ref{assumtopn:distinctANDindependent}, let~$R^{(1)},\ldots,R^{(t)}$ be the~$\epsilon$-redundancies of~$X$. If all~$R^{(i)}$'s are $\epsilon$-captured by~$\cF(\boldz)$, then Gram-Schmidt Component Analysis (Algorithm~\ref{algorithm:GCA}) outputs the non-redundant directions~$\{\boldrho_i\}_{i\in\cN^\epsilon}$.    
\end{theorem}

\begin{proof}
We ought to show that~$\cL=\cN^\epsilon$, i.e., that every element in~$\cN^\epsilon$ is chosen to~$\cL_j$ at some step~$j$ of GCA, and no element in~$\cR^\epsilon$ is chosen to any such~$\cL_j$.
We begin by observing that since~$X=\sum_{i=1}^d\boldrho_iZ_i$, it follows from the definition of Orthogonalize (Algorithm~\ref{algorithm:orthogonalize}) that for every~$j\in[d]$,
\begin{align*}
    d_{j}(X)&=\textstyle X-\sum_{f\in\hat{\cF}(Z_{\cL_{j-1}})}\bE[Xf]f\nonumber\\
    &=\textstyle\sum_{i=1}^d\boldrho_i ( Z_i-\sum_{f\in \hat{\cF}(Z_{\cL_{j-1}})}\bE[Z_if]f)\nonumber\\
    &=\textstyle \sum_{i=1}^d\boldrho_i ( Z_i-\proj_{\cF(Z_{\cL_{j-1}})}Z_i).
\end{align*}
Since~$\boldrho_i$ represent the principal directions, they are orthogonal. Therefore, for every~$i, j \in [d]$, we have:
\begin{align}\label{equation:rhoSigmarho}
    \boldrho_i^\intercal\Sigma_j\boldrho_i&=\boldrho_i^\intercal\bE[d_j(X) d_j(X)^\intercal]\boldrho_i\nonumber\\
    &=\bE[(Z_i-\proj_{\cF(Z_{\cL_{j-1}})}Z_i)^2].
\end{align}
Therefore, in each iteration~$j$ the set~$\cE_j$ contains all indices~$i$ such that~$i\in\cL_{j-1}$ (since for those indices we have $\bE[(Z_i-\proj_{\cF(Z_{\cL_{j-1}})}Z_i)^2]=0$). In addition,~$\cE_j$ also contains all indices~$i'$ such that~$i'\notin\cL_{j-1}$ but~$\bE[(Z_{i'}-\proj_{\cF(Z_{\cL_{j-1}})}Z_{i'})^2]\le\epsilon$, i.e.,~$Z_{i'}$ can be well approximated by its own projection on~$\Span\cF(Z_{\cL_{j-1}})$. 

To prove that~$\cL=\cN^\epsilon$, we prove by induction over~$i\in[t]$ that every~$j\in[j_i]$ is assigned appropriately by GCA, meaning, $j$ is assigned to~$\cL$ if and only if it is in~$\cN^\epsilon$. 

\noindent{\textbf{Base case:}} We begin by showing that the first redundant component~$Z_{j_1}$ is not chosen as non-redundant, i.e., that~$j_1\notin \cL$. Suppose for contradiction that~$j_1\in \cL$, and let~$j\in[d]$ be the smallest integer such that~$j_1\in \cL_j$, i.e., $j$ is the step in which~$Z_{j_1}$ is selected. 
        It follows from Line~\ref{line:GCA_most_variant_component_selection} that~$Z_{j_1}$ is the most variant component of~$\bar{X}_j$. 
        In addition, it also follows from Line~\ref{line:GCA_new_data_distribution}, that~$j_1\notin \cE_j$, since otherwise~$Z_{j_1}$ would not have been a component of~$\bar{X}_j$. 
        Therefore, since $j_1\notin \cE_j$, it follows from Line~\ref{line:GCA_E_j} and from~\eqref{equation:rhoSigmarho} that
        \begin{align}\label{equation:j1Variant}
            \bE[(Z_{j_1}-\proj_{\cF(Z_{\cL_{j-1}})}Z_{j_1})^2]\ge \epsilon.
        \end{align}
        Since~$Z_{j_1}$ is the most variant component of~$\bar{X}_j$, it follows that neither of the more variant~$Z_1,\ldots,Z_{j_1-1}$ is a component of~$\bar{X}_j$, meaning, every~$i\in\{1,\ldots,j_1-1\}$ was either selected at some earlier step of the algorithm as the most variant component in  Line~\ref{line:GCA_most_variant_component_selection}, or not selected in  Line~\ref{line:GCA_most_variant_component_selection} but rather subtracted from~$\bar{X}_{j'}$ at some earlier step~$j'$ of the algorithm in Line~\ref{line:GCA_new_data_distribution} (recall that a component~$i'$ is subtracted in Line~\ref{line:GCA_new_data_distribution} of step~$j'$ if belongs to~$\cE_{j'}$, or equivalently by~\eqref{equation:rhoSigmarho}, if $\bE[(Z_{i'}-\proj_{\cF(Z_{\cL_{j'-1}})}Z_{i'})^2]< \epsilon$). 
        
        If the latter is true, i.e, if there exists~$i'<j_1$ and~$j'<j$ such that $\bE[(Z_{i'}-\proj_{\cF(Z_{\cL_{j'-1}})}Z_{i'})^2]<\epsilon$, if follows that there exists a function~$g\in \Span\cF(Z_{\cL_{j'-1}})$ such that~$\bE[(Z_{i'}-g)^2]<\epsilon$. 
        That is, the function~$Z_{i'}-g$ is an~$\epsilon$-redundancy, whose least variant component~$Z_{i'}$ is more variant than~$Z_{j_1}$, a contradiction to Assumption~\ref{assumtopn:distinctANDindependent}. 
        Therefore, the latter must hold, i.e., every~$i\in\{1,\ldots,j_1-1\}$ was selected at some earlier step of the algorithm as the most variant component in Line~\ref{line:GCA_most_variant_component_selection}, hence~$\cL_{j-1}=\{1,\ldots,j_1-1\}$, or in other words~$\cL_{j-1}=\cN_1^\epsilon$. 
        However, $\cL_{j-1}=\cN_1^\epsilon$ implies by~\eqref{equation:j1Variant} that $\bE[(Z_{j_1}-\proj_{\cF(Z_{\cN_1^\epsilon})}Z_{j_1})^2]\ge\epsilon$, contradicting the fact that~$R^{(1)}$ is captured by~$\cF(\boldz)$. This concludes showing that~$j_1\notin\cL$.
        
        It remains to show that~$\{1,\ldots,j_1-1\}\subseteq \cL$, which is similar. Assuming otherwise, by following steps similar to above we obtain that some~$i'\in\{1,\ldots,j_1-1\}$ is the least variant component of some~$\epsilon$-redundancy, contradicting Assumption~\ref{assumtopn:distinctANDindependent}.

        \noindent\textbf{Induction hypothesis:} Assume that~$j_1,\ldots,j_{i-1}\notin\cL$, and that~$[j_{i-1}]\setminus \{j_1,\ldots,j_{i-1}\}\subseteq\cL$ for some~$i\ge 2$.

        \noindent\textbf{Induction step:} 
        We ought to show that~$j_i\notin \cL$ and that~$\{j_{i-1}+1,\ldots,j_{i}-1\}\subseteq\cL$.
        We begin by showing that~$j_i\notin\cL$. 
        Assume for contradiction that~$j_i\in\cL$, and let~$j\in[d]$ be the smallest integer such that~$j_i\in\cL_j$. 
        It follows from  Line~\ref{line:GCA_most_variant_component_selection} that~$Z_{j_i}$ is the most variant component of~$\bar{X}_j$, and similarly to the base case, that~$j_i\notin\cE_j$ and 
        \begin{align}\label{equation:jiVariant}
            \bE[(Z_{j_i}-\proj_{\cF(Z_{\cL_{j-1}})}Z_{j_i})^2]\ge \epsilon.
        \end{align}
        Since~$Z_{j_i}$ is the most variant component of~$\bar{X}_j$, it follows that neither of the more variant~$Z_{j_{i-1}+1},\ldots,Z_{j_i-1}$ is a component of~$\bar{X}_j$, meaning, every~$i'\in\{j_{i-1}+1,\ldots,j_i-1\}$ was either selected to~$\cL$ at some earlier step~$j'$ in  Line~\ref{line:GCA_most_variant_component_selection},        
        or it was not selected in  Line~\ref{line:GCA_most_variant_component_selection} of any step, but it belongs to~$\cE_{j'}$ for some~$j'<j$. 
        If it is the latter, i.e., if there exists~$i'\in\{j_{i-1}+1,\ldots,j_i-1\}$ and~$j'<j$ such that~$i'\in\cE_{j'}$ but~$i'$ was not selected in  Line~\ref{line:GCA_most_variant_component_selection}, then~$\bE[(Z_{i'}-\proj_{\cF(\cL_{j'-1})}Z_{i'})^2]<\epsilon$, which implies the existence of a function~$g\in\Span\cF(\cL_{j'-1})$ such that~$Z_{i'}-g$ is an~$\epsilon$-redundancy.
        Therefore, we have that~$Z_{i'}-g$ is an~$\epsilon$-redundancy whose least variant component~$Z_{i'}$ is neither~$Z_{j_{i-1}}$ nor~$Z_{j_i}$, a contradiction to Assumption~\ref{assumtopn:distinctANDindependent}. 
        Therefore, the former case must hold, i.e., we must have that each~$i'\in\{j_{i-1}+1,\ldots,j_i-1\}$ was selected in  Line~\ref{line:GCA_most_variant_component_selection}, which implies that~$\cL_{j-1}=\cN_i^\epsilon$ by the induction hypothesis.
        However, $\cL_{j-1}=\cN_i^\epsilon$ implies by~\eqref{equation:jiVariant} that $\bE[(Z_{j_i}-\proj_{\cF(Z_{\cN_i^\epsilon})}Z_{j_i})^2]\ge\epsilon$, contradicting the fact that~$R^{(i)}$ is captured by~$\cF(\boldz)$. Showing that each of~$j_{i-1}+1,\ldots,j_i-1$ are selected to~$\cL$ is once again similar.
\end{proof}

We continue by drawing several conclusions from Theorem~\ref{theorem:epsCapture}. 
In particular, we would like to obtain a better understanding of the term ``$R^{(i)}$ is $\epsilon$-captured by~$\cF(\boldz)$'', and we do so using the well-known \textit{implicit function theorem}. 
To gain some intuition, observe that since we extract components in decreasing order of variance, eliminating redundancy~$R$ depends on the representation of the least variant variable in~$R$ as a function of all other variables which participate in~$R$. 
For example, if~$R(\boldz)$ depends only on~$z_1,z_2,z_3$, we would like to extract~$z_3$ from~$R(z_1,z_2,z_3)=0$ in the form of~$z_3=r(z_1,z_2)$ for some function~$r$ such that~$r(Z_1,Z_2)$ can be well-approximated by a function in~$\Span\cF(Z_1,Z_2)$. 
Then, since~$Z_1,Z_2$ are extracted prior to~$Z_3$ due their higher variance, the orthogonalization and subtraction processes enable GCA to identify that~$Z_3$ can be determined by~$Z_1,Z_2$, and hence skip it.

Recall that the implicit function theorem~\cite{krantz2002implicit,lee2012smooth} provides sufficient conditions for extracting a variable~$z_j$ from a function~$R(\boldz)$ inside its zero set~$\cZ_R\triangleq\{\boldalpha\vert R(\boldalpha)=0\}$.
This theorem states that if a continuously differentiable\footnote{That is, $R$ is differentiable, and its gradient is a continuous function. Further, for~$\boldalpha\in\bR^d$, the notation $\frac{\partial}{\partial z_j}R(\boldz)\vert_{\boldz=\boldalpha}$ signifies the result of deriving~$R(\boldz)$ according to~$z_j$, and then substituting~$z_i\leftarrow \alpha_i$ for all~$i\in[d]$.} function~$R(\boldz)$ satisfies~$\frac{\partial}{\partial z_j}R(\boldz)\vert_{\boldz=\boldalpha}\ne 0$ for some~$j\in[d]$ and some~$\boldalpha\in \cZ_R$, then there exists a continuous function~$r:\bR^{d-1}\to\bR$ such that $$R(\beta_1,\ldots,\beta_{j-1},r(\boldbeta^{- j}),\beta_{j+1},\ldots,\beta_d)=0$$
for every~$\boldbeta$ in an environment of~$\boldalpha$ in~$\cZ_R$, where~$\boldbeta^{-j}\triangleq (\beta_1,\ldots,\beta_{j-1},\beta_{j+1},\ldots,\beta_d)$. 
That is, for any point~$\boldalpha\in\cZ_R$ in which the derivative of~$R(\boldz)$ according to the variable~$z_j$ is nonzero, there exists an environment~$\cV\subseteq\cZ_R$ which contains~$\boldalpha$ and a function~$g$ such that~$\beta_j=g(\boldbeta^{-j})$ for all~$\boldbeta\in\cV$. We will require a continuous extension of the implicit function theorem, which applies a similar principal to entire connected regions of~$\cZ_R$ in which the derivative according to~$z_j$ is nonzero.

\begin{theorem}[Extended Implicit Function Theorem] \cite[Lemma~2.7]{ge2013stable}, \cite{ge2002adaptive,sandberg1981global} \label{theorem:EIFT}
    Let~$R:\bR^d\to\bR$ be a continuously differentiable function, and for~$j\in[d]$ let~$\cB_j\subseteq \cZ_R$ be the collection of all points~$\boldbeta\in\cZ_R$ such that~$\frac{\partial}{\partial z_j}R(\boldz)\vert_{\boldz=\boldbeta}\ne 0$. Further, let~$\cC$ be a connected component\footnote{A connected component is an equivalence class of a relation~$\sim$ over~$\cB_j$, where~$\boldbeta\sim\boldbeta'$ if there exists a continuous path~$q(t)$ such that~$q(0)=\boldbeta$, $q(1)=\boldbeta'$, and~$q(t)\in \cB_j$ for all~$t\in[0,1]$.} of~$\cB_j$. Then, there exists a continuous~$g:\bR^{d-1}\to\bR$ such that~$\gamma_j=g(\boldgamma^{-j})$ for all~$\boldgamma\in\cC$.
\end{theorem}

In simple terms, Theorem~\ref{theorem:EIFT} states that the representation of a variable~$z_j$ as a function of~$\boldz^{-j}$ near a point~$\boldbeta\in\cZ_R$ with~$\frac{\partial}{\partial z_j}R(\boldz)\vert_{\boldz=\boldbeta}\ne 0$ (as guaranteed by the implicit function theorem), can be continuously extended to the entire portion~$\cC$ of~$\cZ_R$ reachable from~$\boldbeta$ such that~$\frac{\partial}{\partial z_j}R(\boldz)\vert_{\boldz=\boldgamma}\ne 0$ for all~$\boldgamma\in\cC$ (see examples in Figure~\ref{fig:IFT_Example}).

\begin{figure*} 
\centering
\subfloat[\label{fig:IFT_Example_a}]{%
\includegraphics[width=0.33\linewidth]{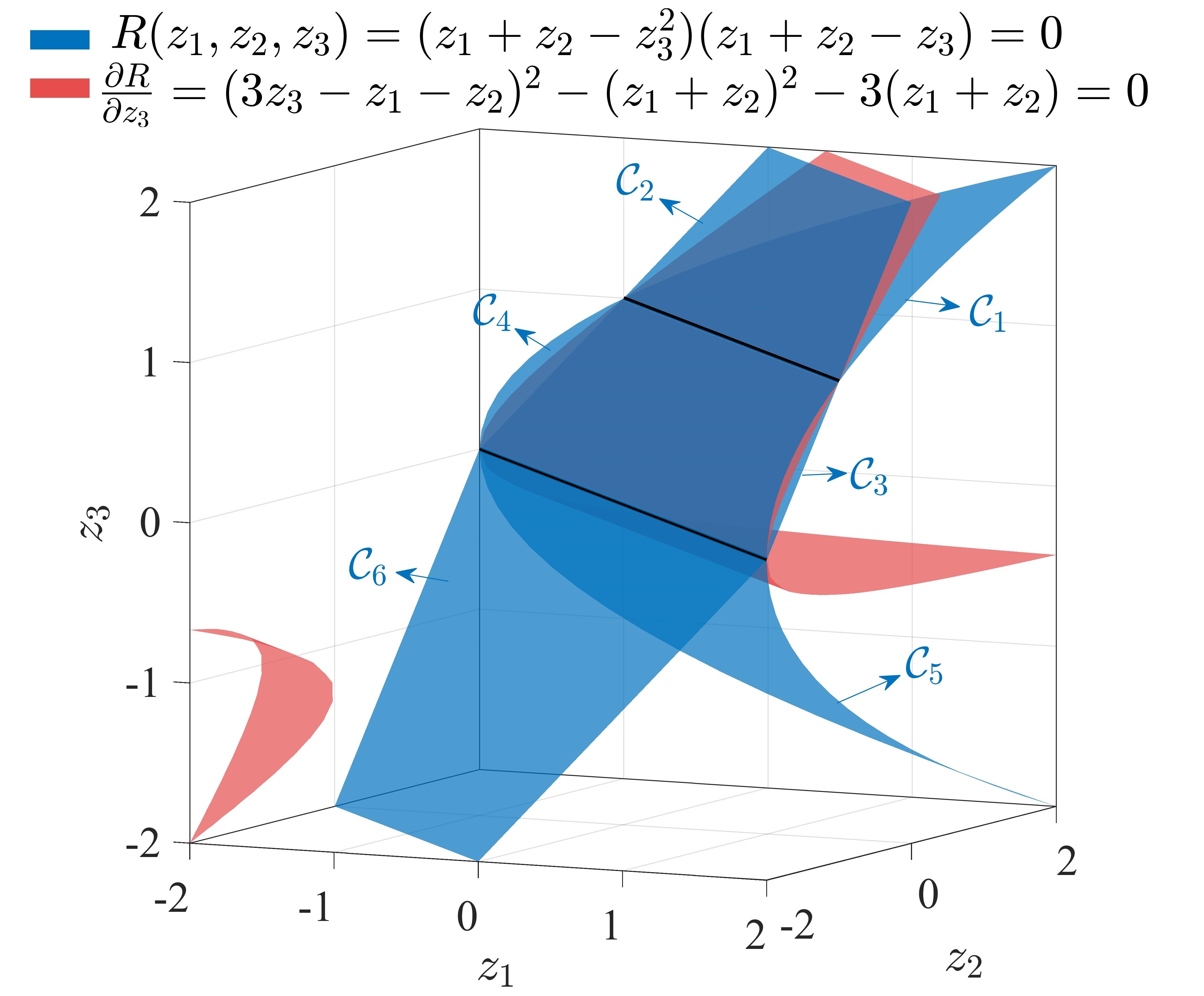}}
\hfill
\subfloat[\label{fig:IFT_Example_b}]{%
\includegraphics[width=0.33\linewidth]{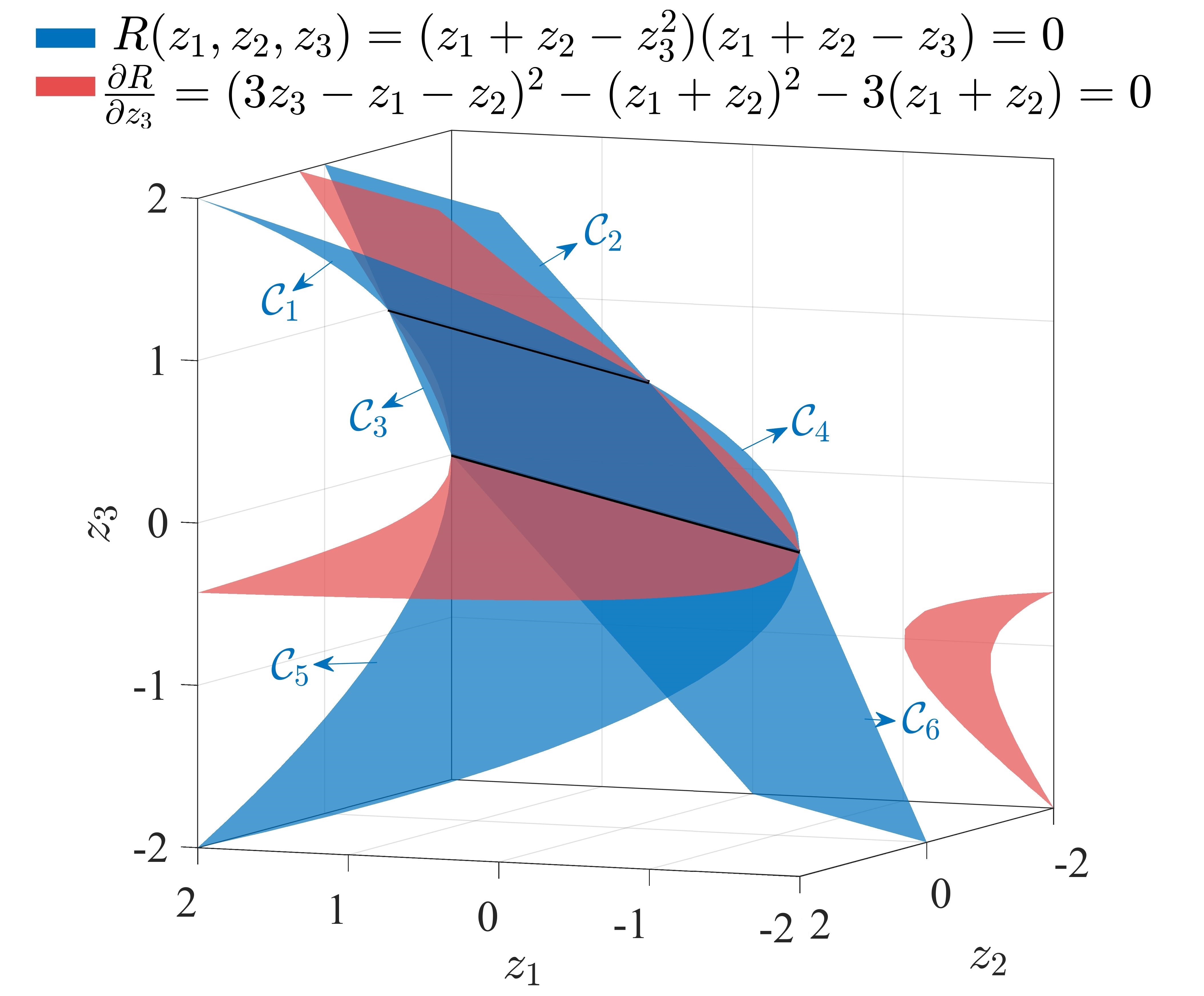}}
\hfill
\subfloat[\label{fig:IFT_Example_c}]{%
\includegraphics[width=0.33\linewidth]{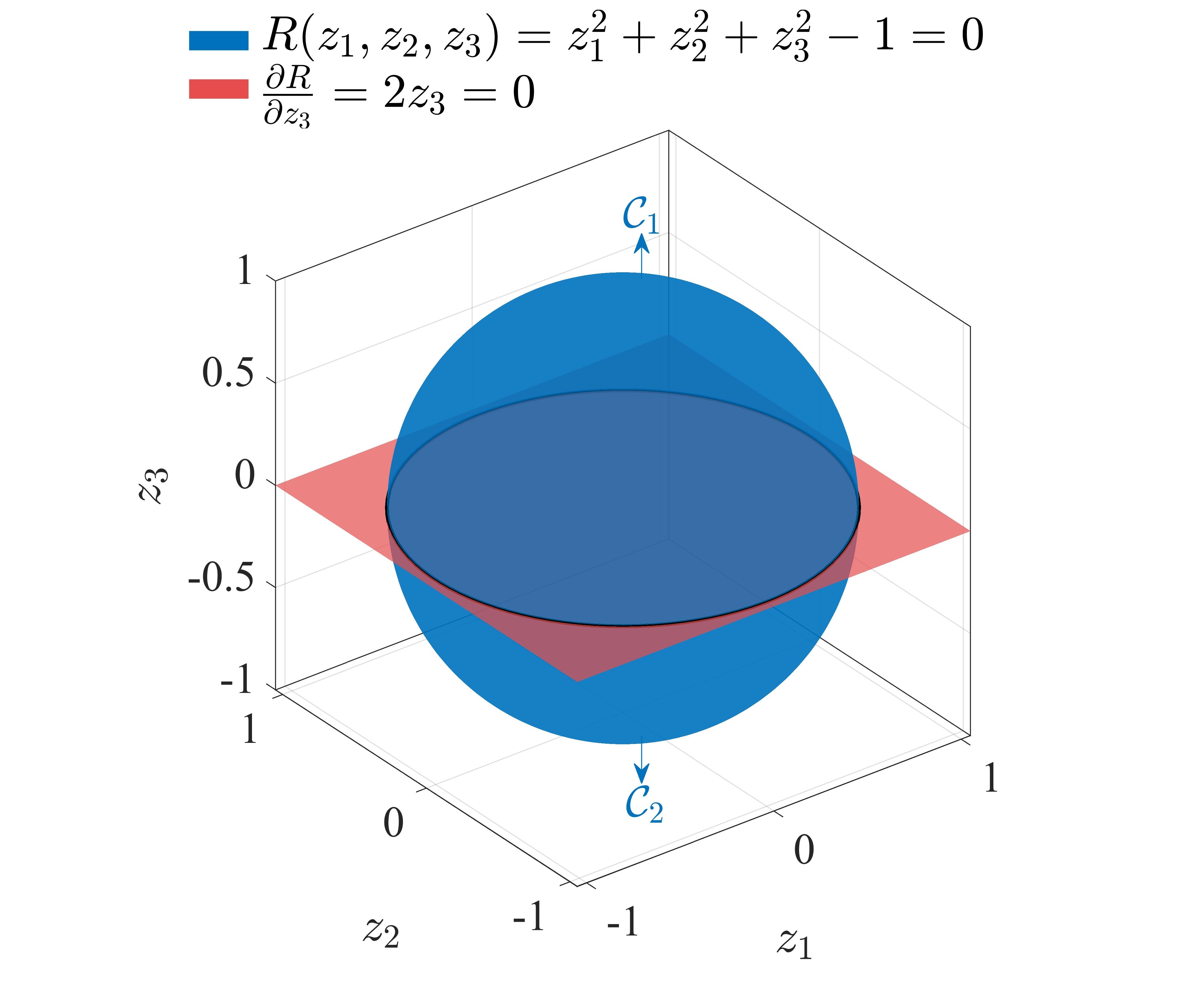}}
\caption{(a) and (b): Two views of the zero set~$\cZ_R$ of~$R(\boldz)=(z_1+z_2-z_3^2)(z_1+z_2-z_3)$ (in blue), the zero set of its derivative by~$z_3$ (in red), and their intersection (in black). The respective subset~$\cB_3\subseteq\cZ_R$ in which Theorem~\ref{theorem:EIFT} applies is partitioned to six connected components~$\cC_1,\ldots,\cC_6$, in each of which~$z_3$ can be isolated as a function of~$z_1,z_2$. It can be seen that in~$\cC_1$ and~$\cC_4$ we have~$z_3=\sqrt{z_1+z_2}$, in~$\cC_2,\cC_3$, and~$\cC_6$ we have~$z_3=z_1+z_2$, and in~$\cC_5$ we have~$z_3=-\sqrt{z_1+z_2}$. (c) A similar figure for~$R(\boldz)=z_1^2+z_2^2+z_3^2-1$, where in~$\cC_1$ we have~$z_3=\sqrt{z_1^2+z_2^2}$ and in~$\cC_2$ we have~$z_3=-\sqrt{z_1^2+z_2^2}$.}
\label{fig:IFT_Example}
\end{figure*}

Going back to drawing conclusions from Theorem~\ref{theorem:epsCapture}, we begin with the case~$\epsilon=0$, in which the data satisfies~$R^{(1)}(Z)=\ldots=R^{(t)}(Z)=0$. 
We would like to identify precise conditions by which Theorem~\ref{theorem:epsCapture} holds, i.e., to properly interpret the conditions that the redundancies~$R^{(i)}$ are $\epsilon$-captured by~$F(\boldz)$. 
To this end, consider the case in which~$\frac{\partial}{\partial z_{j_i}}R^{(i)}(\boldz)\vert_{\boldz=\boldalpha}\ne 0$ for every~$\boldalpha\in\operatorname{Supp}(X)$ and every~$i\in[t]$.
In this case Theorem~\ref{theorem:EIFT} implies that there exists a function~$r^{(i)}:\bR^{d-1}\to\bR$ which enables~$z_{j_i}$ to be ``isolated'' from~$R^{(i)}(\boldz)=0$ in the form~$z_{j_i}=r^{(i)}(\boldz^{-j_i})$. More precisely, since~$R^{(i)}$ need not depend on all~$z_i$'s, let~$\cZ_i\subseteq[d]$ be the set of variables which appear in~$R^{(i)}$ (notice that~$j_i\in\cZ_i$). Then, the assumption~$\frac{\partial}{\partial z_{j_i}}R^{(i)}(\boldz_{\cZ_i})\vert_{\boldz=\boldalpha}\ne 0$ for every~$\boldalpha\in\operatorname{Supp}(X)$ implies the ability to write~$z_{j_i}=r^{(i)}(\boldz_{\cZ_i}^{-j_i})$. The existence of these~$r^{(i)}$'s enables two conclusions:
\begin{enumerate}
    \item[\textbf{C1}.] By composing functions, we can justify the ``inter-dependence'' in Assumption~\ref{assumtopn:distinctANDindependent}. For example, if~$r^{(2)}$ depends on~$z_{j_1}$, one can substitute~$z_{j_1}$ for its functional representation~$r^{(1)}(\boldz_{\cZ_1}^{-j_1})$, and eliminate the (direct) dependence of $r^{(2)}$ on~$z_{j_1}$. By applying this argument repeatedly for all ``inter-dependencies'' among the~$z_{j_i}$'s, one can write~$z_{j_i}=r^{(i)}(\boldz_{\cN_i})$ for all~$i$, where~$\cN_i=[j_i-1]\setminus\{j_\ell\}_{\ell=1}^{i-1}$.
    \item[\textbf{C2}.] Following~$\textbf{C1}$, we obtain a precise interpretation of the term ``the~$0$-redundancy $R^{(i)}$ is $0$-captured by~$\cF(\boldz)$.'' Clearly, in the case where $z_{j_i}=r^{(i)}(\boldz_{\cN_i})$ we have that~$\bE[(Z_{j_i}-\proj_{\cF(Z_{\cN_i})}Z_{j_i})^2]=0$ if and only if~$r^{(i)}(Z_{\cN_i})\in\Span\cF(Z_{\cN_i})$.
\end{enumerate}

Conclusions~\textbf{C1} and~\textbf{C2} readily imply the first corollary of Theorem~\ref{theorem:epsCapture}.

\begin{corollary}\label{corollary:zeroRedZeroApprox}
    Let~$R^{(1)},\ldots,R^{(t)}$ be the~$0$-redundancies of~$X$. If (1) the derivative of each~$R^{(i)}$ according to its least variant variable~$z_{j_i}$ is nonzero on the data support; and (2) the resulting functions~$r^{(i)}(Z_{\cN_i})$ are in~$\Span\cF(Z_{\cN_i})$ for all~$i$, then GCA outputs~$\cN^0=[d]\setminus\{j_1,\ldots,j_t\}$. In words, if the least variant variable in each~$0$-redundancy can be isolated as a function of the remaining variables, and if that function belongs to~$\Span\cF$, then GCA eliminates \emph{all} redundancies.
\end{corollary}

Next, to extend Corollary~\ref{corollary:zeroRedZeroApprox}, we consider cases in which the~$0$-redundancies satisfy the assumptions of the implicit function theorem only up to measure zero. As shown next, under additional assumptions GCA can guarantee successful elimination of the redundancies in this case as well.

Suppose as before that the data contains the $0$-redundancies~$R^{(1)},\ldots,R^{(t)}$, each with least variant variable~$z_{j_i}$. 
However, we relax the assumption of Corollary~\ref{corollary:zeroRedZeroApprox}, and only assume that the respective derivatives vanish on a subset of measure zero. 
That is, for every~$i\in[t]$ we assume that~$\Pr(Z\in \cB_i)=1$, where $\cB_i=\{\boldbeta\in\bR^d:\frac{\partial}{\partial z_{j_i}}R(\boldz)\vert_{\boldz=\boldbeta}\ne 0\}$.
This assumption is not overly restrictive since in many cases the zero set of the derivative will have Lebesgue measure zero, hence measure zero for many continuous distributions. For example, consider Figure~\ref{fig:IFT_Example_c} under the normalized Gaussian distribution. It is readily verified that the probability to sample a point on the black line is zero.

Clearly, the region~$\cB_i$ need not be connected (see Figure~\ref{fig:IFT_Example}), and hence we let~$\cC_{i,1},\ldots,\cC_{i,\ell_i}$ be its connected components. 
Using Theorem~\ref{theorem:EIFT} on each one of~$\cC_{i,1},\ldots,\cC_{i,\ell_i}$ separately guarantees that there exists functions~$g_{i,1},\ldots,g_{i,\ell_i}$ which enable to extract~$z_{j_i}$ as a function of~$\boldz^{-j_i}$ in each of~$\cC_{i,1},\ldots,\cC_{i,\ell_i}$, respectively. 
That is, we have that~$\alpha_{j_i}=g_{i,j}(\boldalpha^{-j_i})$ for every~$i\in[\ell]$ and every~$\boldalpha\in\cC_{i,j}$. 
The success of GCA in this case will follow if~$\cF(\boldz)$ can approximate each function~$g_{i,j}$ \textit{locally}, meaning, in the connected component~$\cC_{i,j}$ in which it applies. To this end, for~$i\in[t]$ we say that the~$0$-redundancy~$R^{(i)}$ is \textit{$\epsilon$-locally captured} by~$\cF(\boldz)$ if~$\bE[(Z_{j_i}-\proj_{\cF(Z_{\cN_i^\epsilon})}Z_{j_i})^2|Z\in \cC_{i,j}]\le \epsilon$ for every~$j\in[\ell_i]$.

\begin{corollary}\label{corollary:some0someeps}
    For~$\epsilon>0$, suppose that~$X$ contains redundancies~$R^{(1)},\ldots,R^{(t)}$ of either of two types:
    \begin{enumerate}
        \item[\textbf{A}.] $\epsilon$-redundancies; or
        \item[\textbf{B}.] $0$-redundancies whose derivatives by the least variant variable vanishes only on a set of measure zero.
    \end{enumerate}
     Further, suppose that every redundancy of type~\textbf{A} is $\epsilon$-captured by~$\cF(\boldz)$, and every redundancy of type~$\textbf{B}$ is $\epsilon$-locally captured by~$\cF(\boldz)$. Then, GCA outputs\footnote{Note that a $0$-redundancy is also an~$\epsilon$-redundancy for every~$\epsilon>0$, and hence one can denote~$\cR^\epsilon,\cN^\epsilon,\cN_i^\epsilon$ in accordance with Assumption~\ref{assumtopn:distinctANDindependent}.}~$\cN^\epsilon=[d]\setminus\{j_1,\ldots,j_t\}$.
\end{corollary}
\begin{proof}
    Observe that for every redundancy~$R^{(i)}$ of type \textbf{B} we have
    \begin{align*}
        \bE&[(Z_{j_i}-\proj_{\cF(Z_{\cN_i^\epsilon})}Z_{j_i})^2]\\
        &=\textstyle\sum_{a=1}^{\ell_i}\Pr(Z\in\cC_{i,a})\bE[(Z_{j_i}-\proj_{\cF(Z_{\cN_i^\epsilon})}Z_{j_i})^2|Z\in\cC_{i,a}]\\
        &\le\textstyle \epsilon\sum_{a=1}^{\ell_1}\Pr(Z\in\cC_{i,a})\le \epsilon,
    \end{align*}
    where the latter inequality follows since~$\Pr(Z\notin\cup_{a=1}^{\ell_1}\cC_{i,a})=0$ by the assumption on the respective derivative of redundancies of type~\textbf{B}. Therefore, it follows that each of~$R^{(1)},\ldots,R^{(t)}$ is an~$\epsilon$-redundancy which is captured by~$\cF(\boldz)$, and hence the lemma follows from Theorem~\ref{theorem:epsCapture}.
\end{proof}

In particular, for a redundancy~$R^{(i)}$ of type \textbf{B} in Corollary~\ref{corollary:some0someeps} we have that~$Z_{j_i}=r^{(i)}_j(Z_{\cN_i^\epsilon})$ for some function~$r^{(i)}_j$ in the connected component~$\cC_{i,j}$, and therefore the term ``$R^{(i)}$ is~$\epsilon$-locally captured by~$\cF(\boldz)$'' can be interpreted as
\begin{align*}
    \bE&[(Z_{j_i}-\proj_{\cF(Z_{\cN_i^\epsilon})}Z_{j_i})^2|Z\in \cC_{i,j}]\nonumber\\
    &=\bE[(r^{(i)}_j(Z_{\cN_i^\epsilon})-\proj_{\cF(Z_{\cN_i^\epsilon})}r^{(i)}_j(Z_{\cN_i^\epsilon}))^2|Z\in \cC_{i,j}]\le \epsilon,
\end{align*}
meaning,~$\cF(Z_{\cN_i^\epsilon})$ can~$\epsilon$-approximate the function~$r^{(i)}_j$ in the connected component~$\cC_{i,j}$.

In summary, Theorem~\ref{theorem:epsCapture} implies that whenever Assumption~\ref{assumtopn:distinctANDindependent} is satisfied, GCA successfully removes all $\epsilon$-redundancies from the data for any~$\cF(\boldz)$ which is sufficiently descriptive, i.e., it can approximate the least variant component in each redundancy. 
To interpret what this approximation means, in Corollary~\ref{corollary:zeroRedZeroApprox} it was shown that for~$\epsilon=0$, the implicit function theorem can further imply that the descriptiveness of~$\cF$ reduces to $\Span(\cF)$ containing the ``isolation'' of each least variant variable from its respective redundancy; this required nonzero derivative by that variable in all of the data support. The latter requirement was alleviated in Corollary~\ref{corollary:some0someeps} in the case where nonzero derivatives existed almost everywhere, where strict containment of the isolation in~$\Span(\cF)$ is relaxed to~$\epsilon$-approximation.

For $\epsilon$-redundancies with~$\epsilon>0$, interpreting Theorem~\ref{theorem:epsCapture} is more challenging. The main reason is that the existence of such redundancy~$R$ does not necessarily imply that~$R(Z)=0$, and hence Theorem~\ref{theorem:EIFT} cannot be directly applied. Therefore, the term ``capturing'' does not translate to containing some function in~$\Span(\cF)$ (Corollary~\ref{corollary:zeroRedZeroApprox}), or containing some local $\epsilon$-approximation (Corollary~\ref{corollary:some0someeps}).

However, an interesting extension of Corollary~\ref{corollary:zeroRedZeroApprox} arises in the case~$\epsilon>0$, which applies for the case with~$0$-redundancies and~$\cF(\boldz)$ that is \textit{not} sufficiently descriptive in order to capture, but can only~$\epsilon$-approximate, those redundancies. The next corollary follows immediately from Theorem~\ref{theorem:epsCapture} since a~$0$-redundancy is also an~$\epsilon$ redundancy for every~$\epsilon>0$.

\begin{corollary}\label{corollary:deg5polys}
    Suppose there exists~$0$-redundancies~$R^{(1)},\ldots,R^{(t)}$ in~$X$ such that $Z_{j_i}=r^{(i)}(Z_{\cN_i^0})$ for every~$i\in[t]$. Further, suppose that for~$\cF(\boldz)$ and some~$\epsilon>0$ we have
    \begin{align}\label{equation:lowdegapproximation}
        \bE[(r^{(i)}(Z_{\cN_i^0})-\proj_{\cF(Z_{\cN_{i}^0})}r^{(i)}(Z_{\cN_i^0}))^2]\le \epsilon.
    \end{align}
    Then GCA can eliminate all the redundancies.
\end{corollary}

Corollary~\ref{corollary:deg5polys} states that if~$\cF(\boldz)$ is insufficiently descriptive in order to \textit{contain} the~$0$-redundancies in its span (as in Corollary~\ref{corollary:zeroRedZeroApprox}), but can reasonably \textit{approximate} them, then GCA still succeeds in eliminating the redundancy. 
One can extend Corollary~\ref{corollary:deg5polys} by employing the implicit function theorem globally (as in Corollary~\ref{corollary:zeroRedZeroApprox}), or locally (as in Corollary~\ref{corollary:some0someeps}), and the details are omitted. 

\begin{example}\label{example:deg5polys}
    Suppose that~$X$ satisfies some~$0$-redundancies such that the respective~$r^{(i)}$ from Corollary~\ref{corollary:deg5polys} are polynomials of degree~$5$. Further suppose that~$\cF(\boldz)$ are polynomials of degree~$3$ which can~$\epsilon$-approximate the~$r^{(i)}$'s in the sense of~\eqref{equation:lowdegapproximation}. Then GCA with~$\cF(\boldz)$ and~$\epsilon$ eliminates the redundancy.
\end{example}

A complexity tradeoff is evident from Corollary~\ref{corollary:deg5polys} and Example~\ref{example:deg5polys}. One can use simpler function family~$\cF(\boldz)$ with~$\epsilon>0$ instead of a more complex one with~$\epsilon=0$, and reduce the runtime of GCA. 
As long as the simpler~$\cF(\boldz)$ is able to $\epsilon$-approximate the redundancies, GCA will succeed. If the simpler~$\cF(\boldz)$ cannot~$\epsilon$-approximate some redundancy, it will be treated as non-redundant.

\begin{remark}
    In order to corroborate Corollaries~\ref{corollary:zeroRedZeroApprox}, \ref{corollary:some0someeps}, and~\ref{corollary:deg5polys}, GCA was implemented and tested on synthetic data sets successfully, and the results are reported in Section~\ref{section:experiments_GCA}.
\end{remark}
\section{Gram-Schmidt Functional Selection (GFS)}\label{section:GFS}
Throughout the feature selection algorithms (in this section and in the following one), we make use of the Hadamard product~$\otimes$, where for two vectors~$\boldx=(x_i)_{i=1}^n,\boldy=(y_i)_{i=1}^n$ we have~$\boldx\otimes\boldy=(x_iy_i)_{i=1}^n$; similarly~$\boldx^{\otimes 2}=(x_i^2)_{i=1}^n$.

In what follows, it is shown that a slight variant of GFR provides a feature selection algorithm with similar information-theoretic guarantees. 
Specifically, we replace the maximization step (which finds the eigenvector~$\boldnu_j$, Line~\ref{line:eigenvector_maximization} of Algorithm~\ref{algorithm:GFR}) by a discrete version, and similarly terminate if the resulting variance is smaller than the threshold~$\epsilon^2$. 
The output of the algorithm is a subset~$\cS_\epsilon\subseteq[d]$ of selected features. 
In feature selection algorithms, we only need to compute the alternative \textit{variance vectors}~$\boldsigma_j$ (a variance vector of a random variable is the diagonal of its covariance matrix) instead of covariance matrices~$\Sigma_j$ in their entirety.
Using~$\boldsigma_j$'s instead of~$\Sigma_j$'s allows some simplifications. For instance, if Algorithm~\ref{algorithm:orthogonalize} is used for feature \textit{selection} purposes, we make the following simple changes:
\begin{align*}
    &(\text{Line }\ref{line:covariance_matrix})~\text{Define }\Sigma_{j+1} = \bE[d_j(X) d_j^\intercal(X)] \\
    &\phantom{(\text{Line }\ref{line:covariance_matrix})}\rightarrow \text{Define: }\boldsigma_{j+1}=\bE[d_j(X)^{\otimes 2}].\\
    &(\text{Line }\ref{line:return})~\text{Return }\Sigma_{j+1}, \cT \rightarrow \text{Return }\boldsigma_{j+1}, \cT.
\end{align*}
In Algorithm~\ref{algorithm:GFS} below we propose \textit{Gram-Schmidt Functional Selection} (GFS), which at every step~$j$ uses the 
variance vector~$\boldsigma_j$ to select the most variant feature of~$d_{j-1}$. 
If the variance of this feature is less than some threshold~$\epsilon^2$, the algorithm stops, and otherwise it continues by specifying~$Z_j$ as the highest variance feature of $d_{j-1}$, which enables the next call to ``Orthogonalize.'' 
The equivalent algebraic counterpart of GFS, called \textit{Algebraic-GFS}, is presented in Algorithm~\ref{algorithm:GFS_algebraic} in Appendix~\ref{section:algebraic-GFS-GFA}.
The formal guarantees are very similar to those of GFR (Section~\ref{section:GFR}), and also require a discrete~$X$. 
Yet, experiments in Section~\ref{section:experiments} show significant gains in real-world setting in which the data is continuous.
The information-theoretic guarantee of GFS is as follows.

\begin{algorithm}[!h] 
\caption{Gram-Schmidt Functional Selection (GFS)}\label{algorithm:GFS}
\begin{algorithmic}[1]
\STATE {\bfseries Input:} A function family~$\cF(\boldz)$ and a threshold $\epsilon>0$.
\STATE {\bfseries Output:} Selected features~$s_1,\ldots,s_m$ ($m$ is a varying number that depends on~$\cF(\boldz)$, $\epsilon$, and~$X$)
\STATE {\bfseries Initialize:} $\boldsigma_1=\bE[X^{\otimes 2}]$.
\FOR{$j\gets1$ {\bfseries to} $d$}
\STATE Let~$s_j\!\triangleq\!\arg\max_{i\in[d]}\{\sigma_{j,i}\}_{i=1}^d$, where $\boldsigma_j\!=\!(\sigma_{j,i})_{i=1}^d$.
\IF{$\norm{\boldsigma_j}_\infty\le\epsilon^2$}\label{line:GFS_epsilon}
\STATE break. 
\ELSE
\STATE Define~$Z_j=X_{s_j}$.
\ENDIF
\STATE $\boldsigma_{j+1},\hat{\cF}(Z_1,\ldots,Z_j)=\text{Orthogonalize}(\hat{\cF}(Z_1,\ldots,Z_{j-1}),\cF(\boldz_{[j]})\setminus\cF(\boldz_{[j-1]}), \{Z_i\}_{i=1}^j)$
\ENDFOR
\end{algorithmic}
\end{algorithm}

\begin{theorem}\label{theorem:GFS}
Suppose~$X$ is over a discrete domain~$\cX^d$, for some~$\cX\subseteq\bR$, and let~$\cS_{\epsilon}=\{s_1,\ldots,s_m\}$ be the out of GFS whose input is a given~$\epsilon$ and~$\cF(\boldz)$. Then~$H(X|X_{\cS_\epsilon})\le dO(\epsilon)$.
\end{theorem}

The proof of Theorem~\ref{theorem:GFS} is very similar to that of Theorem~\ref{theorem:GFR}, and is given in full in Appendix~\ref{appendix:GFS}.

\section{Gram-Schmidt Feature Analysis (GFA)}\label{section:GFA}
As mentioned in Section~\ref{section:GCA}, Gram-Schmidt Component Analysis (GCA) can serve as an additional layer of redundancy removal beyond PCA. 
While PCA identifies if the data lies on a subspace (or closely so, if one wishes to ignore principal directions whose respective eigenvalue is small), GCA can further identify if nonlinear redundancies exist within that subspace. 
However, in GCA it is merely assumed that~$X=\sum_{i=1}^d\boldrho_iZ_i$, and the fact that the~$\boldrho_i$'s are the principal directions does not play a role other than them being orthonormal. 
Therefore, one can substitute the principal directions~$\{\boldrho_i\}_{i=1}^d$ in GCA by any other orthonormal basis of~$\bR^d$, and obtain similar results.

One such orthonormal basis which is of interest, other than the principal directions, is the standard basis.
By substituting the principal directions in GCA by the standard basis, the guarantees of GCA pertaining to redundancy elimination along the principal directions are substituted by similar redundancy elimination guarantees along the \textit{features themselves}.
Arguably, the ability to eliminate redundancy along the features themselves is also of great importance.
Therefore, in this section we present \textit{Gram-Schmidt Feature Analysis} (GFA), an extension of GCA which serves a different purpose. 
GFA is given in Algorithm~\ref{algorithm:GFA}, which is obtained directly from GCA by substituting principal directions by the standard basis, and covariance matrices by covariance vectors. 
To maintain notational consistency with GCA we denote~$Z_i=X_i$.
The corresponding equivalent algebraic formulation, denoted \textit{Algebraic-GFA}, is presented in Algorithm~\ref{algorithm:GFA_algebraic} in Appendix~\ref{section:algebraic-GFS-GFA}.
\begin{algorithm}[h] 
\caption{Gram-Schmidt Feature Analysis (GFA)}
\label{algorithm:GFA}
\begin{algorithmic}[1]
\STATE {\bfseries Input:} A function family~$\cF(\boldz)$ and a threshold~$\epsilon>0$.
\STATE {\bfseries Output:} A set~$\cS\subseteq[d]$ of indices of selected features.
\STATE {\bfseries Initialize:} $\boldsigma_1=\bE[X^{\otimes 2}]$ and~$\cS_0=\varnothing$.
\FOR{$j\gets1$ {\bfseries to} $d$}
\STATE Let $\cE_j=\{ i\vert \sigma_{j,i} <\epsilon \}$.
\IF{$\cE_j=[d]$}
\RETURN{}\hspace{-1mm}$\cS_{j-1}$.
\ENDIF
\STATE Let $s_j\triangleq\arg\max_{i\in[d]\setminus\cE_j}\{\sigma_{1,i}\}$ \\ and $\cS_j=\cS_{j-1}\cup\{s_j\}$. 
\STATE $\boldsigma_{j+1},\hat{\cF}(Z_{\cS_j})=\text{Orthogonalize}(\hat{\cF}(Z_{\cS_{j-1}}),\cF(\boldz_{[j]})\setminus\cF(\boldz_{[j-1]}), \{Z_{s_i}\}_{i=1}^j)$\\
\ENDFOR
\RETURN{}\hspace{-1mm}$\cS_{d}$.
\end{algorithmic}
\end{algorithm}

The proof of the following theorem can be easily verified.

\begin{theorem}\label{theorem:GFA}
    Gram-Schmidt Feature Analysis (GFA, Algorithm~\ref{algorithm:GFA}) satisfies Theorem~\ref{theorem:epsCapture}, as well as Corollaries~\ref{corollary:zeroRedZeroApprox}, \ref{corollary:some0someeps}, and~\ref{corollary:deg5polys}, while replacing each principal direction~$\boldrho_i$ by the standard basis vector~$\bolde_i$ (and in particular replacing~$Z_i=X^\intercal\boldrho_i$ by~$Z_i=X_i$).
\end{theorem}

\section{Literature review}\label{section:literature}
Unsupervised dimension reduction is an old topic, dating as far back as 1901. The literature in this area is vast, out of which we summarize the main key contributions below.
PCA is a well-known technique~\cite{pearson1901liii}; for a contemporary introduction see~\cite{theodoridis2015machine}. Multidimensional scaling (MDS)~\cite{cox2000multidimensional} is a class of methods (which includes Isomap~\cite{tenenbaum2000global}) whose objective is to maximize the scatter of the projection, under the rationale that doing so would yield the most informative projection. 
There exists a duality between PCA and MDS where scatter is measured by Euclidean distance~\cite{torgerson1952multidimensional,cox2000multidimensional}. 
Maximum Autocorrelation Factors (MAF) and Slow Feature Analysis are linear feature extraction methods tailored to multivariate time series data \cite{larsen2002decomposition,wiskott2002slow}. 
Locality Preserving Projections (LPP) aim to find a linear transformation of the data that preserves the local pairwise relationships among data points~\cite{he2003locality}. 
Locally Linear Embedding (LLE)~\cite{roweis2000nonlinear} finds representations by preserving local linear relationships between adjacent points. Independent component analysis (ICA)~\cite{stone2004independent} is used to separate a multivariate signal into independent, non-Gaussian components, but is also popular for feature extraction. 
t-Distributed Stochastic Neighbor Embedding (t-SNE) is a nonlinear technique widely used for visualizing high-dimensional data.
It minimizes the divergence between two probability distributions: one representing pairwise similarities of data points in the original space and the other in the reduced space. 
This method excels at preserving local structures and revealing intricate patterns, making it an invaluable tool for exploring and interpreting complex datasets in two or three dimensions~\cite{hinton2002stochastic}. 
Uniform Manifold Approximation and Projection (UMAP) is a dimensionality reduction technique that can be used for visualization purposes similar to t-SNE, as well as for general nonlinear dimensionality reduction tasks. 
It is particularly good at preserving both local and global structure of the data. 
UMAP is based on three key assumptions: the data is uniformly distributed on a Riemannian manifold, the Riemannian metric is locally constant or can be approximated as such, and the manifold is locally connected \cite{mcinnes2018umap}. 
Autoencoders are hourglass shaped neural networks usually trained to minimize the Euclidean deviation of their input from their output; the output of the narrowest intermediate layer is the ensuing nonlinear dimension reduction mechanism. 
Their ability to learn nonlinear transformations makes them a powerful tool for reducing the dimensionality of complex datasets~\cite{hinton2006reducing}. 

Main techniques in the area include Laplacian score~\cite{he2005laplacian,huang2018manifold}, Fisher score~\cite{duda2006pattern,urbanowicz2018relief}, and trace ratio~\cite{nie2008trace}, which assess the importance of features by their ability to preserve data similarity; 
Multi-cluster feature selection~\cite{cai2010unsupervised,wang2021multi}, nonnegative discriminative feature selection~\cite{li2012unsupervised}, and unsupervised discriminative feature selection~\cite{yang2011l2,nie2020subspace}, which assume sparsity and employ lasso type minimization; and 
minimal-redundancy-maximal-relevance~\cite{peng2005feature}, mutual information based feature selection~\cite{battiti1994using}, and fast correlation based filter~\cite{yu2003feature}, which maximize mutual information between the original and selected features.

Ref.~\cite{dash1999handling} proposed to use information entropy to evaluate the importance of features, and used the trace criterion to select the most relevant feature subset. Ref.~\cite{vandenbroucke2000unsupervised} proposes to use a competitive learning to divides the original feature into subsets, and the subset which minimizes dispersion is selected. Ref.~\cite{alibeigi2011unsupervised} analyzed the distribution of the feature data by use of the probability density of different features; a feature is selected by the data distribution relation among the other features. Ref.~\cite{mitra2002unsupervised} uses the maximum information compression index to measure the similarity between features. Ref.~\cite{zhou2015unsupervised} first calculates the maximal information coefficient for each attribute pair to construct an attribute distance matrix, and then clusters all attributes and selects features according to their cluster affiliation. Some further techniques use the Laplacian score~\cite{padungweang2009univariate,he2005laplacian,saxena2010evolutionary}, which measures the local topology of the data clusters. 

Additionally, several generalizations of PCA were proposed in the literature.
First, kernel PCA~\cite[Ch.~19.9.1]{theodoridis2015machine} is a popular method for \textit{nonlinear} dimension reduction, in which data is embedded into a higher dimensional space using an embedding function~$\Phi:\bR^d\to\bR^D$, which is in general not known a priori. 
Then, principal directions are found in~$\bR^D$, aided by the kernel trick to reduce computational complexity.
The resulting dimension reduction is obtained by projecting data transformed by~$\Phi$ on the principal directions in~$\bR^D$; since each entry in~$\bR^D$ is a nonlinear function of the entries in~$\bR^d$, this yields a \textit{nonlinear} dimension reduction method.
Therefore, kernel PCA is different from all methods proposed in this paper, which are linear.
Yet, for completeness we discuss several similarities and differences between our approach and kernel PCA.
In terms of similarities, both methods seek high variance directions in a space transformed by a nonlinear function---the function~$\Phi$ mentioned above in the case of kernel PCA, and the function~$\boldz\mapsto (f(\boldz))_{f\in \cF}$ in our case---in which each entry is a nonlinear function of the entries in~$\bR^d$.
However, our methods rely on devising a \textit{linear} transform~$V$ in~$\bR^d$ which is applied before the function~$\boldz\mapsto(f(\boldz))_{f\in\cF}$, i.e., in a sense our methods are based on exploring high variance directions in a space transformed by~$\boldz\mapsto(f(\boldz^\intercal V))_{f\in\cF}$.
More so, the linear transform~$V$ is not determined a priori (unlike kernel PCA, in which~$\Phi$ must be determined a priori), but discovered throughout the algorithm---each row~$\boldnu_i$ of~$V$ is determined according to variance properties in the space~$(f(X^\intercal\boldnu_1,\ldots,X^\intercal\boldnu_{i-1}))_{f\in\cF_{i-1}}$ (where~$\cF_i\subseteq \cF$ is a subset which depends only on the first~$i-1$ variables).
The resulting dimension reduction in our case is given by projecting the original data on the rows of~$V$, which is fundamentally different from kernel PCA where it is given by computing~$\Phi$ and projecting over the principal directions in~$\bR^D$.

Second, generalized PCA~\cite{vidal2005generalized} is a framework which assumes the data lies on a \textit{union} of subspaces of~$\bR^d$. Using De-Morgan's laws, it is shown~\cite[Thm.~2]{vidal2005generalized} that data which lies on a union of~$\ell$ subspaces also lies in the kernel of a homogeneous polynomial system. The number of polynomials in this system is the product of the co-dimensions of the subspaces, and all polynomials only contain monomials of degree exactly~$\ell$. Then, using differential calculus, the bases of these subspaces are found. Whenever~$\cF(\boldz)$ is a set of monomials, the definition of redundancy in Section~\ref{section:GCA} (also Section~\ref{section:GFA})
implies that $X$ satisfies a certain polynomial system as well, albeit not necessarily homogeneous nor containing only monomials of a certain degree. The similarities to generalized PCA do not seem to extend beyond that, as generalized PCA framework addresses only a very particular form of redundancy.

Finally, similar to our work, function-space orthogonality is also employed in the theory of~\textit{Principal Inertia Components} (PIC)~\cite{hsu2021generalizing,du2017principal}, where it is used to define the \textit{principal functions} of a pair of random variables~$(X,Y)$. 
These principal functions can then be used to compute functions of~$X$ from samples of~$Y$, and vice versa.
While evaluations of the principal functions can be seen as a dimension reduction method, there are several substantial difference between the PIC paradigm and this paper: (a) the principal functions are highly nonlinear and in practice computed via a neural network (as opposed to our linear methods); (b) PIC is tailored for simultaneous analysis of two variables (rather than one variable in our methods); and (c) in our methods the orthogonal functions are a tool in the dimension reduction process (e.g., to define the alternative covariance matrices), and not the dimension reduction itself as in PIC.

\section{Experimental Results}\label{section:experiments}
In this section the theoretical results and performance of the proposed algorithms are evaluated through simulation studies over synthetic and real-world datasets. 
Below we detail the methods and the rationale behind each set of experiments.

Section~\ref{section:experiments_GFR_residual_variance_reduction} supports the claim that GFR significantly reduces the data's residual variance in comparison with PCA.
Intuitively, residual variance captures how variant the data is after extracting features. 
Formally, in PCA residual variance after extracting~$m$ features is captured by the~$(m+1)$'th eigenvalue of the covariance matrix~$\Sigma_{m+1}$, and in GFR it is the largest eigenvalue of the alternative covariance matrix~$\Sigma_{m+1}$.
In both cases, residual variance is the largest eigenvalue of the covariance matrix, after setting to zero the parts of the data which were extracted so far. 
We show that for a similar threshold~$\epsilon^2$, GFR achieves similar residual variance using significantly fewer extracted features in comparison to PCA, by almost an order of magnitude in certain cases. 

In section~\ref{section:experiments_GFR_classification_accuracy} we evaluate GFR for classification tasks over benchmark datasets and compare the results to other well-known unsupervised feature extraction algorithms: Kernel PCA, FastICA, Locality Preserving Projections (LPP), UMAP, and autoencoders; this is a representative list of the best known linear methods, and the most popular nonlinear ones.
We use each method for feature extraction while disregarding the labels. The labels are later re-attached, a classifier is trained over the extracted features, and the resulting accuracies are compared.

In section~\ref{section:experiments_GCA}, we study the performance of GCA based on its success in removing all redundant features among principal components in synthetic datasets. 
Although its correctness is formally proven, we study the effect of employing empirical covariance matrices, and its connection to dataset size. 
We show that GCA can indeed remove all redundant principal components from the data, given moderately large datasets.
On the other hand, GCA's performance degrades for smaller datasets, particularly in ones with higher-degree multilinear polynomial redundancy.

In section~\ref{section:experiments_GFS_selected_features} we turn to experiment on our feature selection techniques, and similar to the experiments for feature extraction, we measure variance reduction and classification accuracy. 
We begin by showing that selecting higher degree multilinear polynomials as the redundant functions~$\cF(\boldz)$ in GFS results in a significant reduction  in maximum variance among the entries of~$d_j(X)$ in each iteration (i.e.,~$\max\{\var[d_j(X)_i]\}_{i=1}^d$).
We proceed in Section~\ref{section:experiments_GFS_classification_accuracy} to validate the performance of GFS for classification tasks on the benchmark datasets. 
The numerical experiments show superior performance of GFS in comparison to other algorithms. 

In Section~\ref{section:experiments_GFS_vs_UFFS}, We corroborate GFS's superiority over Unsupervised Fourier Feature Selection (UFFS)~\cite{heidari2022sufficiently} in terms of running time, the ability to capture redundant features, and classification accuracy, over synthetic datasets. 
Finally, we study the performance of GFA over synthetic datasets in section~\ref{section:experiments_GFA}.

The experiments were performed on a laptop with Intel(R) Core(TM) i9-9880H CPU @ 2.30GHz and 64GB RAM, and the source code is available at~\url{https://github.com/byaghooti/Gram_schmidt_feature_extraction}.

\subsection{Residual Variance Reduction of GFR}\label{section:experiments_GFR_residual_variance_reduction}
To support our claim that GFR results in significant variance reduction in comparison to PCA, we apply GFR to benchmark datasets taken from UCI repository \cite{Dua:2017} and \cite{li2017feature}. The properties of the tested benchmark datasets are provided in Table~\ref{table:properties-of-benchmark-datasets}. 
\begin{table}[!ht]
\centering
\begin{tabular}{ l | c c c c c } 
\toprule
\textbf{Dataset} & \textbf{USPS} & \textbf{MNIST} & \textbf{COIL-20} & \textbf{Musk} & \textbf{Credit Approval} \\
\midrule
\rowcolor{gray!20}
Features & 256 & 784 & 1024 & 166 & 15 \\
Samples & 7291 & 60000 & 1440 & 6597 & 690 \\
\bottomrule
\end{tabular}
\caption{Properties of the tested benchmark datasets.}
\label{table:properties-of-benchmark-datasets}
\end{table}

We compare PCA against GFR, where~$\cF(\boldz)$ in GFR is set as multilinear polynomials of degrees up to~$2$,~$3$, or~$4$,
and the results are shown in Table~\ref{table:number-of-extracted-features-benchmark-datasets}. 
We compare the amount of residual variance left in the dataset with the number of extracted features. 
Specifically, Table~\ref{table:number-of-extracted-features-benchmark-datasets} shows the number of features that had to be extracted to achieve a certain amount of residual variance. 
\begin{table}[h]
\centering
\subfloat[USPS Dataset]{%
\begin{tabular}{c | c c c c c}
    \toprule
    \multirow{2}{*}{\makecell{Degree of \\ Polynomial}} & \multicolumn{5}{c}{$\epsilon^2$} \\
    \cmidrule{2-6}
     & 0.1 & 0.15 & 0.2 & 0.5 & 1 \\
    \midrule
    \rowcolor{gray!20}
    1 (PCA) & 208 & 130 & 113 & 70 & 45 \\
    2 & 61 & 52 & 47 & 30 & 21 \\
    \rowcolor{gray!20}
    3 & 29 & 27 & 25 & 20 & 16 \\
    4 & 20 & 19 & 19 & 17 & 14 \\
    \bottomrule
\end{tabular}
}
\hspace{45pt}
\subfloat[MNIST Dataset]{%
\begin{tabular}{c | c c c c c c}
    \toprule
    \multirow{2}{*}{\makecell{Degree of \\ Polynomial}} & \multicolumn{6}{c}{$\epsilon^2$} \\
    \cmidrule{2-7}
     & $0.01$ & $0.05$ & $0.1$ & $0.2$ & $0.5$ & $0.75$ \\
    \midrule
    \rowcolor{gray!20}
    1 (PCA) & 696 & 576 & 451 & 344 & 233 & 191 \\
    2 & 205 & 145 & 114 & 84 & 53 & 43 \\
    \rowcolor{gray!20}
    3 & 55 & 49 & 45 & 39 & 31 & 26 \\
    4 & 29 & 28 & 27 & 25 & 22 & 20 \\
    \bottomrule
\end{tabular}
}
\\
\subfloat[COIL-20 Dataset]{%
\begin{tabular}{c | c c c c c c}
    \toprule
    \multirow{2}{*}{\makecell{Degree of \\ Polynomial}} & \multicolumn{6}{c}{$\epsilon^2$} \\
    \cmidrule{2-7}
     & 0.01 & 0.02 & 0.05 & 0.1 & 0.2 & 0.5 \\
    \midrule
    \rowcolor{gray!20}
    1 (PCA) & 577 & 477 & 353 & 264 & 189 & 120 \\
    2 & 50 & 47 & 42 & 38 & 33 & 27 \\
    \rowcolor{gray!20}
    3 & 21 & 20 & 19 & 18 & 17 & 15 \\
    4 & 14 & 14 & 14 & 14 & 13 & 12 \\
    \bottomrule
\end{tabular}
}
\hspace{20pt}
\subfloat[Musk Dataset]{%
\begin{tabular}{c | c c c c c c c}
    \toprule
    \multirow{2}{*}{\makecell{Degree of \\ Polynomial}} & \multicolumn{7}{c}{$\epsilon^2$} \\
    \cmidrule{2-8}
     & 0.01 & 0.02 & 0.05 & 0.1 & 0.2 & 0.5 & 1 \\
    \midrule
    \rowcolor{gray!20}
    1 (PCA) & 122 & 104 & 83 & 68 & 52 & 37 & 26 \\
    2 & 51 & 46 & 38 & 32 & 27 & 20 & 16 \\
    \rowcolor{gray!20}
    3 & 27 & 26 & 23 & 21 & 19 & 15 & 13 \\
    4 & 20 & 19 & 18 & 17 & 16 & 14 & 12 \\
    \bottomrule
\end{tabular}
}
\\
\subfloat[Credit Approval Dataset]{%
\begin{tabular}{c | c c c}
    \toprule
    \multirow{2}{*}{\makecell{Degree of \\ Polynomial}} & \multicolumn{3}{c}{$\epsilon^2$} \\
    \cmidrule{2-4}
     & $0.5$ & $0.75$ & $1$ \\
    \midrule 
    \rowcolor{gray!20}
    1 (PCA) & 14 & 11 & 7 \\
    2 & 11 & 8 & 6 \\
    \rowcolor{gray!20}
    3 & 9 & 8 & 6 \\
    4 & 9 & 8 & 5 \\
    \bottomrule
\end{tabular}
}
\caption{Number of extracted features by GFR with multilinear polynomials on benchmark datasets for different values of the threshold~$\epsilon^2$.}
\label{table:number-of-extracted-features-benchmark-datasets}
\end{table}

Table~\ref{table:number-of-extracted-features-benchmark-datasets} clearly shows a drastic reduction of residual variance as the degree of the multilinear polynomials increases, showcasing that the redundancies in these datasets are well described by them. We emphasize the particular power of multilinear polynomials of degree~$4$, achieving up to $6$-$10$ times fewer extracted features in USPS (as opposed to~$2$-$5$ times in other cases with USPS), up to~$24$ fewer features for MNIST, and up to~$6$ times fewer features for Musk.

We also compare GFR with Kernel PCA, where~$\cF(\boldz)$ in GFR consists of multilinear polynomials of degrees up to 4 and general polynomials of degrees up to 4, and Kernel PCA is used with a degree~$4$ polynomial kernel.
Kernel PCA maps the features to a higher-dimensional space using a nonlinear kernel and computes the principal components in the transformed space, resulting in a feature representation with a different statistical distribution compared to the original space.
Thus, we cannot directly compare the variance of the extracted features from GFR and Kernel PCA. 
Instead, when extracting~$m$ features, we compare the ratio of the variance of the~$(m+1)$'th feature to the variance of the first extracted feature, i.e., $\var(Z_{m+1})/\var(Z_{1})$, where~$Z_1$ and~$Z_{m+1}$ are the first and $(m+1)$'th extracted features, respectively. This ratio normalizes the variance of the features with respect to the first extracted feature.
The lower this ratio, the more effectively the algorithm removes redundant features, as it indicates how much of the variance of the first extracted feature remains in the data after extracting~$m$ features.
Table~\ref{table:residual-variance-GFR-KPCA} shows this ratio for different numbers of extracted features on benchmark datasets.  
The results show that, in most benchmark datasets (except Credit Approval) GFR achieves a lower ratio, demonstrating its ability to remove redundant features better than Kernel PCA.

\begin{table}[h]
\centering
\subfloat[USPS Dataset]{
\begin{tabular}{c | c c c c}
    \toprule
    \multirow{2}{*}{Method} & \multicolumn{4}{c}{Number of Extracted Features} \\
    \cmidrule{2-5}
     & 20 & 19 & 17 & 14 \\
    \midrule
    \rowcolor{gray!20}
    \makecell{GFR with \\ multilinear polynomial} & 0.26 & 0.52 & 1.32 & 2.61 \\
    \makecell{GFR with \\ polynomial} & 0.08 & 0.09 & 0.11 & 1.03 \\
    \rowcolor{gray!20}
    Kernel PCA & 7.95 & 8.50 & 9.92 & 11.59 \\
    \bottomrule
\end{tabular}
}
\hspace{45pt}
\subfloat[MNIST Dataset]{
\begin{tabular}{c | c c c c c c}
    \toprule
    \multirow{2}{*}{Method} & \multicolumn{6}{c}{Number of Extracted Features} \\
    \cmidrule{2-7}
    & 29 & 28 & 27 & 25 & 22 & 20 \\
    \midrule
    \rowcolor{gray!20}
    \makecell{GFR with \\ multilinear polynomial} & 0.02 & 0.12 & 0.24 & 0.49 & 1.22 & 1.82 \\
    \makecell{GFR with \\ polynomial} & 0.01 & 0.04 & 0.17 & 0.34 & 0.91 & 1.48 \\
    \rowcolor{gray!20}
    Kernel PCA & 0.59 & 0.61 & 0.62 & 0.68 & 0.90 & 0.91 \\
    \bottomrule
\end{tabular}
}
\\
\subfloat[COIL-20 Dataset]{
\begin{tabular}{c | c c c}
    \toprule
    \multirow{2}{*}{Method} & \multicolumn{3}{c}{Number of Extracted Features} \\
    \cmidrule{2-4}
     & 14 & 13 & 12 \\
    \midrule
    \rowcolor{gray!20}
    \makecell{GFR with \\ multilinear polynomial} & 0.45 & 0.91 & 2.27 \\
    \makecell{GFR with \\ polynomial} & 0.55 & 0.063 & 0.087 \\
    \rowcolor{gray!20}
    Kernel PCA & 15.08 & 16.15 & 16.62 \\
    \bottomrule
\end{tabular}
}
\hspace{20pt}
\subfloat[Musk Dataset]{
\resizebox{\columnwidth}{!}{\begin{tabular}{c | c c c c c c c}
    \toprule
    \multirow{2}{*}{Method} & \multicolumn{7}{c}{Number of Extracted Features} \\
    \cmidrule{2-8}
     & 20 & 19 & 18 & 17 & 16 & 14 & 12 \\
    \midrule
    \rowcolor{gray!20}
    \makecell{GFR with \\ multilinear\\ polynomial} & 0.02 & 0.04 & 0.10 & 0.21 & 0.41 & 1.03 & 2.06 \\
    \makecell{GFR with \\ polynomial} & 0.01 & 0.02 & 0.04 & 0.051 & 0.057 & 0.31 & 0.76 \\
    \rowcolor{gray!20}
    Kernel PCA & 8.71 & 8.93 & 10.02 & 10.45 & 11.22 & 12.93 & 14.17 \\
    \bottomrule
\end{tabular}
}}
\\
\subfloat[Credit Approval Dataset]{
\begin{tabular}{c | c c c}
    \toprule
    \multirow{2}{*}{Method} & \multicolumn{3}{c}{Number of Extracted Features} \\
    \cmidrule{2-4}
     & 9 & 8 & 5 \\
    \midrule
    \rowcolor{gray!20}
    \makecell{GFR with \\ multilinear polynomial} & 1.87 & 2.81 & 3.74 \\
    \makecell{GFR with \\ polynomial} & 1.69 & 2.05 & 2.56 \\
    \rowcolor{gray!20}
    Kernel PCA & 0.07 & 0.10 & 0.21 \\
    \bottomrule
\end{tabular}
}
\caption{Comparison of the residual variance reduction of GFR with multilinear polynomials of degree up to 4 and  polynomials of degree up to 4 and Kernel PCA with a polynomial kernel of degree 4 on benchmark datasets. 
The entries represent the ratio between the variance of the~$(m+1)$'th extracted feature and that of the first extracted feature, i.e.,$\var(Z_{m+1})/\var(Z_{1})$, for different number of extracted features~$m$.}
\label{table:residual-variance-GFR-KPCA}
\end{table}

\subsection{Classification Accuracy of GFR}\label{section:experiments_GFR_classification_accuracy}
\subsubsection{Classification Accuracy of GFR in Benchmark Datasets}
We validate GFR for classification tasks on the benchmark datasets in Table~\ref{table:properties-of-benchmark-datasets}. 
In Table~\ref{table:GFR-classification-accuracy-benchmark-datasets}, we apply GFR with multilinear polynomials of degree up to~$4$ 
against Kernel PCA, FastICA, PCA, and LPP.
The classification accuracy in the benchmark datasets shows overall superior performance of GFR in comparison to other algorithms. 
It is worth mentioning that as we were expecting, the classification accuracy of GFR is higher than PCA in all the experiments in the mentioned benchmark datasets. 

We used a support vector machine classifier with radial basis function as kernel. 
$5$-fold cross-validation on the entire datasets is used to validate the performance of the algorithms. 
In multi-class classification tasks, we used One-vs-Rest strategy. 
To implement Kernel PCA, FastICA, LPP, and UMAP, we used \texttt{KernelPCA}, \texttt{FastICA}, \texttt{LocalityPreservingProjection}, and \texttt{UMAP} from \texttt{sklearn.decomposition}, \texttt{lpproj} packages, and \texttt{umap-learn}, respectively. In all classification tasks utilizing Kernel PCA, polynomial kernels are employed, whose degree matches that of the function family~$\cF$.

Since authoencoders and UMAP have some hyperparameters, we tuned them to find some rough optimum points, and then changed their values around those points to show their behavior for different values of hyperparameters. Specifically, in autoencoders, we change the number of hidden layers, the number of nodes in hidden layers, and the number of epochs. For UMAP, we change \texttt{min\_dist}, which is the effective minimum distance between embedded points, and \texttt{n\_neighbors}, which is the size of the local neighborhood (in terms of the number of neighboring sample points) used for manifold approximation. For more details about these hyperparameters, we refer readers to~\cite{mcinnes2018umap}.

Figures~\ref{fig:GFR-UMAP-classification-accuracy-benchmark-datasets}-\ref{fig:GFR-Autoencoder-classification-accuracy-benchmark-datasets} compare the classification accuracy of GFR with UMAP and autoencoders, respectively. 
The numerical results show that GFR outperforms both UMAP and autoencoder in terms of classification accuracy. 
Note that on each box, the central mark indicates the median, and the bottom and top edges of the box indicate the $25$\textsuperscript{th} and $75$\textsuperscript{th} percentiles, respectively. 
The whiskers extend to the most extreme data points not considered outliers, and the outliers are plotted individually using the ``\red{+}'' marker symbol. Another important result of these experiments is the sensitivity of the autoencoders and UMAP with respect to the hyperparameters, which results in a higher uncertainty of these methods.
As we see in these tables and figures, in all the experiments, the classification accuracy of GFR with multilinear polynomials of degree up to~$4$ over USPS, MNIST, and COIL-20 datasets is higher than all other algorithms.   
The results for the Musk data set demonstrate that none of the GFR, Kernel PCA, FastICA, PCA, and LPP are better than others in most cases, but the classification accuracy of GFR is higher than the average classification accuracy of UMAP and autoencoder.
Finally, from the simulation results for the Credit Approval dataset, it is obvious that GFR's performance is better than other algorithms in most cases. Only in cases with lower extracted features, LPP performs better than GFR.

\begin{table}[h]
\centering
\subfloat[USPS Dataset]{%
\begin{tabular}{c | c c c c}
    \toprule
    \multirow{2}{*}{Method} & \multicolumn{4}{c}{Number of Extracted Features} \\
    \cmidrule{2-5}
     & 20 & 19 & 17 & 14 \\
    \midrule
    \rowcolor{gray!20}
    GFR & \textbf{97.02} & \textbf{97} & \textbf{96.86} & \textbf{96.04} \\
    Kernel PCA & 96.90 & 96.69 & 96.52 & 95.86 \\
    \rowcolor{gray!20}
    FastICA & 96.14 & 95.77 & 95.61 & 95.54 \\
    PCA & 96.18 & 95.89 & 95.72 & 95.33 \\
    \rowcolor{gray!20}
    LPP & 93.54 & 93.48 & 93.29 & 92.32 \\
    \bottomrule
\end{tabular}
}
\hspace{45pt}
\subfloat[MNIST Dataset]{%
\begin{tabular}{c | c c c c c c}
    \toprule
    \multirow{2}{*}{Method} & \multicolumn{6}{c}{Number of Extracted Features} \\
    \cmidrule{2-7}
    & 29 & 28 & 27 & 25 & 22 & 20 \\
    \midrule
    \rowcolor{gray!20}
    GFR & \textbf{96.91} & \textbf{96.90} & \textbf{96.87} & \textbf{96.79} & \textbf{96.63} & \textbf{96.46} \\
    Kernel PCA & 93.11 & 93.02 & 92.87 & 92.72 & 92.69 & 92.58 \\
    \rowcolor{gray!20}
    FastICA & 95.01 & 94.98 & 94.94 & 94.85 & 94.52 & 94.35 \\
    PCA & 90.11 & 90.02 & 89.54 & 89.12 & 88.77 & 88.35 \\
    \rowcolor{gray!20}
    LPP & 87.13 & 86.95 & 86.8 & 86.31 & 86.03 & 85.94 \\
    \bottomrule
\end{tabular}
}
\\
\subfloat[COIL-20 Dataset]{%
\begin{tabular}{c | c c c}
    \toprule
    \multirow{2}{*}{Method} & \multicolumn{3}{c}{Number of Extracted Features} \\
    \cmidrule{2-4}
     & 14 & 13 & 12 \\
    \midrule
    \rowcolor{gray!20}
    GFR & \textbf{97.99} & \textbf{97.92} & \textbf{97.85} \\
    Kernel PCA & 96.67 & 96.81 & 96.25 \\
    \rowcolor{gray!20}
    FastICA & 95.63 & 95.28 & 92.85 \\
    PCA & 96.04 & 95.13 & 92.68 \\
    \rowcolor{gray!20}
    LPP & 90.56 & 90.35 & 89.03 \\
    \bottomrule
\end{tabular}
}
\hspace{20pt}
\subfloat[Musk Dataset]{%
\resizebox{\columnwidth}{!}{\begin{tabular}{c | c c c c c c c}
    \toprule
    \multirow{2}{*}{Method} & \multicolumn{7}{c}{Number of Extracted Features} \\
    \cmidrule{2-8}
     & 20 & 19 & 18 & 17 & 16 & 14 & 12 \\
    \midrule
    \rowcolor{gray!20}
    GFR & 83.32 & 83.24 & 83.57 & \textbf{83.1} & 82.21 & 82.24 & \textbf{82.21} \\
    Kernel PCA & \textbf{83.97} & \textbf{84.31} & 83.99 & 82.79 & 81.37 & 79.25 & 78.87 \\
    \rowcolor{gray!20}
    FastICA & 83.24 & 82.9 & \textbf{84.44} & 82.65 & \textbf{83.78} & \textbf{82.49} & 82.1 \\
    PCA & 81.04 & 81.04 & 81.81 & 82.55 & 82.12 & 81.34 & 80.68 \\
    \rowcolor{gray!20}
    LPP & 82.67 & 82.69 & 81.49 & 81.49 & 82.57 & 81.78 & 81.1 \\
    \bottomrule
\end{tabular}
}}
\\
\subfloat[Credit Approval Dataset]{%
\begin{tabular}{c | c c c}
    \toprule
    \multirow{2}{*}{Method} & \multicolumn{3}{c}{Number of Extracted Features} \\
    \cmidrule{2-4}
     & 9 & 8 & 5 \\
    \midrule
    \rowcolor{gray!20}
    GFR & \textbf{84.78} & \textbf{84.49} & 81.59 ($2^{nd}$ best) \\
    Kernel PCA & 81.30 & 79.13 & 80 \\
    \rowcolor{gray!20}
    FastICA & 82.46 & 82.46 & 81.44 \\
    PCA & 82.46 & 82.46 & 81.44 \\
    \rowcolor{gray!20}
    LPP & 84.34 & 84.34 & \textbf{82.31} \\
    \bottomrule
\end{tabular}
}
\caption{Classification accuracy (\%) of GFR with multilinear polynomials of degree up to 4, Kernel PCA, FastICA, PCA, and LPP on benchmark datasets.}
\label{table:GFR-classification-accuracy-benchmark-datasets}
\end{table}

\begin{figure}[h]
\centering
\begin{minipage}[b]{0.24\textwidth}
    \centering
    \includegraphics[width=\textwidth]{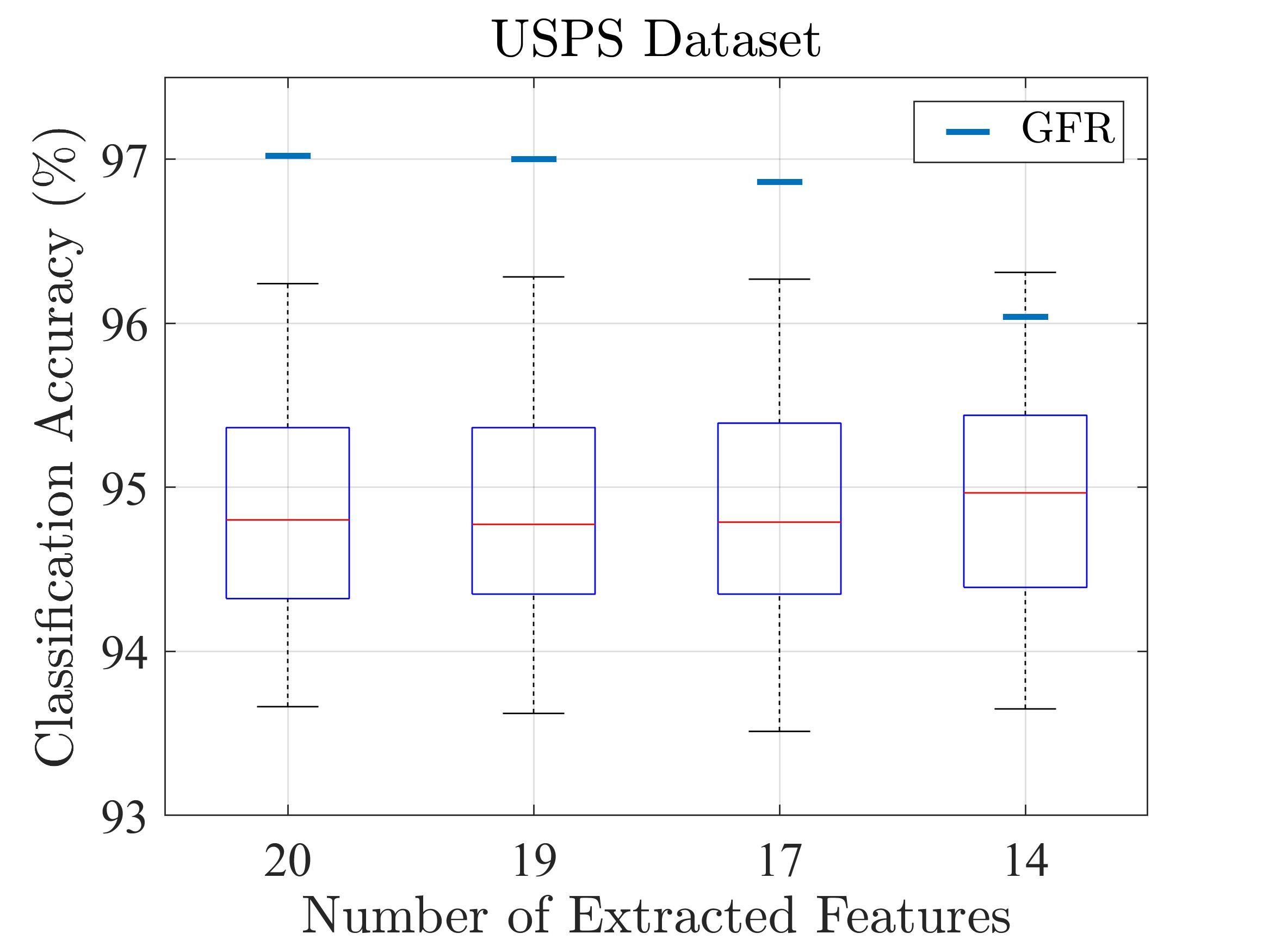}
\end{minipage}
\hfill
\begin{minipage}[b]{0.24\textwidth}
    \centering
    \includegraphics[width=\textwidth]{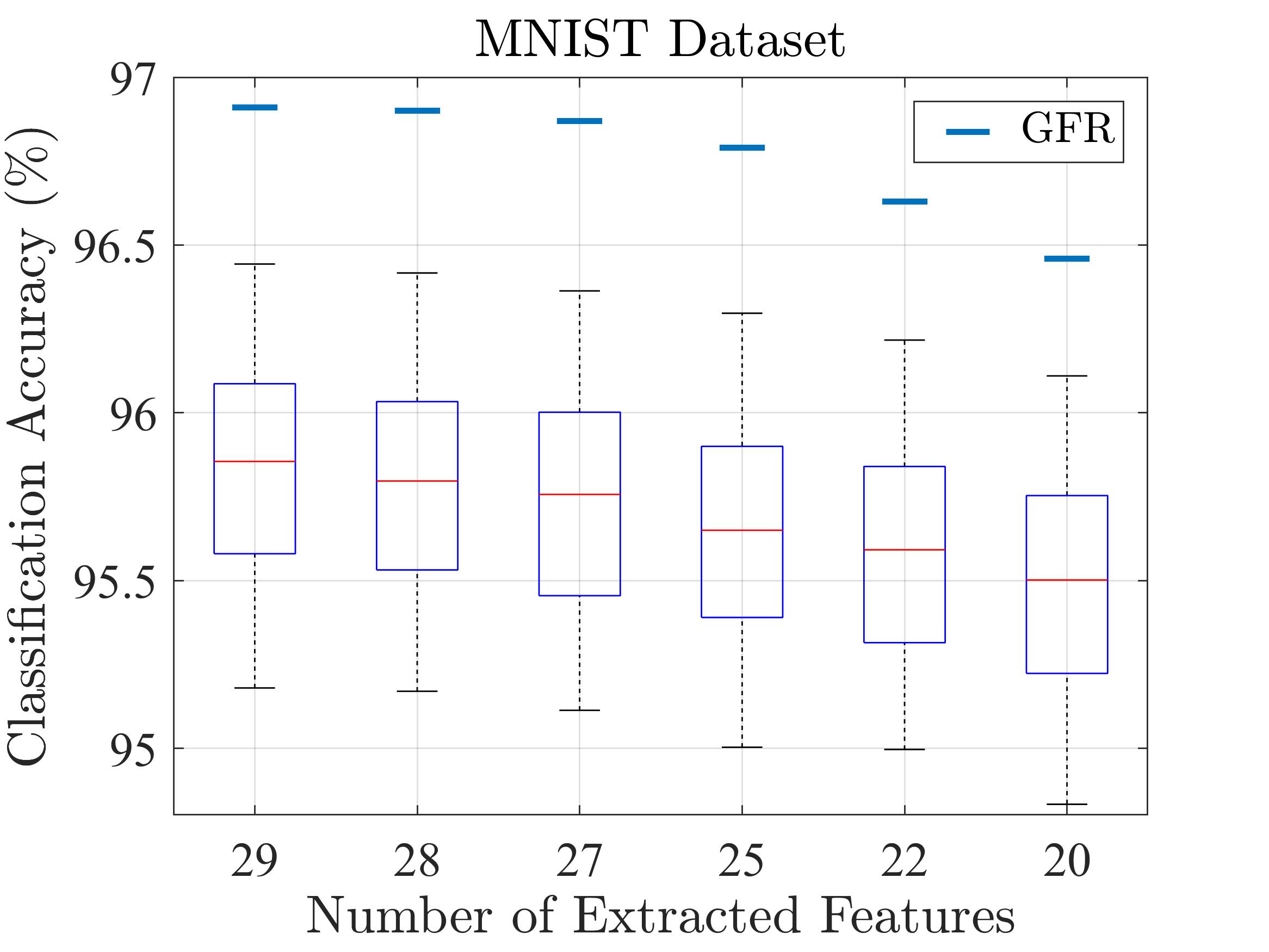}
\end{minipage}
\hfill
\begin{minipage}[b]{0.24\textwidth}
    \centering
    \includegraphics[width=\textwidth]{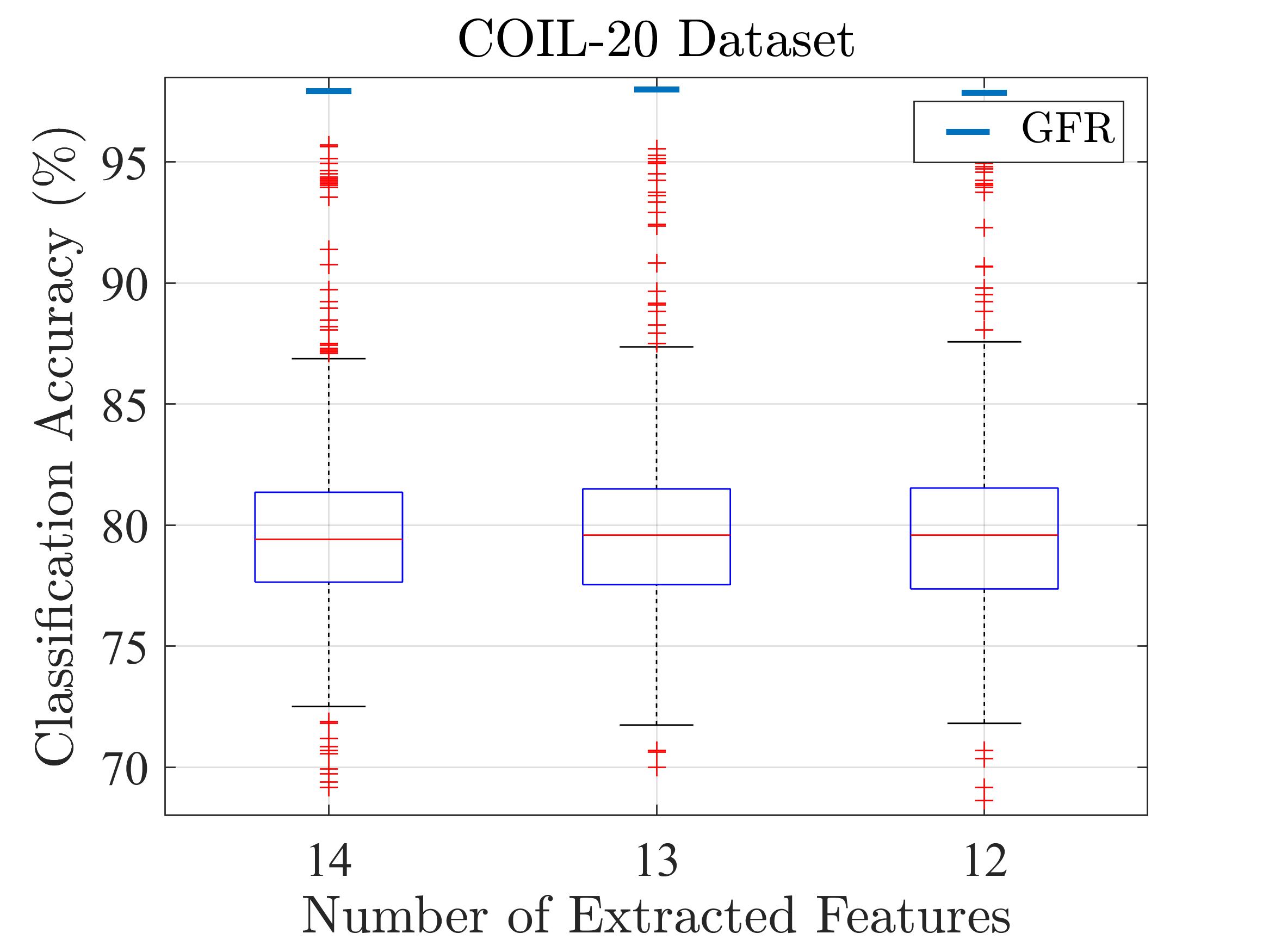}
\end{minipage}
\hfill
\begin{minipage}[b]{0.24\textwidth}
    \centering
    \includegraphics[width=\textwidth]{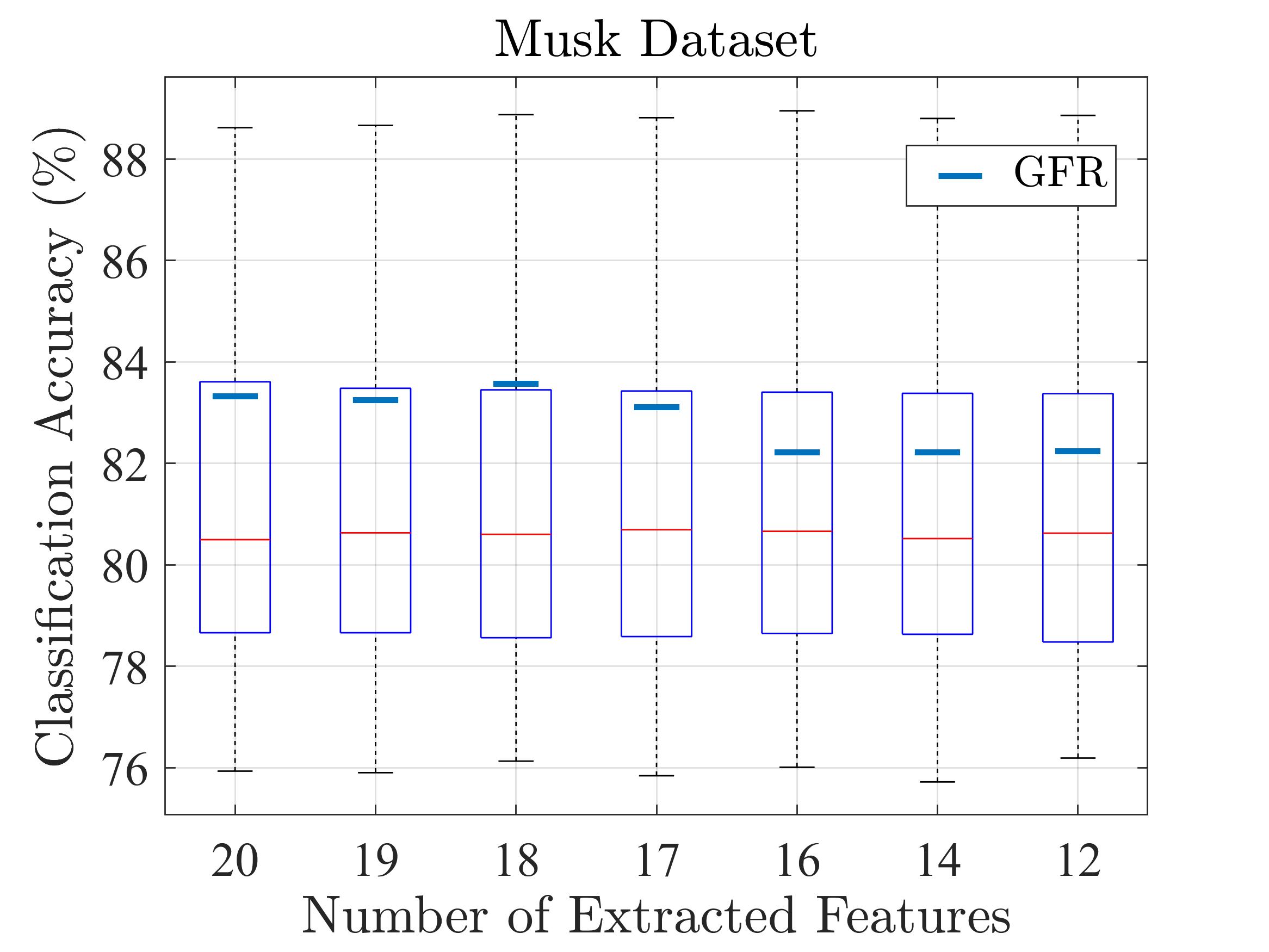}
\end{minipage}
\caption{Classification accuracy (\%) of GFR with multilinear polynomials of degree up to~$4$ versus UMAP on benchmark datasets.}
\label{fig:GFR-UMAP-classification-accuracy-benchmark-datasets}
\end{figure}

\begin{figure}[h]
\centering
\begin{minipage}[b]{0.24\textwidth}
    \centering
    \includegraphics[width=\textwidth]{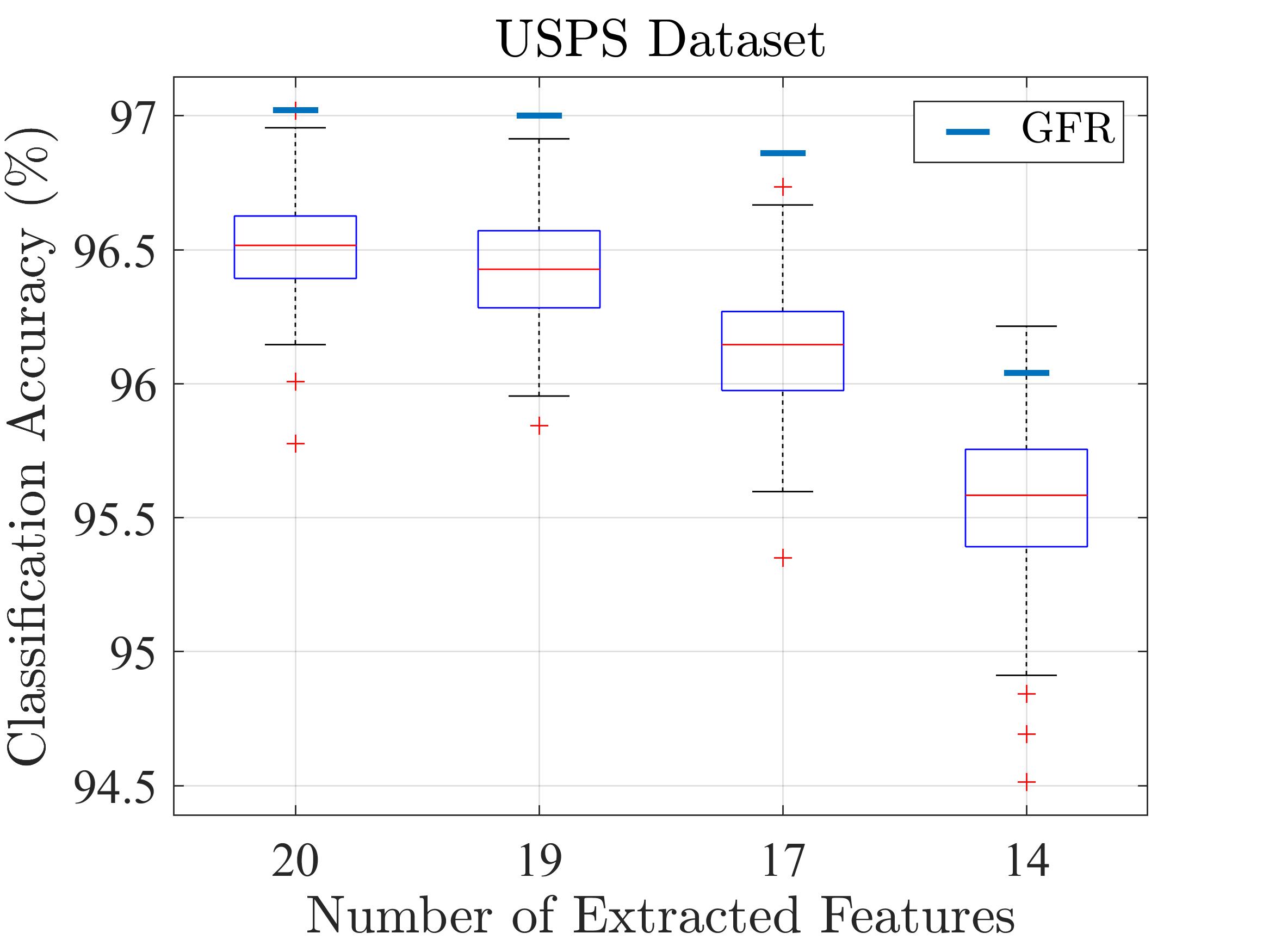}
\end{minipage}
\hfill
\begin{minipage}[b]{0.24\textwidth}
    \centering
    \includegraphics[width=\textwidth]{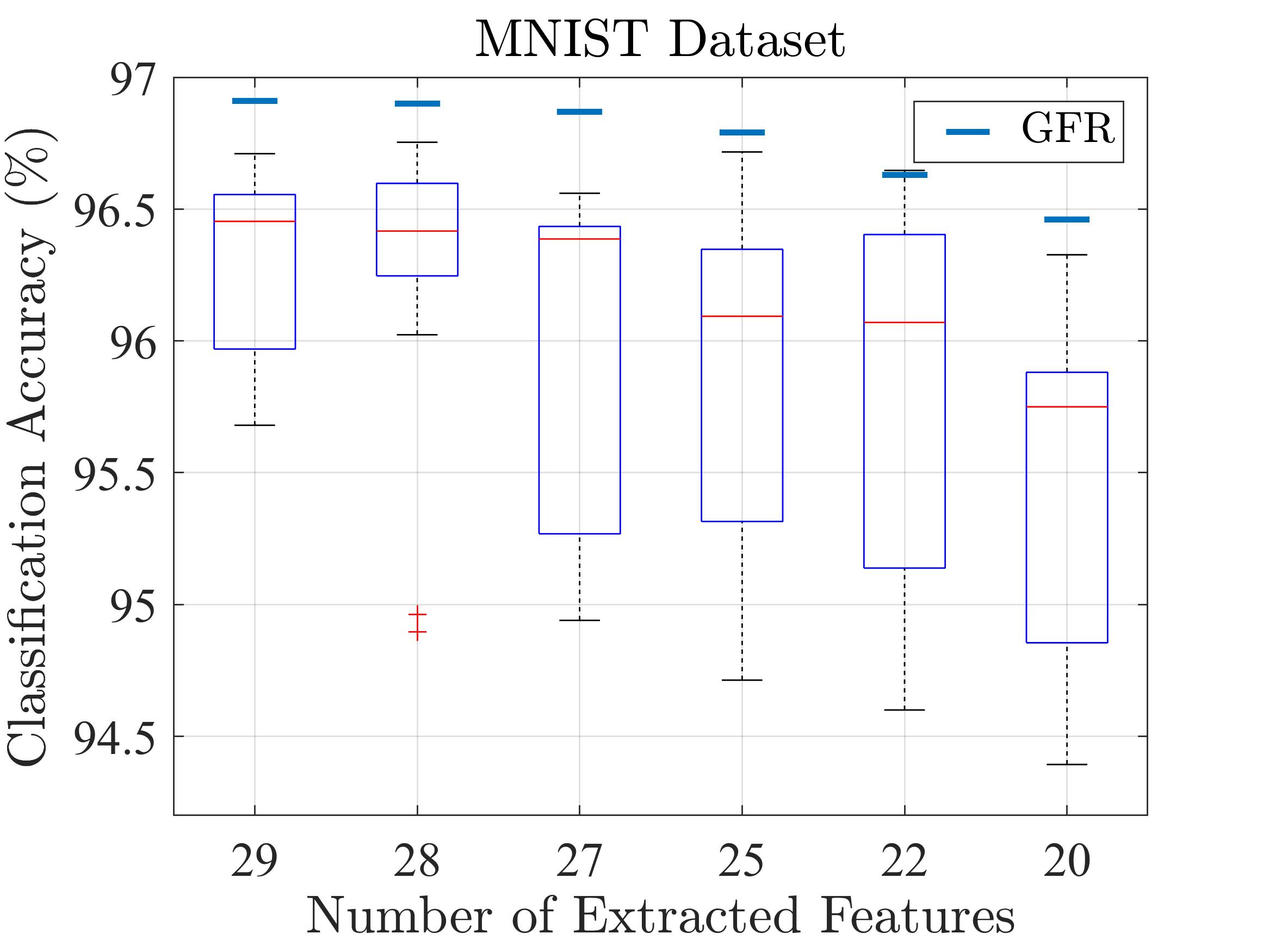}
\end{minipage}
\hfill
\begin{minipage}[b]{0.24\textwidth}
    \centering
    \includegraphics[width=\textwidth]{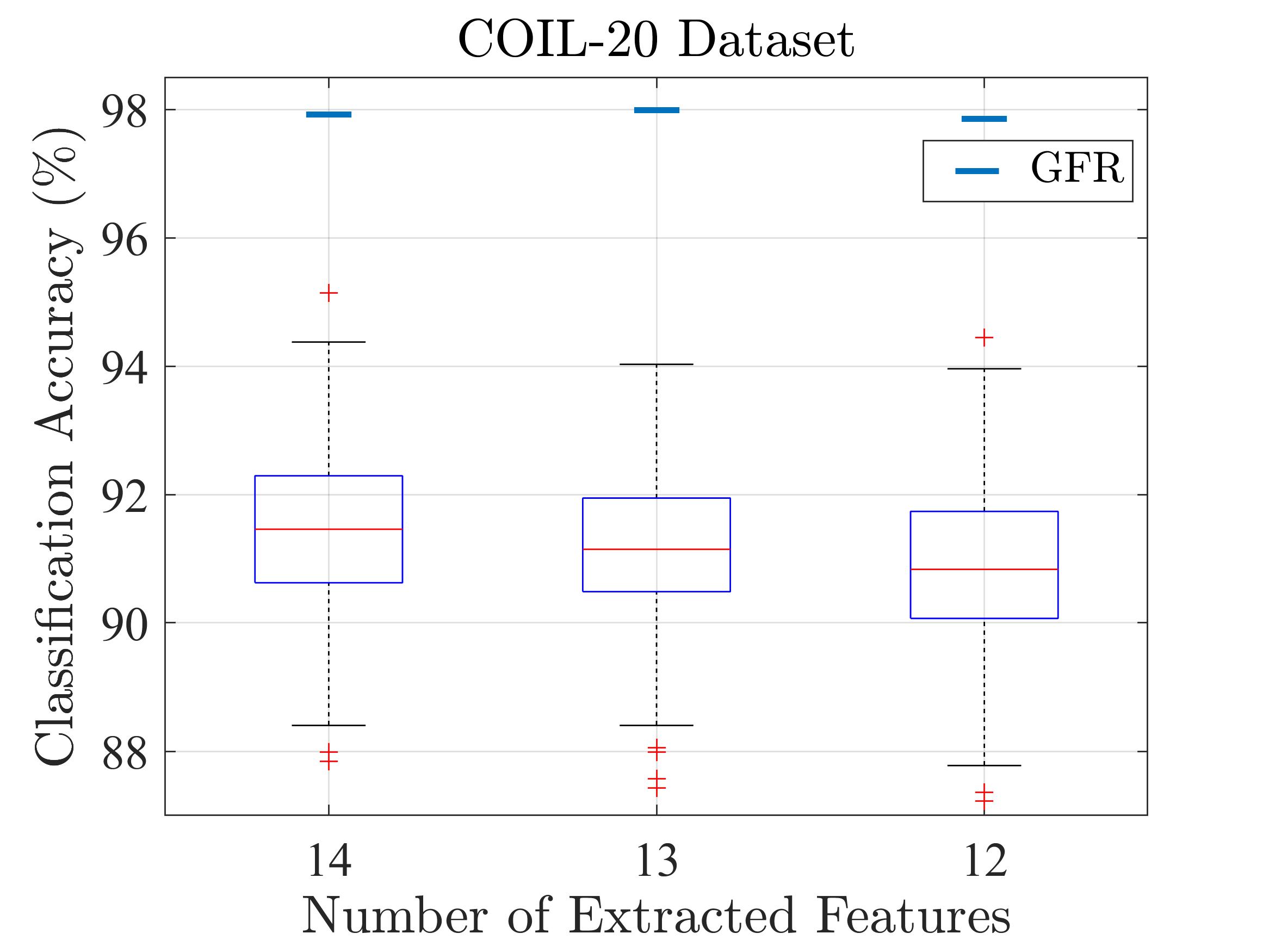}
\end{minipage}
\hfill
\begin{minipage}[b]{0.24\textwidth}
    \centering
    \includegraphics[width=\textwidth]{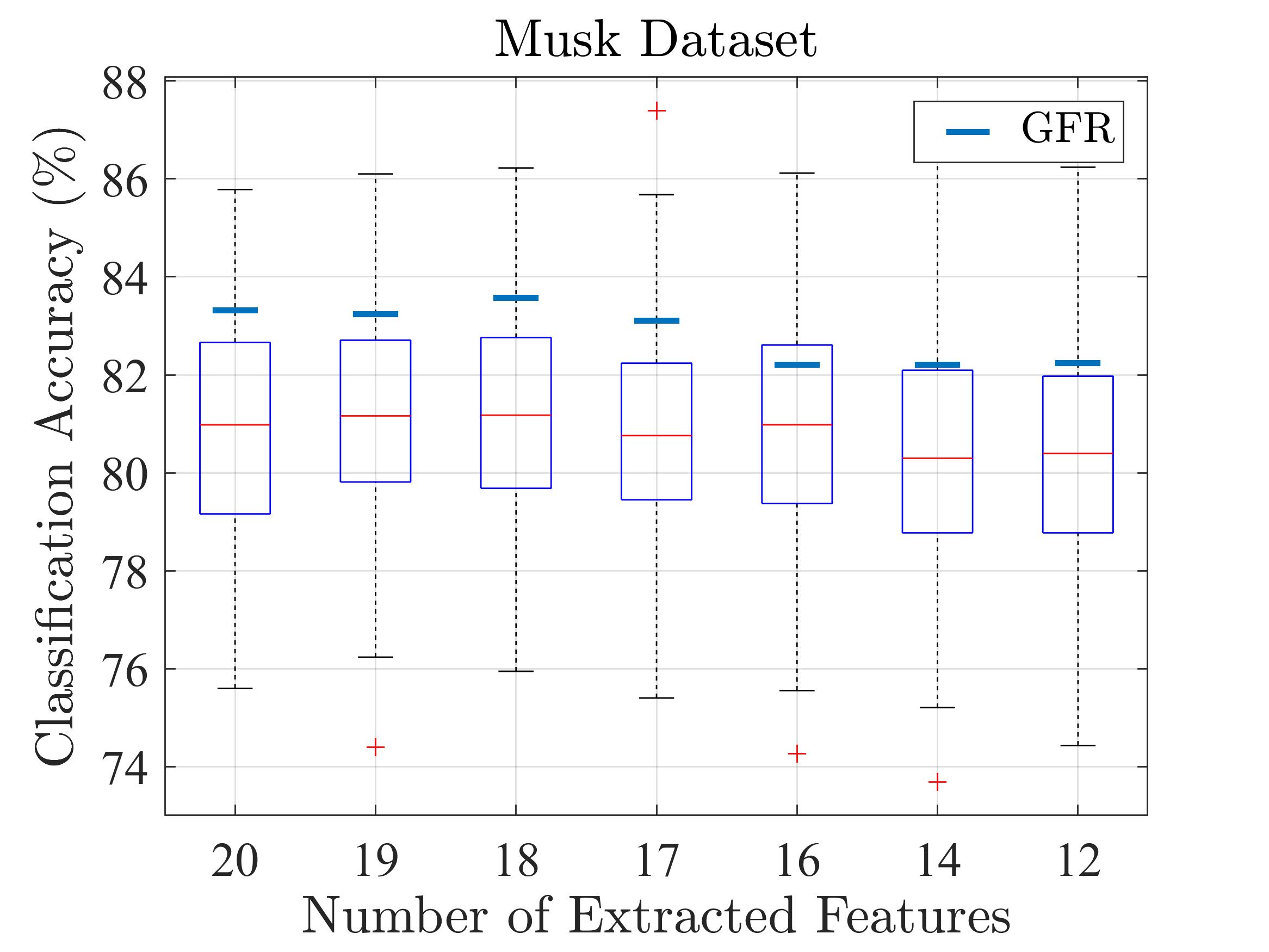}
\end{minipage}
\caption{Classification accuracy (\%) of GFR with multilinear polynomials of degree up to~$4$ versus autoencoders on benchmark datasets.}
\label{fig:GFR-Autoencoder-classification-accuracy-benchmark-datasets}
\end{figure}

\newpage

\subsubsection{GFR vs PCA in Synthetic Datasets}
In this subsection, we evaluate GFR's performance in terms of classification accuracy in synthetic datasets. 
The main goal of these experiments is to show that GFR outperforms PCA in datasets with underlying normal distribution.
We generate the synthetic datasets with~$20$ features, including~$10$ mutually independent random variables $X_1,\ldots,X_{10}$, where each $X_i$ is normally distributed with zero mean and a random variance chosen from~$\operatorname{Unif}(0,2)$. 
For~$i\in\{11,\ldots,20\}$ we chose~$X_i=X_jX_k$ with uniformly random distinct~$j,k\in\{1,\ldots,10\}$.
We applied GFR with~$\cF(\boldz)$ being multilinear polynomials of degree up to two, 
and after extracting features we applied PCA to extract exactly the same number of features obtained by GFR. 
To evaluate their performance, we use a common method in unsupervised feature extraction \cite{heidari2021information,scholkopf2005kernel,he2003locality}, where we assign labels to the data points but never use them in the feature extraction procedure. 
After extracting new features, we use the labels for computing the error of binary classification tasks. 
To generate labels we used either a polynomial threshold function (PTF)
\begin{align*}
    \textstyle f(X_1,\ldots,X_{10}) = \operatorname{sign}\left[ \prod_{1\leq j \leq 3} \left(b_{0,j} + \sum_{i=1}^{10} b_{i,j} X_i \right) \right],
\end{align*}
with $b_{i,j}  \sim \operatorname{Unif}(0, 1)$ and mutually independent, or a linear threshold function (LTF)
\begin{align*}
    \textstyle f(X_1,\ldots,X_{10}) = \operatorname{sign}\left[ b_{0} + \sum_{i=1}^{10} b_{i} X_i \right],
\end{align*}
with $b_{i}  \sim \operatorname{Unif}(0, 1)$ and mutually independent.

We used a support vector machine classifier with a radial basis function as kernel. $5$-fold cross-validation on the entire datasets is used to validate the performance of the algorithms. We repeated the experiments for different dataset sizes, and the simulation results are shown in Table~\ref{table:classification-synthetic-GFR-PCA-20-features}. As we see in these tables, the classification accuracy of GFR for both LTF and PTF label functions is larger than PCA.

\begin{table}[h]
\centering
\subfloat[Synthetic Dataset with PTF Labels]{%
\begin{tabular}{c | c c c c}
    \toprule
    \multirow{2}{*}{Method} & \multicolumn{4}{c}{Dataset Size} \\
    \cmidrule{2-5}
     & 1000 & 2000 & 5000 & 10000 \\
    \midrule
    \rowcolor{gray!20}
    GFR & \textbf{57.03} & \textbf{58.62} & \textbf{61.21} & \textbf{63.93} \\
    PCA & 56.43 & 57.41 & 58.43 & 59.4 \\
    \bottomrule
\end{tabular}
}
\hspace{45pt}
\subfloat[Synthetic Dataset with LTF Labels]{%
\begin{tabular}{c | c c c c}
    \toprule
    \multirow{2}{*}{Method} & \multicolumn{4}{c}{Dataset Size} \\
    \cmidrule{2-5}
     & 1000 & 2000 & 5000 & 10000 \\
    \midrule
    \rowcolor{gray!20}
    GFR & \textbf{87.72} & \textbf{89.96} & \textbf{92.92} & \textbf{95.15} \\
    PCA & 84.97 & 85.66 & 86.22 & 86.36 \\
    \bottomrule
\end{tabular}
}
\caption{Classification accuracy (\%) of GFR with multilinear polynomials of degree up to~$2$ versus PCA in synthetic datasets with 20 features.}
\label{table:classification-synthetic-GFR-PCA-20-features}
\end{table}

To show the effect of the number of extracted features on the performance of GFR and PCA, we performed another experiment with a dataset dimension of~$d = 60$, among which~$30$ are independent, and LTF label functions (Table~\ref{table:classification-synthetic-GFR-PCA-60-features}). 
The result show that GFR retains its advantage over PCA even as the number of extracted features grows.

\begin{table*}[!h]
\centering
\begin{tabular}{c | c c c c c c c c c c}
    \toprule
    \multirow{2}{*}{Method} & \multicolumn{10}{c}{Number of Extracted Features} \\
    \cmidrule{2-11}
     & 1 & 2 & 3 & 4 & 5 & 10 & 15 & 20 & 25 & 30 \\
    \midrule
    \rowcolor{gray!20}
    GFR & \textbf{59.26} & \textbf{61.25} & \textbf{63.20} & \textbf{64.60} & \textbf{66.14} & \textbf{71.64} & \textbf{76.62} & \textbf{79.92} & \textbf{82.42} & \textbf{84.28} \\
    PCA & 59.26 & 61.25 & 62.92 & 64.13 & 65.43 & 70.02 & 74.80 & 77.42 & 79.73 & 80.95 \\
    \bottomrule
\end{tabular}
\caption{Classification accuracy (\%) of GFR with multilinear polynomials of degree up to~$2$ against PCA in synthetic datasets with LTF labels and~$60$ features, among which~$30$ are mutually independent and normally distributed, and the remaining ones are randomly chosen monomials of degree~$2$ of the~$30$ independent ones.}
\label{table:classification-synthetic-GFR-PCA-60-features}
\end{table*}

\subsection{GCA's Performance}\label{section:experiments_GCA}
We evaluate GCA on the basis of its success in eliminating all redundant features among the principal components in synthetic datasets. 
Similar to Section~\ref{section:GCA}, suppose~$X=\sum_{i=1}^d\boldrho_iZ_i$.
We let~$Z_{\ell_1},\ldots,Z_{\ell_n}$ for $\ell_1,\ldots,\ell_n\in[d]$ be mutually independent non-redundant principal components chosen from~$\cN(0,\operatorname{Unif}(0,n))$, and let~$Z_{\ell_{n+1}},\ldots,Z_{\ell_d}$ be functions of~$Z_{\ell_1},\ldots,Z_{\ell_n}$, induced by the redundancies as explained shortly. 
The experiments detailed below were repeated~$1000$ times, in each experiment GCA was applied, and its output was compared against the non-redundant principal components.

\subsubsection{Experiments for Corollary~\ref{corollary:zeroRedZeroApprox}}\label{section:experiments-GCA-Corollary1}
We define~$t$ $0$-redundancies using two types of multilinear polynomials up to degree~$2$ or~$3$, i.e.,
\begin{align*}
    R^{(i)}&=Z_{\ell_{i+n}}-k_iZ_{\ell_j}Z_{\ell_k} \text{ or}\\
    R^{(i)}&=Z_{\ell_{i+n}}-k_iZ_{\ell_j}Z_{\ell_k}Z_{\ell_r},
\end{align*}
where~$\ell_j,\ell_k,\ell_r$ are chosen at random, and~$k_i$ is chosen\footnote{For example, it is readily seen that $k_i<\sqrt{\min\{\var(Z_{\ell_j}),\var(Z_{\ell_k})\}/(\var(Z_{\ell_j})\var(Z_{\ell_k}))}$ suffices since the~$Z_{\ell_i}$'s are mutually independent.} so that~$\var(Z_{\ell_{i+n}})$ is smallest among the variables participating in~$R^{(i)}$.

We apply GCA with~$(d,n)=(30,15)$ and~$\cF(\boldz)$ being multilinear polynomials of degree up to~$2$ or~$3$, according to the degree of the~$R^{(i)}$'s mentioned above. 
In our experiments, we vary dataset sizes from $300$ to $2000$ and in each experiment GCA is applied, and its output is compared against the non-redundant principal components. 
We repeat the experiments for different dataset sizes~$1000$ times and the success rate of these experiments, expressed as a percentage, is reported in Table~\ref{table:GCA-corollary-1} for the two different types of redundancies. 
The results show that if the dataset size is large enough (roughly 1000 in these experiments), GCA can remove all redundant principal components from the data. 
However, as the dataset size decreases, the approximation of the covariance matrices degrades. 
This degradation affects GCA's performance, particularly in datasets with redundancies of higher-degree multilinear polynomials. 

To show the effect of dataset dimension on GCA's performance, we also generate synthetic datasets with $d=50$ features, among them $n=25$ are non-redundant principal components. The success rate of GCA in removing redundant principal components are shown in Table~\ref{table:GCA-corollary-1}. It is obvious that as the dimension of datasets increases, more samples are needed to achieve similar performance compared to lower-dimensional datasets. 
We mention that even though we consider $0$-redundancies in these experiments, due to numerical errors, the parameter~$\epsilon$ in line~\ref{line:GCA_E_j} of Algorithm~\ref{algorithm:GCA} is chosen as a small constant $\epsilon=10^{-4}$.

\begin{table*}[!h]
\centering
\begin{tabular}{c | c c c c c c c c c c c c c c}
    \toprule
    \multirow{2}{*}{$\big(d,n,\operatorname{deg}(R^{(i)})\big)$} & \multicolumn{14}{c}{Dataset Size} \\
    \cmidrule{2-15}
     & 300 & 350 & 400 & 450 & 500 & 550 & 600 & 700 & 800 & 900 & 1000 & 1200 & 1500 & 2000 \\
    \midrule
    \rowcolor{gray!20}
    $(30,15,2)$ & 65.5 & 76.0 & 83.5 & 87.5 & 93.6 & 96.2 & 97.6 & 99.0 & 99.4 & 100.0 & 100.0 & 100.0 & 100.0 & 100.0 \\
    $(30,15,3)$ & 0 & 0 & 0.8 & 12.1 & 68.6 & 84.5 & 96.5 & 97.9 & 99.2 & 99.8 & 100.0 & 100.0 & 100.0 & 100.0 \\
    \rowcolor{gray!20}
    $(50,25,2)$ & 5.1 & 32.5 & 47.9 & 62.1 & 72.7 & 75.6 & 82.3 & 92.0 & 95.7 & 98.7 & 99.0 & 99.9 & 100.0 & 100.0 \\
    \bottomrule
\end{tabular}
\caption{The success rate (\%) of GCA in extracting correct non-redundant principal components from $1000$ experiments conducted on synthetic datasets with~$d$ features, where~$n$ features are non-redundant (mutually independent)~$Z_{\ell_1},\ldots,Z_{\ell_n}$, chosen from~$\cN(0,\operatorname{Unif}(0,n))$, and the remaining~$d-n$ features are $0$-redundancies of multilinear polynomials up to degree $\operatorname{deg}(R^{(i)})$ (Corollary~\ref{corollary:zeroRedZeroApprox}).}
\label{table:GCA-corollary-1}
\end{table*}

\subsubsection{Experiments for Corollary~\ref{corollary:some0someeps}}\label{section:experiments-GCA-Corollary2} 
We repeat similar experiments to Section~\ref{section:experiments-GCA-Corollary1} with the addition of noise, i.e.,
\begin{align*}
    R^{(i)}&=Z_{\ell_{i+n}}-k_iZ_{\ell_j}Z_{\ell_k}+\delta_i \text{ or}\\
    R^{(i)}&=Z_{\ell_{i+n}}-k_iZ_{\ell_j}Z_{\ell_k}Z_{\ell_r}+\delta_i,
\end{align*}
with mutually independent~$\delta_i\sim\cN(0,0.1)$.
Again Similar to section~\ref{section:experiments-GCA-Corollary1}, we apply GCA with~$\cF(\boldz)$ being multilinear polynomials of degree up to~$2$ or~$3$.
The parameter~$\epsilon$ in line~\ref{line:GCA_E_j} of Algorithm~\ref{algorithm:GCA} is chosen as the empirical variance of the added noise $\delta_i$ in each experiment. 
The results, which are summarized in Table~\ref{table:GCA-corollary-2}, demonstrate resilience of GCA to stochastic noise.
\begin{table*}[!h]
\centering
\begin{tabular}{c | c c c c c c c c c c c c c c}
    \toprule
    \multirow{2}{*}{$\big(d,n,\operatorname{deg}(R^{(i)})\big)$} & \multicolumn{14}{c}{Dataset Size} \\
    \cmidrule{2-15}
     & 300 & 350 & 400 & 450 & 500 & 550 & 600 & 700 & 800 & 900 & 1000 & 1200 & 1500 & 2000 \\
    \midrule
    \rowcolor{gray!20}
    $(30,15,2)$ & 63 & 68.9 & 75 & 81.5 & 83.2 & 85.9 & 93.9 & 96.8 & 98.2 & 98.3 & 98.7 & 98.8 & 99.9 & 100 \\
    $(30,15,3)$ & 0 & 0 & 1.2 & 25.3 & 74.3 & 83.7 & 85.1 & 88.2 & 89.1 & 91.3 & 94 & 96.9 & 99 & 100 \\
    \rowcolor{gray!20}
    $(50,25,2)$ & 0.5 & 25.5 & 44.2 & 48.6 & 56.3 & 60.1 & 65.8 & 75.1 & 81.2 & 83.7 & 90.5 & 95.3 & 100 & 100 \\
    \bottomrule
\end{tabular}
\caption{The success rate (\%) of GCA in extracting correct non-redundant principal components from $1000$ experiments conducted on synthetic datasets with~$d$ features, where~$n$ features are non-redundant (mutually independent)~$Z_{\ell_1},\ldots,Z_{\ell_n}$, chosen from~$\cN(0,\operatorname{Unif}(0,n))$, and the remaining~$d-n$ features are $\epsilon$-redundancies (stochastic noise) of multilinear polynomials up to degree $\operatorname{deg}(R^{(i)})$ (Corollary~\ref{corollary:some0someeps}).}
\label{table:GCA-corollary-2}
\end{table*}

\subsubsection{Experiments for Corollary~\ref{corollary:deg5polys}}\label{section:experiments-GCA-Corollary3} 
We repeat experiments similar to Section~\ref{section:experiments-GCA-Corollary1}, with an added mismatch between the degree of the~$0$-redundancies and the degree of the polynomials in~$\cF(\boldz)$. That is, we take 
\begin{align*}
    R^{(i)}&=Z_{\ell_{i+n}}-k_iZ_{\ell_j}Z_{\ell_k}Z_{\ell_r}Z_{\ell_a}, \text{ or}\\
    R^{(i)}&=Z_{\ell_{i+n}}-k_iZ_{\ell_j}Z_{\ell_k}Z_{\ell_r}Z_{\ell_a}Z_{\ell_b},
\end{align*}
but apply GCA with~$\cF(\boldz)$ containing multilinear polynomials only up to degree three.

The~$\epsilon$ parameter in GCA is chosen as the maximum value of 
\begin{align*}
    &\bE[(Z_{\ell_{i+n}}-\proj_{\cF(Z_{\ell_j},Z_{\ell_k},Z_{\ell_r},Z_{\ell_a})}Z_{\ell_{i+n}})^2], \text{or}\\
    &\bE[(Z_{\ell_{i+n}}-\proj_{\cF(Z_{\ell_j},Z_{\ell_k},Z_{\ell_r},Z_{\ell_a},Z_{\ell_b})}Z_{\ell_{i+n}})^2],
\end{align*}
respectively.
The results, which are summarized in Table~\ref{table:GCA-corollary-3}, demonstrate resilience of GCA to so-called \textit{deterministic} noise, i.e., the inability of~$\cF(\boldz)$ to represent the redundancies.
\begin{table*}[!h]
\centering
\begin{tabular}{c | c c c c c c c c c c c c c c}
    \toprule
    \multirow{2}{*}{$\big(d,n,\operatorname{deg}(R^{(i)})\big)$} & \multicolumn{14}{c}{Dataset Size} \\
    \cmidrule{2-15}
     & 300 & 350 & 400 & 450 & 500 & 550 & 600 & 700 & 800 & 900 & 1000 & 1200 & 1500 & 2000 \\
    \midrule
    \rowcolor{gray!20}
    $(30,15,4)$ & 0.7 & 28.7 & 47.7 & 50.2 & 59.6 & 63.7 & 70.1 & 73.2 & 82.2 & 87.9 & 91.3 & 95.2 & 97.8 & 100 \\
    $(30,15,5)$ & 0 & 0 & 0 & 11.6 & 37.4 & 45.3 & 60.2 & 62.1 & 74.5 & 79.3 & 89.1 & 90.1 & 93.3 & 99.3 \\
    \rowcolor{gray!20}
    $(50,25,4)$ & 0 & 0 & 0.1 & 13.4 & 42.1 & 46.7 & 63.1 & 67.2 & 78.9 & 81.1 & 90 & 90.1 & 95.2 & 99.5 \\
    \bottomrule
\end{tabular}
\caption{The success rate (\%) of GCA in extracting correct non-redundant principal components from $1000$ experiments conducted on synthetic datasets with~$d$ features, where~$n$ features are non-redundant (mutually independent)~$Z_{\ell_1},\ldots,Z_{\ell_n}$, chosen from~$\cN(0,\operatorname{Unif}(0,n))$, and the remaining~$d-n$ features are $\epsilon$-redundancies (deterministic noise) of multilinear polynomials up to degree $\operatorname{deg}(R^{(i)})$ (Corollary~\ref{corollary:deg5polys}).}
\label{table:GCA-corollary-3}
\end{table*}

\subsection{The Number of Selected Features using GFS}\label{section:experiments_GFS_selected_features}
To validate the performance of GFS, we begin by demonstrating that selecting higher degree multilinear polynomials as~$\cF(\boldz)$ leads to a significant reduction in maximum variance among the entries of~$d_j(X)$ in each iteration (i.e.,~$\max\{\var[d_j(X)_i]\}_{i=1}^d$). 
Consequently, GFS effectively detects and removes more redundant features from the datasets as the complexity of~$\cF(\boldz)$ increases.
To support this claim, we apply GFS with~$\cF(\boldz)$ being multilinear polynomials of degrees at most~$1$, at most~$2$, at most~$3$ and at most~$4$, 
to the benchmark datasets in Table~\ref{table:properties-of-benchmark-datasets}. 
The number of selected features for different values of the threshold~$\epsilon^2$ and degree of the polynomials are shown in Tables~\ref{table:number-of-selected-features-benchmark-datasets}. 
As demonstrated in the experimental results, by increasing the degree of multilinear polynomials, GFS selects fewer features given the same threshold $\epsilon^2$. 
\begin{table}[h]
\centering
\subfloat[USPS Dataset]{%
\begin{tabular}{c | c c c c c c}
    \toprule
    \multirow{2}{*}{\makecell{Degree of \\ Polynomial}} & \multicolumn{6}{c}{$\epsilon^2$} \\
    \cmidrule{2-7}
     & $0.03$ & $0.04$ & $0.05$ & $0.06$ & $0.07$ & $0.08$ \\
    \midrule
    \rowcolor{gray!20}
    1 & 85 & 60 & 48 & 41 & 36 & 29 \\
    2 & 39 & 31 & 25 & 21 & 18 & 16 \\
    \rowcolor{gray!20}
    3 & 25 & 20 & 18 & 16 & 15 & 12\\
    4 & 24 & 17 & 14 & 13 & 12 & 11 \\
    \bottomrule
\end{tabular}
}
\\
\subfloat[COIL-20 Dataset]{%
\begin{tabular}{c | c c c c c}
    \toprule
    \multirow{2}{*}{\makecell{Degree of \\ Polynomial}} & \multicolumn{5}{c}{$\epsilon^2$} \\
    \cmidrule{2-6}
     & $0.01$ & $0.0125$ & $0.015$ & $0.0175$ & $0.02$ \\
    \midrule
    \rowcolor{gray!20}
    1 & 98 & 76 & 60 & 50 & 40  \\
    2 & 36 & 21 & 17 & 15 & 14  \\
    \rowcolor{gray!20}
    3 & 36 & 20 & 13 & 11 & 9  \\
    4 & 31 & 19 & 13 & 10 & 9  \\
    \bottomrule
\end{tabular}
}
\\
\subfloat[Credit Approval Dataset]{%
\begin{tabular}{c | c c}
    \toprule
    \multirow{2}{*}{\makecell{Degree of \\ Polynomial}} & \multicolumn{2}{c}{$\epsilon^2$} \\
    \cmidrule{2-3}
     & $0.1$ & $0.4$ \\
    \midrule 
    \rowcolor{gray!20}
    1 & 14 & 13  \\
    2 & 13 & 12  \\
    \rowcolor{gray!20}
    3 & 13 & 12  \\
    4 & 13 & 12  \\
    \bottomrule
\end{tabular}
}
\caption{Number of selected features by GFS with multilinear polynomials on benchmark datasets for different values of the threshold,~$\epsilon^2$.}
\label{table:number-of-selected-features-benchmark-datasets}
\end{table}

\subsection{Classification Accuracy of GFS}\label{section:experiments_GFS_classification_accuracy} 
We turn to validate the performance of GFS for classification tasks on the benchmark datasets in Table~\ref{table:properties-of-benchmark-datasets}. 
We applied GFS with multilinear polynomial of degree up to~$4$ against several well-known feature selection algorithms: Multi-Cluster Feature Selection (MCFS)~\cite{cai2010unsupervised}, Nonnegative Discriminative Feature Selection (NDFS)~\cite{li2012unsupervised}, Unsupervised Discriminative Feature Selection (UDFS)~\cite{yang2011l2}, Laplacian Score (LS)~\cite{he2005laplacian}, Trace Ratio (TR)~\cite{nie2008trace}, and Fisher Score (FS)~\cite{duda2006pattern}. 

We used a support vector machine classifier with radial basis function as kernel, and 5-fold cross-validation on the entire dataset to validate the performance of the algorithms. 
To implement MCFS, NDFS, UDFS, LS, TR, and FS, we used \texttt{skfeature-chappers} package.
The results, given in Table~\ref{table:GFS-classification-accuracy-benchmark-datasets}, demonstrate superior performance in terms of  classification accuracy in the benchmark datasets in comparison to other algorithms.
\begin{table}[h]
\centering
\subfloat[USPS Dataset]{%
\begin{tabular}{c | c c c c c c}
    \toprule
    \multirow{2}{*}{Method} & \multicolumn{6}{c}{Number of Selected Features} \\
    \cmidrule{2-7}
     & 24 & 17 & 14 & 13 & 12 & 11 \\
    \midrule
    \rowcolor{gray!20}
    GFS & \textbf{92.62} & \textbf{89.27} & \textbf{85.10} & \textbf{83.73} & \textbf{82.25} & \textbf{79.60} \\
    MCFS & 90.99 & 87.37 & 84.39 & 82.81 & 79.76 & 75.13 \\
    \rowcolor{gray!20}
    NDFS & 89.48 & 82.84 & 78.38 & 73.98 & 72.92 & 69.98 \\
    UDFS & 86.37 & 78.81 & 76.56 & 75.71 & 73.06 & 69.03 \\
    \rowcolor{gray!20}
    LS & 85.21 & 80.94 & 76.59 & 75.59 & 74.12 & 73.36 \\
    TR & 83.87 & 78.48 & 74.52 & 72.10 & 71.81 & 68.71 \\
    \rowcolor{gray!20}
    FS & 83.87 & 78.93 & 74.39 & 74.02 & 71.81 & 68.71 \\
    \bottomrule
\end{tabular}
}
\hspace{30pt}
\subfloat[COIL-20 Dataset]{%
\begin{tabular}{c | c c c c c}
    \toprule
    \multirow{2}{*}{Method}  & \multicolumn{5}{c}{Number of Selected Features} \\
    \cmidrule{2-6}
     & 31 & 19 & 13 & 10 & 9 \\
    \midrule
    \rowcolor{gray!20}
    GFS & \textbf{90.76} & \textbf{87.08} & \textbf{84.72} & \textbf{81.18} & \textbf{78.19} \\
    MCFS & 79.03 & 77.57 & 67.57 & 61.18 & 59.93 \\
    \rowcolor{gray!20}
    NDFS & 87.78 & 83.89 & 81.94 & 77.64 & 75.76 \\
    UDFS & 69.44 & 63.47 & 58.75 & 53.82 & 52.36 \\
    \rowcolor{gray!20}
    LS & 68.47 & 65.28 & 57.08 & 49.86 & 48.40 \\
    TR & 66.52 & 57.85 & 43.75 & 41.25 & 39.10 \\
    \rowcolor{gray!20}
    FS & 59.23 & 56.32 & 43.02 & 41.20 & 38.90 \\
    \bottomrule
\end{tabular}
}
\\
\subfloat[Credit Approval Dataset]{%
\begin{tabular}{c | c c}
    \toprule
    \multirow{2}{*}{Method}  & \multicolumn{2}{c}{Number of Selected Features} \\
    \cmidrule{2-3}
     & 13 & 12 \\
    \midrule 
    \rowcolor{gray!20}
    GFS & 84.20 ($2^{nd}$ best) & 83.91 ($2^{nd}$ best) \\
    MCFS & \textbf{84.34} & 83.62 \\
    \rowcolor{gray!20}
    NDFS & 83.62 & 83.47 \\
    UDFS & 75.79 & 75.22 \\
    \rowcolor{gray!20}
    LS & 84.05 & \textbf{84.20} \\
    TR & 75.79 & 75.51 \\
    \rowcolor{gray!20}
    FS & 75.80 & 75.51 \\
    \bottomrule
\end{tabular}
}
\caption{Classification Accuracy (\%) of GFS with~$\cF(\boldz)$ multilinear polynomials of degree up to 4 versus MCFS~\cite{cai2010unsupervised}, NDFS~\cite{li2012unsupervised}, UDFS~\cite{yang2011l2}, LS~\cite{he2005laplacian}, TR~\cite{nie2008trace}, FS~\cite{duda2006pattern} on benchmark datasets.}
\label{table:GFS-classification-accuracy-benchmark-datasets}
\end{table}

\subsection{Comparison between GFS and UFFS~\cite{heidari2022sufficiently}}\label{section:experiments_GFS_vs_UFFS}
As mentioned earlier (and shown in detail in Appendix~\ref{section:previousFourier}), GFS generalizes the Unsupervised Fourier Feature Selection (UFFS) algorithm of~\cite{heidari2022sufficiently} at significantly reduced complexity. Specifically, UFFS arises as a special case of GFS when~$\cF(\boldz)$ is chosen as multilinear polynomials up to degree~$d$. 
The reduction in complexity is due to our step-by-step orthogonalization process (Algorithm~\ref{algorithm:orthogonalize}), in contrast to UFFS which orthogonalizes all multilinear polynomials. 
We corroborate GFS's superiority over UFFS in terms of running time, the ability to capture redundant features, and classification accuracy over synthetic datasets. 
The synthetic datasets we use in these experiments include~$15$ independent normally distributed random variables $X_{\ell_1},\ldots,X_{\ell_{15}}\sim\cN(0,\operatorname{Unif}(0.5,1))$\footnote{We use the lower bound~$0.5$ in the variance of independent~$X_{\ell_i}$'s to make sure their variance is not very small and not considered as redundant in our analysis.} and~$15$ redundant features $X_{\ell_{16}},\ldots,X_{\ell_{30}}$, which are randomly and uniformly taken from $\{X_{\ell_i}X_{\ell_j}\}_{\mbox{\footnotesize distinct }i,j\in[15]}\cup\{X_{\ell_i}X_{\ell_j}X_{\ell_k}\}_{\mbox{\footnotesize distinct }i,j,k\in[15]}$; the integers~$\ell_1,\ldots,\ell_{30}$ are a random permutation of~$1,\ldots,30$.
We also consider 1000 sample points in each dataset. 
In all the following comparisons between GFS and UFFS, we choose multilinear polynomials of degree up to three as the function family~$\cF(\boldz)$ in GFS, and in UFFS we applied the heuristic proposed in~\cite{heidari2022sufficiently} to orthogonalize multilinear polynomials only up to degree three\footnote{Technically, UFFS requires one to orthogonalize \textit{all} multilinear polynomials up to degree~$30$, and as a heuristic,~\cite{heidari2022sufficiently} proposed to orthogonalize only up to a lower degree. Clearly, orthogonalizing all~$2^{30}$ multilinear polynomials to execute UFFS in its entirety would result in infeasible run times.}. 
In the experimental section of~\cite{heidari2022sufficiently}, the features are standardized\footnote{I.e., the mean vector is subtracted from all points of the data, and then every feature is divided by its variance. This results in data which has mean zero, and every feature has variance~$1$.} before applying UFFS. We apply both GFS and UFFS to the synthetic datasets and their standardized versions. 
\subsubsection{Running Time} 
In Table~\ref{table:processing-time}, the running time ratio of UFFS and GFS over synthetic datasets is compared for the non-standardized datasets and their standardized versions. We repeat the experiments for 1000 different synthetic datasets and report the average values in Table~\ref{table:processing-time}. The results show that on average GFS is about~$22$ times faster than UFFS for the chosen parameters. 
\begin{table}[!h]
\centering
\begin{tabular}{c !{\vrule} c | c}
    \toprule
    & \makecell{GFS over \\ non-standardized \\ datasets} & \makecell{GFS over \\ standardized \\ datasets} \\
    \midrule 
    \makecell{UFFS over \\ non-standardized \\ datasets} & 20.92 & 6.99 \\
    \midrule
    \makecell{UFFS over \\ standardized \\ datasets} & 22.11 & 7.39 \\
    \bottomrule
\end{tabular}
\caption{Running time ratio of GFS vs. UFFS, i.e. $t_{\text{UFFS}}/t_{\text{GFS}}$, for non-standardized and standardized synthetic datasets with $30$ features, where $15$ features are non-redundant (mutually independent) $X_{\ell_1},\ldots,X_{\ell_{15}}$ chosen from~$\cN(0,\operatorname{Unif}(0.5,1))$, and the remaining $15$ features are multilinear polynomials of these non-redundant features up to degree $3$.}
\label{table:processing-time}
\end{table}

\subsubsection{Capturing Redundant Features} 
In this section, we demonstrate that GFS can capture nonlinear redundancies better than UFFS. We perform 1000 experiments over synthetic datasets and in each of them we apply GFS and UFFS with the threshold~$\epsilon^2=0.01$ in line~\ref{line:GFS_epsilon} of Algorithm~\ref{algorithm:GFS}. 
Figure~\ref{fig:gfs_uffs_selected_features} shows the histograms of the number of non-redundant features selected by GFS and UFFS (out of~$1000$ experiments) on both non-standardized and standardized synthetic datasets. 
The results indicate that GFS selects fewer non-redundant features than UFFS on both non-standardized and standardized datasets, suggesting that GFS is more effective at removing redundant features compared to UFFS.
In~\cite{heidari2022sufficiently}, the authors apply UFFS three times and report the results after the third round. For a fair comparison, we apply both GFS and UFFS only once. 
Another insight from these simulations is the sensitivity of UFFS to the order of features. 
To validate this claim, we set the non-redundant features as the first~$15$ features in the datasets, i.e., $X_{\ell_1}=X_1,\ldots,X_{\ell_{15}}=X_{15}$, and repeated the experiments 1000 times. 
In all cases, UFFS successfully selected the non-redundant features and removed the redundant ones. However, in real datasets, the order is unknown. 
GFS, on the other hand, is not sensitive to the order of features because it uses a step-by-step orthogonalization process. Moreover, as seen in Figures~\ref{fig:gfs_uffs_selected_features_c}-\ref{fig:gfs_uffs_selected_features_d}, standardization does not affect UFFS's performance, as it starts from a random feature and orthogonalizes all multilinear polynomials up to degree 3. However, in Figures~\ref{fig:gfs_uffs_selected_features_a}-\ref{fig:gfs_uffs_selected_features_b} it is evident that standardization affects GFS's performance. GFS begins with the highest variant feature and uses step-by-step orthogonalization to select the most informative features. When the features are standardized, they are supposed to have unit variances. However, due to computational errors, one of them may end up with a higher variance, even if it is initially one of the features with the smallest variance. This discrepancy is a result of the standardization process. It impacts the rest of the orthogonalization process, causing some redundant features to be selected as non-redundant. Figures~\ref{fig:gfs_uffs_selected_features_b}-\ref{fig:gfs_uffs_selected_features_c} show that even in standardized datasets, GFS outperforms UFFS in terms of capturing nonlinear redundant features.

\begin{figure*}[!h] 
\centering
\subfloat[\label{fig:gfs_uffs_selected_features_a}]{%
\includegraphics[width=0.25\linewidth]{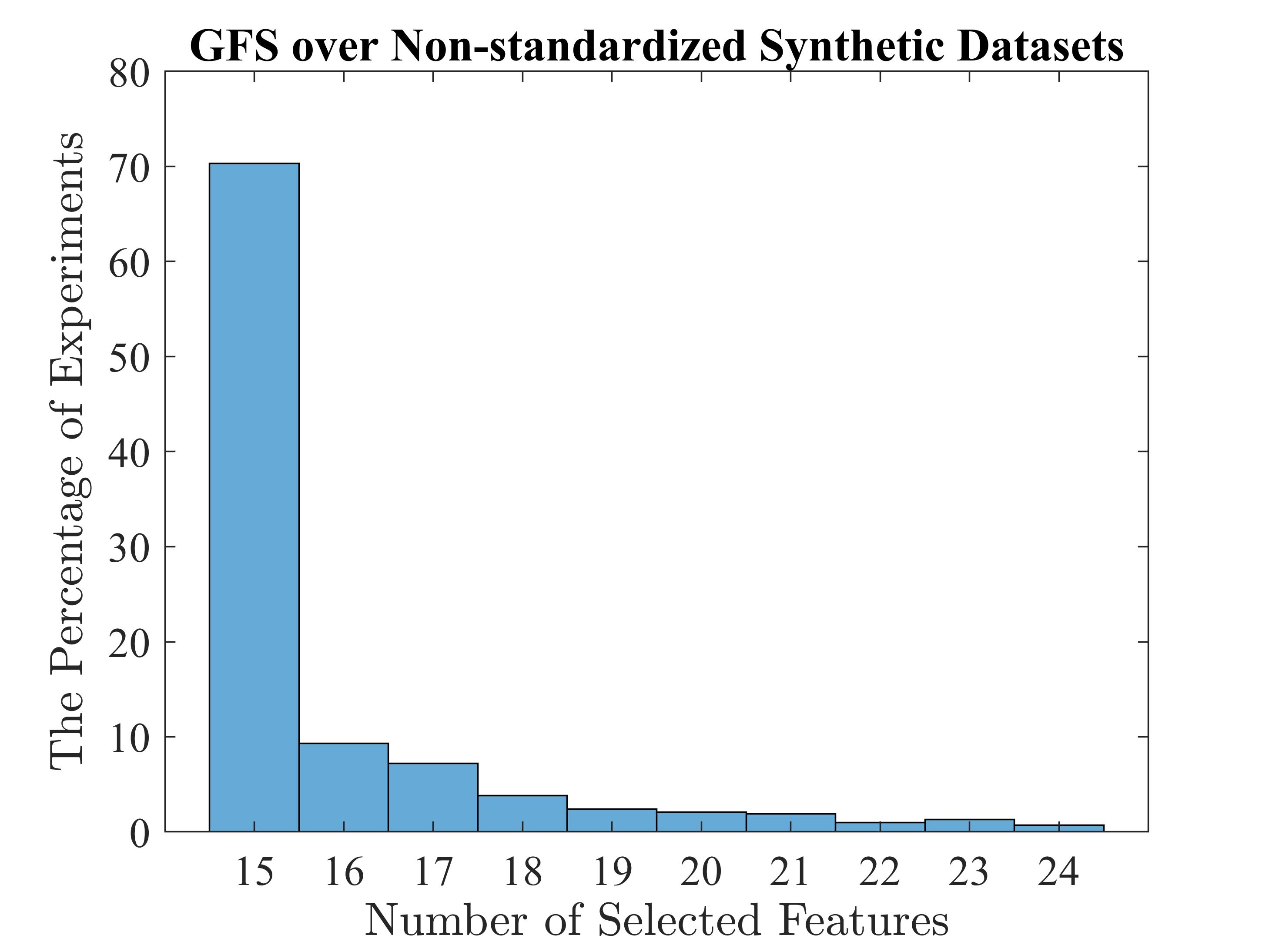}}
\hfill
\subfloat[\label{fig:gfs_uffs_selected_features_b}]{%
\includegraphics[width=0.25\linewidth]{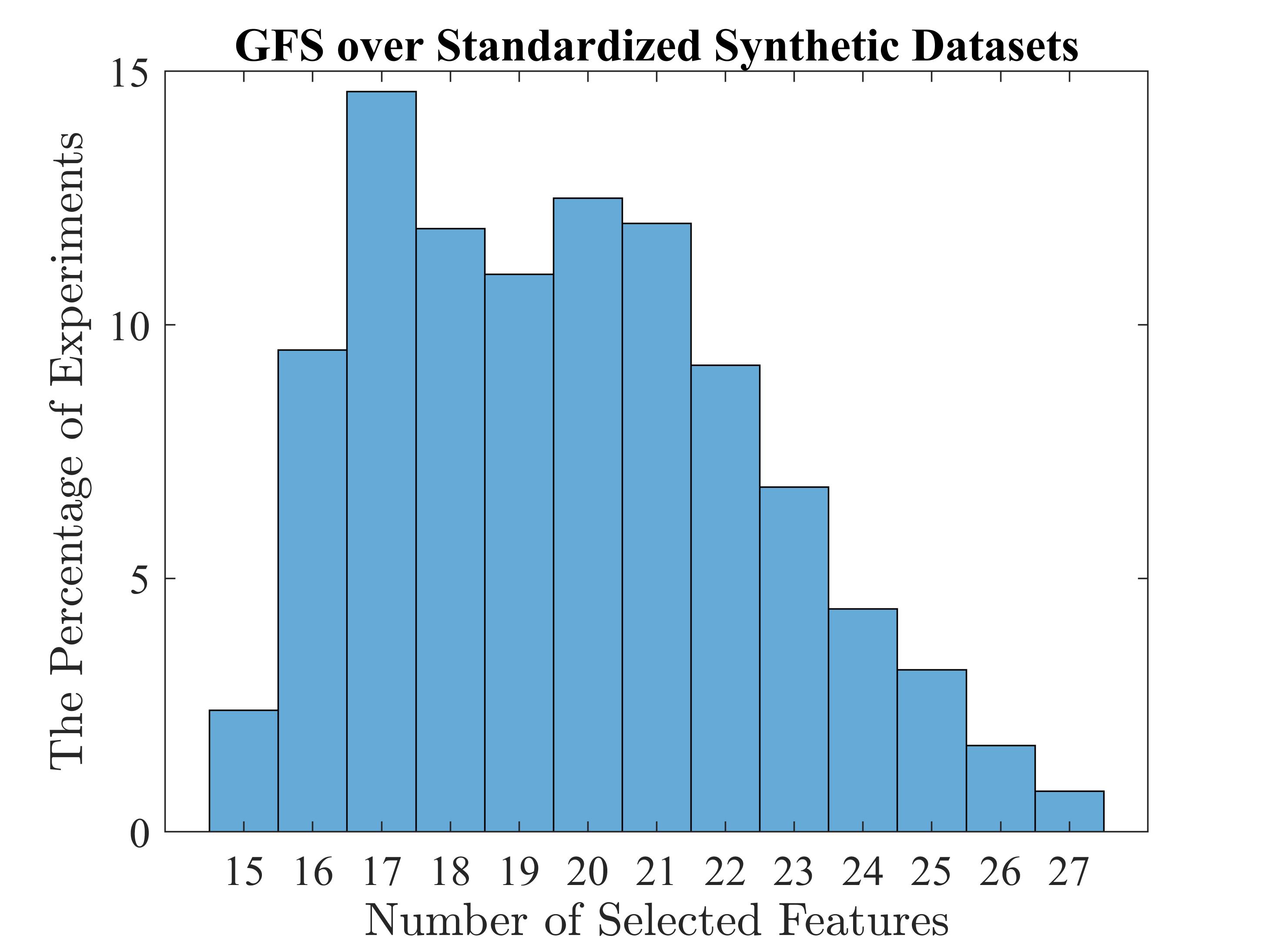}}
\hfill
\subfloat[\label{fig:gfs_uffs_selected_features_c}]{%
\includegraphics[width=0.25\linewidth]{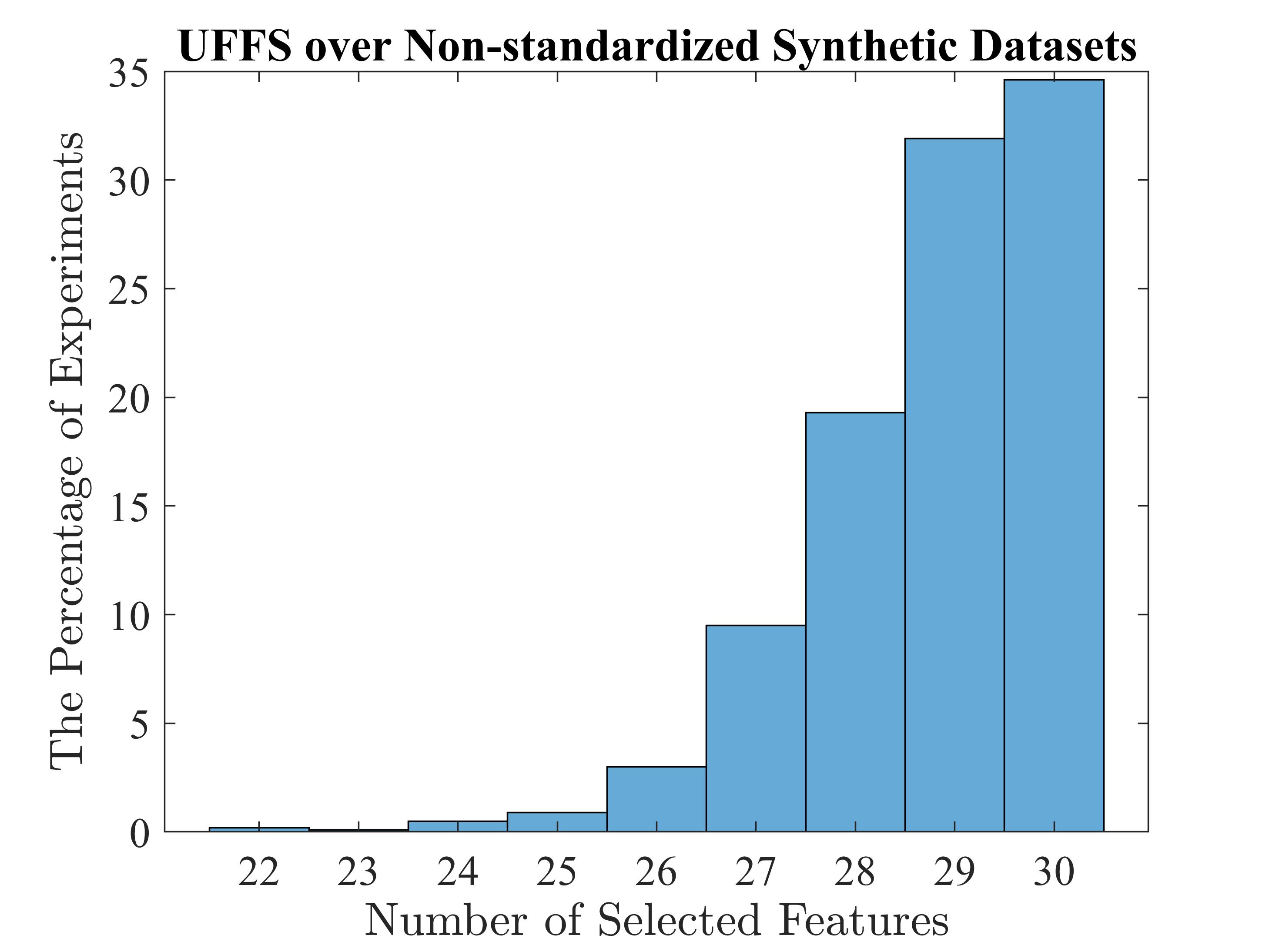}}
\hfill
\subfloat[\label{fig:gfs_uffs_selected_features_d}]{%
\includegraphics[width=0.25\linewidth]{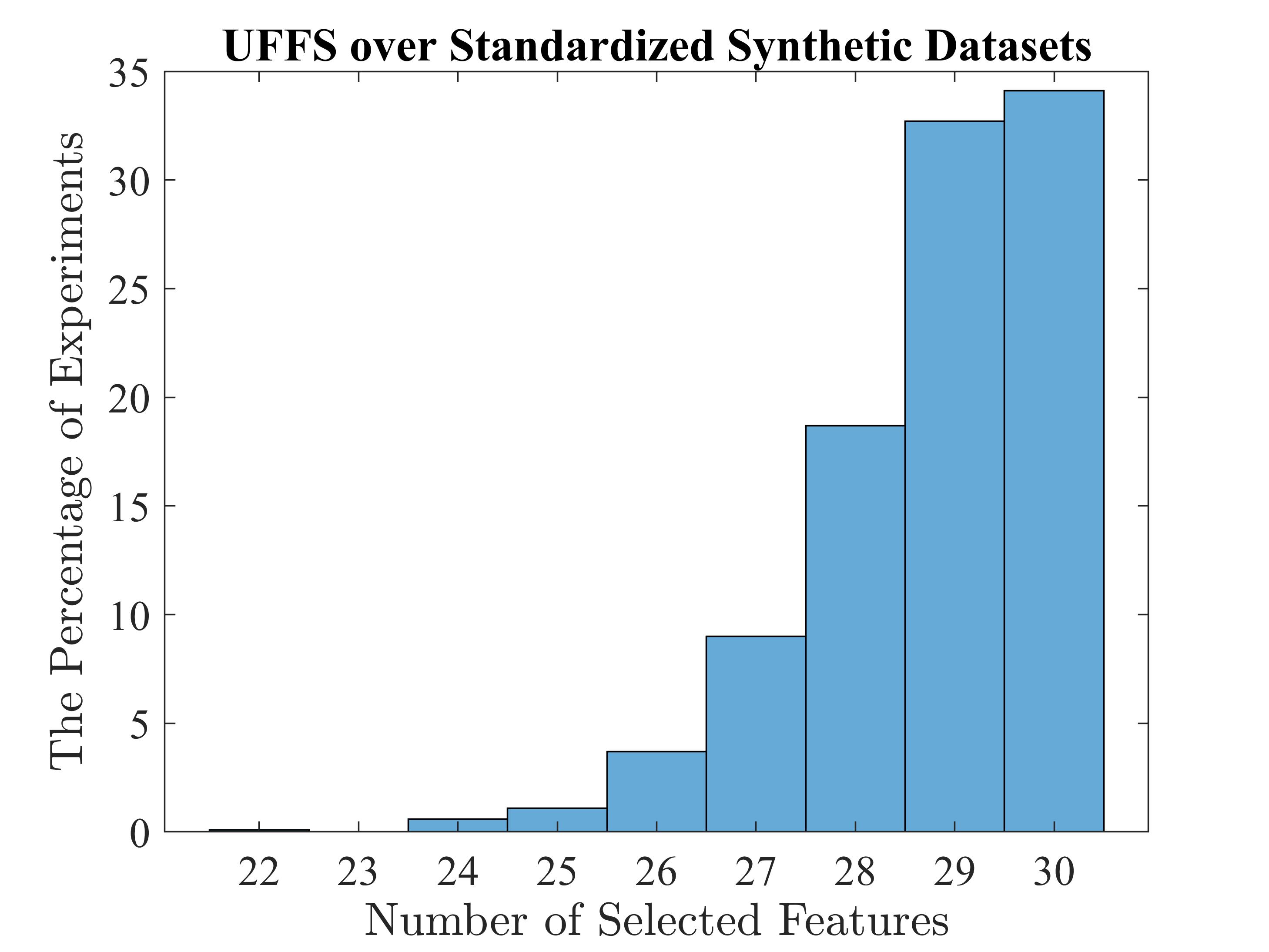}}
\caption{The percentage of experiments, out of 1000, corresponding to the number of non-redundant features selected by GFS and UFFS over synthetic datasets with $30$ features, where~$15$ features are non-redundant (mutually independent)~$X_{\ell_1},\ldots,X_{\ell_{15}}$ and the remaining~$15$ features are multilinear polynomials of these non-redundant features up to degree 3. 
(a) GFS over non-standardized datasets; (b) GFS over standardized datasets; (c) UFFS over non-standardized datasets; (d) UFFS over standardized datasets.}
\label{fig:gfs_uffs_selected_features}
\end{figure*}

\subsubsection{Classification Accuracy}
In this section, we compare GFS and UFFS in terms of classification accuracy in synthetic datasets. The main goal of these experiments is to show that GFS outperforms UFFS in datasets with nonlinear redundancies. To evaluate their performance, we use a commone method in unsupervied feature selection, where we assign labels to the data points but never use them in the feature selection procedure. After selecting non-redundant features, we use the labels for computing the error of binary classification tasks. To generate labels, we use the following polynomial threshold functions
\begin{align*}
    \textstyle f_1(X_{\ell_1},\ldots,X_{\ell_{15}}) = \operatorname{sign}\left[ \prod_{1\leq j \leq 2} \left(b_{0,j} + \sum_{i=1}^{15} b_{i,j} X_{\ell_i} \right) \right]
\end{align*}
or
\begin{align*}
    \textstyle f_2(X_{\ell_1},\ldots,X_{\ell_{15}}) = \operatorname{sign}\left[\prod_{1\leq j \leq 3} \left(b_{0,j} + \sum_{i=1}^{15} b_{i,j} X_{\ell_i} \right) \right],
\end{align*}
where $b_{i,j}\sim\operatorname{Unif}(0,1)$ and mutually independent. 
In all the experiments, GFS selects fewer features than UFFS. To have a fair comparison, we select an identical number of features using GFS and UFFS. Since UFFS does not rank the features, we complete the orthogonalization process for all features and then select the ones whose orthogonalized versions exhibit the greatest variance. 
We apply a support vector machine classifier with a radial basis function as kernel. $5$-fold cross-validation on the entire datasets is used and the results are shown in Table~\ref{table:GFS-UFFS-classification-accuracy} which show that GFS outperforms UFFS in terms of classification accuracy. 
\begin{table}[!h]
\centering
\begin{tabular}{ c | c c } 
    \toprule
    \multirow{2}{*}{Method} & \multicolumn{2}{c}{Labels} \\
    \cmidrule{2-3}
     & $f_1$ & $f_2$ \\
    \midrule
    \rowcolor{gray!20}
    GFS & 74.33 & 61.08 \\ 
    UFFS & 67.22 & 57.26 \\
    \bottomrule
\end{tabular}
\caption{Classification accuracy (\%) of GFS vs. UFFS over synthetic datasets with $30$ features, where~$15$ features are non-redundant (mutually independent)~$X_{\ell_{1}},\ldots,X_{\ell_{15}}$ and the remaining~$15$ features are multilinear polynomials up to degree 3.}
\label{table:GFS-UFFS-classification-accuracy}
\end{table}

\subsection{GFA's Performance}\label{section:experiments_GFA}
We evaluate GFA on the basis of its success in removing all redundant features in synthetic datasets. 
Similar to Section~\ref{section:GFA}, suppose~$X=(X_1,\ldots,X_d)^\intercal$. 
We let~$X_{\ell_1},\ldots,X_{\ell_n}$ for $\ell_1,\ldots,\ell_n\in[d]$ be mutually independent non-redundant features chosen from~$\cN(0,\operatorname{Unif}(0,n))$, and let~$X_{\ell_{n+1}},\ldots,X_{\ell_d}$ be functions of~$X_{\ell_1},\ldots,X_{\ell_n}$, induced by the redundancies as explained shortly. 
The experiments detailed below were repeated~$1000$ times, in each experiment GFA was applied, and its output was compared against the non-redundant features.

\subsubsection{Experiments for Corollary~\ref{corollary:zeroRedZeroApprox}}\label{section:experiments-GFA-Corollary1}
Similar to section~\ref{section:experiments-GCA-Corollary1}, We define~$t$ $0$-redundancies using two types of multilinear polynomials up to degree~$2$ or~$3$, i.e., 
\begin{align*}
    R^{(i)}&=X_{\ell_{i+n}}-k_iX_{\ell_j}X_{\ell_k} \text{ or}\\
    R^{(i)}&=X_{\ell_{i+n}}-k_iX_{\ell_j}X_{\ell_k}X_{\ell_r},
\end{align*}
where~$\ell_j,\ell_k,\ell_r$ are chosen at random, and~$k_i$ is chosen\footnote{For example, it is readily seen that $k_i<\sqrt{\min\{\var(X_{\ell_j}),\var(X_{\ell_k})\}/(\var(X_{\ell_j})\var(X_{\ell_k}))}$ suffices since the~$X_{\ell_i}$'s are mutually independent.} so that~$\var(X_{\ell_{i+n}})$ is smallest among the variables participating in~$R^{(i)}$.

We apply GFA with~$(d,n)=(30,15)$ and~$\cF(\boldz)$ being multilinear polynomials of degree up to~$2$ or~$3$, according to the degree of the~$R^{(i)}$'s mentioned above. 
In our experiments, we vary dataset sizes from $300$ to $2000$ and in each experiment GFA is applied, and its output is compared against the non-redundant features. 
We repeat the experiments for different dataset sizes~$1000$ times and the success rate of these experiments, expressed as a percentage, is reported in Table~\ref{table:GFA-corollary-1} for the two different types of redundancies. 
The results show that if the dataset size is large enough (roughly 800 in these experiments), GFA can remove all redundant features from the data. 
However, as the dataset size decreases, the approximation of the covariance matrices degrades. 
This degradation affects GFA's performance, particularly in datasets with redundancies of higher-degree multilinear polynomials. 

To show the effect of dataset dimension on GFA's performance, we also generate synthetic datasets with $d=50$ features, among them $n=25$ are non-redundant features. The success rate of GFA in removing redundant features is shown in Table~\ref{table:GFA-corollary-1}. It is obvious that as the dimension of datasets increases, more samples are needed to achieve similar performance compared to lower-dimensional datasets. 
\begin{table*}[!h]
\centering
\begin{tabular}{c | c c c c c c c c c c c c c c}
    \toprule
    \multirow{2}{*}{$\big(d,n,\operatorname{deg}(R^{(i)})\big)$} & \multicolumn{14}{c}{Dataset Size} \\
    \cmidrule{2-15}
     & 300 & 350 & 400 & 450 & 500 & 550 & 600 & 700 & 800 & 900 & 1000 & 1200 & 1500 & 2000 \\
    \midrule
    \rowcolor{gray!20}
    $(30,15,2)$ & 98.9 & 99.8 & 99.9 & 100 & 100 & 100 & 100 & 100 & 100 & 100 & 100 & 100 & 100 & 100 \\
    $(30,15,3)$ & 0 & 0.5 & 23.6 & 58.5 & 90.1 & 100 & 100 & 100 & 100 & 100 & 100 & 100 & 100 & 100 \\
    \rowcolor{gray!20}
    $(50,25,2)$ & 12.3 & 57 & 76.2 & 88.9 & 96.1 & 97.9 & 99.7 & 99.8 & 100 & 100 & 100 & 100 & 100 & 100 \\
    \bottomrule
\end{tabular}
\caption{The success rate (\%) of GFA in selecting correct non-redundant features from $1000$ experiments conducted on synthetic datasets with~$d$ features, where~$n$ features are non-redundant (mutually independent)~$X_{\ell_1},\ldots,X_{\ell_n}$, chosen from~$\cN(0,\operatorname{Unif}(0,n))$, and the remaining~$d-n$ features are $0$-redundancies of multilinear polynomials up to degree $\operatorname{deg}(R^{(i)})$ (Corollary~\ref{corollary:zeroRedZeroApprox}).}
\label{table:GFA-corollary-1}
\end{table*}

\subsubsection{Experiments for Corollary~\ref{corollary:some0someeps}} 
We repeat similar experiments to Section~\ref{section:experiments-GFA-Corollary1} with the addition of noise, i.e.,
\begin{align*}
    R^{(i)}&=X_{\ell_{i+n}}-k_iX_{\ell_j}X_{\ell_k}+\delta_i \text{ or}\\
    R^{(i)}&=X_{\ell_{i+n}}-k_iX_{\ell_j}X_{\ell_k}X_{\ell_r}+\delta_i,
\end{align*}
with mutually independent~$\delta_i\sim\cN(0,0.1)$.
Again Similar to section~\ref{section:experiments-GFA-Corollary1}, we apply GFA with~$\cF(\boldz)$ being multilinear polynomials of degree up to~$2$ or~$3$.
The parameter~$\epsilon$ in Algorithm~\ref{algorithm:GFA} is chosen as the empirical variance of the added noise $\delta_i$ in each experiment.  
The results, which are summarized in Table~\ref{table:GFA-corollary-2}, demonstrate resilience of GFA to stochastic noise. 
\begin{table*}[!h]
\centering
\begin{tabular}{c | c c c c c c c c c c c c c c}
    \toprule
    \multirow{2}{*}{$\big(d,n,\operatorname{deg}(R^{(i)})\big)$} & \multicolumn{14}{c}{Dataset Size} \\
    \cmidrule{2-15}
     & 300 & 350 & 400 & 450 & 500 & 550 & 600 & 700 & 800 & 900 & 1000 & 1200 & 1500 & 2000 \\
    \midrule
    \rowcolor{gray!20}
    $(30,15,2)$ & 82.9 & 86.6 & 88.6 & 90.8 & 92.8 & 94 & 94.4 & 96.6 & 97 & 97.9 & 98.8 & 98.8 & 99.5 & 99.9 \\
    $(30,15,3)$ & 1.24 & 2.5 & 27.3 & 60.1 & 72.1 & 79.3 & 80.8 & 81.7 & 82.3 & 83.8 & 90.9 & 96 & 98.8 & 99.8 \\
    \rowcolor{gray!20}
    $(50,25,2)$ & 8.7 & 43.2 & 52.8 & 63.8 & 72.2 & 77.1 & 82.2 & 86.9 & 91.8 & 93.9 & 94.9 & 96.2 & 99.1 & 99.8 \\
    \bottomrule
\end{tabular}
\caption{The success rate (\%) of GFA in selecting correct non-redundant features from $1000$ experiments conducted on synthetic datasets with~$d$ features, where~$n$ features are non-redundant (mutually independent)~$X_{\ell_1},\ldots,X_{\ell_n}$, chosen from~$\cN(0,\operatorname{Unif}(0,n))$, and the remaining~$d-n$ features are $\epsilon$-redundancies (stochastic noise) of multilinear polynomials up to degree $\operatorname{deg}(R^{(i)})$ (Corollary~\ref{corollary:some0someeps}).}
\label{table:GFA-corollary-2}
\end{table*}

\subsubsection{Experiments for Corollary~\ref{corollary:deg5polys}} 
We repeat experiments similar to Section~\ref{section:experiments-GFA-Corollary1}, with an added mismatch between the degree of the~$0$-redundancies and the degree of the polynomials in~$\cF(\boldz)$. That is, we take 
\begin{align*}
    R^{(i)}&=X_{\ell_{i+n}}-k_iX_{\ell_j}X_{\ell_k}X_{\ell_r}X_{\ell_a}, \text{ or}\\
    R^{(i)}&=X_{\ell_{i+n}}-k_iX_{\ell_j}X_{\ell_k}X_{\ell_r}X_{\ell_a}X_{\ell_b},
\end{align*}
but apply GFA with~$\cF(\boldz)$ containing multilinear polynomials only up to degree three.
The~$\epsilon$ parameter in GFA is chosen as the maximum value of 
\begin{align*}
    &\bE[(X_{\ell_{i+n}}- \proj_{\cF(X_{\ell_j},X_{\ell_k},X_{\ell_r},X_{\ell_a})}X_{\ell_{i+n}})^2], \text{or}\\
    &\bE[(X_{\ell_{i+n}}-\proj_{\cF(X_{\ell_j},X_{\ell_k},X_{\ell_r},X_{\ell_a},X_{\ell_b})}X_{\ell_{i+n}})^2],
\end{align*}
respectively.
The results, which are summarized in Table~\ref{table:GFA-corollary-3}, demonstrate resilience of GFA to so-called \textit{deterministic} noise, i.e., the inability of~$\cF(\boldz)$ to represent the redundancies.
\begin{table*}[!h]
\centering
\begin{tabular}{c | c c c c c c c c c c c c c c}
    \toprule
    \multirow{2}{*}{$\big(d,n,\operatorname{deg}(R^{(i)})\big)$} & \multicolumn{14}{c}{Dataset Size} \\
    \cmidrule{2-15}
     & 300 & 350 & 400 & 450 & 500 & 550 & 600 & 700 & 800 & 900 & 1000 & 1200 & 1500 & 2000 \\
    \midrule
    \rowcolor{gray!20}
    $(30,15,4)$ & 20.3 & 38.4 & 62.1 & 55.8 & 64.7 & 69.1 & 81.2 & 83.2 & 86.2 & 91.1 & 94.7 & 97.2 & 99.4 & 100 \\
    $(30,15,5)$ & 0 & 3.4 & 13.5 & 22.4 & 37.4 & 49.3 & 62.5 & 63.4 & 78.9 & 83.4 & 89.9 & 91.7 & 94.3 & 98.9 \\
    \rowcolor{gray!20}
    $(50,25,4)$ & 10.1 & 10.2 & 11.2 & 32.3 & 48.3 & 51.1 & 67.4 & 72.1 & 82.2 & 84.3 & 92.3 & 93.1 & 95.6 & 99.7 \\
    \bottomrule
\end{tabular}
\caption{The success rate (\%) of GFA in selecting correct non-redundant features from $1000$ experiments conducted on synthetic datasets with~$d$ features, where~$n$ features are non-redundant (mutually independent)~$X_{\ell_1},\ldots,X_{\ell_n}$, chosen from~$\cN(0,\operatorname{Unif}(0,n))$, and the remaining~$d-n$ features are $\epsilon$-redundancies (deterministic noise) of multilinear polynomials up to degree $\operatorname{deg}(R^{(i)})$ (Corollary~\ref{corollary:deg5polys}).}
\label{table:GFA-corollary-3}
\end{table*}

\section{Discussion and Future Work}\label{section:discussion}

The algorithms presented herein rely on an orthogonalization process of an arbitrary function family~$\cF(\boldz)$, whose variables are substituted one-by-one with functions of the data random variable~$X$. 
While we pose few constraints on the choice of~$\cF(\boldz)$, it is evident that informative dimension reduction strongly relies on choosing a ``proper'' function family for the data distribution.

The choice of~$\cF(\boldz)$ should correspond to a prior belief that the redundancies can be described, or can be well-approximated by, the functions in~$\Span\cF(\boldz)$. 
Several clear interpretations of the required correspondence between~$\cF(\boldz)$ and~$X$ are given in Section~\ref{section:GCA} (also Section~\ref{section:GFA}).
 
More broadly, the choice of~$\cF(\boldz)$ is tightly connected to Fourier analysis, an area which studies representation of functions over orthogonal bases (w.r.t some underlying distribution). 
For example, it is known~\cite{o2014analysis} that under the uniform distribution over the discrete hypercube~$\{\pm1\}^t$, the set of multilinear polynomials in~$t$ variables is an orthonormal basis to the set of functions of the form~$f:\{\pm1\}^t\to \bR$; cases where the distribution is not uniform can be treated using the GS process. 
Continuous variants exist (e.g., Hermite polynomials for Gaussian distribution), and the interested reader is referred to~\cite{ismail2017review}. 
Therefore, it is reasonable to assume that letting~$\cF(\boldz)$ be low-degree polynomials is a safe choice, and this is very clearly demonstrated in our experimental section. 
However, one might also suggest using Fourier basis, which is left for future work.

We once again emphasize the different approach taken in GFR and GFS vs. the approach taken in GCA and GFA.
In GFR and GFS, we extract/select variance maximizers of the \textit{reduced} data random variable~$d_j(X)$, whereas in GCA and GFA we merely use the~$d_j(X)$ to identify which features/components need to be subtracted from~$X$ itself. 
The latter enables a finer understanding of the exact structure of data distributions on which it is effective, and its correspondence with~$\cF(\boldz)$.
For the former we are not able to provide such analyses, but we are able to provide entropy reduction guarantees.

For future work we propose to study more closely which redundancy structures can be accommodated using GFR, GFS, and potentially, waive the requirement for discrete alphabets in their entropy reduction guarantees. 
Second, we implicitly assumed that expectations can be computed exactly, which is clearly not the case in practice. 
Nevertheless, our extensive experiments show that this does not pose a significant barrier on many benchmark and synthetic datasets. 
However, it is still imperative to study convergence results, i.e., the required dataset sizes for which empirical means are sufficiently close to the true expectations.
Although GCA is proposed to eliminate redundancy among the principal components, it can be applied using any orthonormal basis other than the principal directions. In our future work, we propose to study the extension of GCA to other linear and nonlinear (kernel-based) feature extraction methods.

\appendices

\section{Proof of Theorem~\ref{theorem:GFS}}\label{appendix:GFS}
Similar to the Proof of Theorem~\ref{theorem:GFR}, we begin by showing that the variance of the data in all features which are not selected in smaller than~$\epsilon^2$. 
To this end, let~$Z^\bot\triangleq(Z_{m+1},\ldots,Z_d)$ be the features of~$X$ that were \textit{not} selected by GFS, in any arbitrary order. 
Also, suppose we complete the orthogonalization process which GFS started and obtain $\hat{\cF}(Z_1,\ldots,Z_d)$, as well as the respective variance vector~$\boldsigma_{m+1},\ldots,\boldsigma_d$.

\begin{lemma}\label{lemma:alldirectionsFS}
For every~$i\in\{m+1,\ldots,d\}$ we have $\norm{\boldsigma_i}_{\infty}\le\epsilon^2$.
\end{lemma}
The following technical statement is required for the subsequent proof of Lemma~\ref{lemma:alldirectionsFS}. For brevity we denote~$\hat{\cF}_j\triangleq\hat{\cF}(Z_1,\ldots,Z_j)$ for every~$j\in[d]$.

\begin{lemma}\label{lemma:GFSaux}
    In GFS, we have the following.
    \begin{enumerate}
        \item[$(a)$] For~$i<j$ in~$[d]$ we have \\$\boldsigma_j = \boldsigma_i-\sum_{ f\in\hat{\cF}_{j-1}\setminus\hat{\cF}_{i-1}}\bE[ fX]^{\otimes 2}$.
        \item[$(b)$] The integers~$s_1,\ldots,s_m$ are distinct.
    \end{enumerate}
\end{lemma}

\begin{proof}[Proof of Lemma~\ref{lemma:GFSaux}.(a)]
It is readily verified that~$d_{j-1}(X)$ can be written as
\begin{align}\label{equation:dj-1-FS}
    d_{j-1}(X)=d_{i-1}(X)-\sum_{ f\in\hat{\cF}_{j-1}\setminus\hat{\cF}_{i-1}}\bE[X f] f,
\end{align}
and that~$d_{i-1}(X)$ can be written as
\begin{align}\label{equation:di-1-FS}
    d_{i-1}(X)=X-\textstyle\sum_{ f\in\hat{\cF}_{i-1}}\bE[X f] f.
\end{align}
Hence, it follows from~\eqref{equation:dj-1-FS} that
\begin{align}\label{equation:sigmajassigmai-FS}
    \boldsigma_j&=\bE[d_{j-1}(X)^{\otimes 2}]\nonumber\\
    &=\textstyle\bE\bigg[\left(d_{i-1}(X)-\sum_{ f\in\hat{\cF}_{j-1}\setminus\hat{\cF}_{i-1}}\bE[X f] f\right)^{\otimes 2}\bigg]\nonumber\\    
    &= \boldsigma_i-\textstyle\sum_{ f\in\hat{\cF}_{j-1}\setminus\hat{\cF}_{i-1}}\bE[  fd_{i-1}(X)]\otimes\bE[X f]\nonumber\\
    &\phantom{=\boldsigma_i}-\textstyle\sum_{ f\in\hat{\cF}_{j-1}\setminus\hat{\cF}_{i-1}}\bE[ fX]\otimes\bE[ d_{i-1}(X) f]\nonumber\\
    &\phantom{=\boldsigma_i}+\textstyle\sum_{ f\in \hat{\cF}_{j-1}\setminus\hat{\cF}_{i-1}} \bE[X f]^{\otimes 2},
\end{align}
where the last summand follows from the orthonormality of~$\hat{\cF}$. Further, it follows from~\eqref{equation:di-1-FS} that every~$ f\in\hat{\cF}_{j-1}\setminus\hat{\cF}_{i-1}$ satisfies
\begin{align*}
    \bE[ fd_{i-1}(X)]=\bE[ f\cdot(X-\textstyle\sum_{ g\in\hat{\cF}_{i-1}}\bE[X g] g)]=\bE[ fX],
\end{align*}
and therefore~$\eqref{equation:sigmajassigmai-FS}=\boldsigma_i-\sum_{ f\in\hat{\cF}_{j-1}\setminus\cF_{i-1}}\bE[X f]^{\otimes 2}$ as required.
\end{proof}

\begin{proof}[Proof of Lemma~\ref{lemma:GFSaux}.(b)]
Notice that for every~$i\in[d]$ and every~$r\in[m]$ we have
\begin{align}\label{equation:eisigmaei_}
    \sigma_{r,i}=\bE[d_{r-1}(X)_{i}^2].
\end{align}
and since $d_{r-1}(X)=X-\sum_{ f\in\hat{\cF}_{r-1}}\bE[X  f] f$, it follows that 
\begin{align}\label{equation:alternativemaximizer_}
    \eqref{equation:eisigmaei_}=\bE\left[\left( X_{i}-\textstyle\sum_{ f\in\hat{
    \cF}_{r-1}}\bE[X_{i} f] f \right)^2\right]
\end{align}
That is, at the beginning of the~$r$'th iteration for any~$r$, GFS will find the maximizer~$s_r$ over~$i$ of~\eqref{equation:alternativemaximizer_}. 
Now, observe that if~$i=s_k$ for some~$k\in\{1,\ldots,r-1\}$, i.e., if the~$i$'th feature was already selected in an earlier iteration~$k$, then~$Z_k=X_i$. 
Let~$\tilde{Z}_k$ and~$\hat{Z}_k$ be the orthogonalized and orthonormalized versions of~$Z_k$.
Since~$Z_k = \Tilde{Z}_k + \sum_{ f\in\hat{\cF}_{k-1}}\bE[Z_k f] f$, we have
\begin{align}\label{equation:alternativemaximizer__}
\eqref{equation:alternativemaximizer_}
&=\textstyle \bE\bigg[\Big(
    \Tilde{Z}_k 
    + \sum_{g\in \hat{\cF}_{k-1}} \bE[Z_k g]\, g \nonumber\\
&\qquad\textstyle
    - \sum_{f\in\hat{\cF}_{r-1}} 
        \bE\Big[
            \Big(\Tilde{Z}_k 
            + \sum_{g\in \hat{\cF}_{k-1}} \bE[Z_k g]\, g 
            \Big) f
        \Big] f 
    \Big)^2\bigg] \nonumber\\[0.5em]
&=\textstyle \bE\bigg[\Big(
    \Tilde{Z}_k 
    + \sum_{g\in \hat{\cF}_{k-1}} \bE[Z_k g]\, g \nonumber\\
&~~~\textstyle
    - \sum_{f\in\hat{\cF}_{r-1}} 
        \Big(
            \bE[\Tilde{Z}_k f] 
            + \sum_{g\in \hat{\cF}_{k-1}} \bE[Z_k g] \bE[g f]
        \Big) f 
    \Big)^2\bigg] \nonumber\\[0.5em]
&=\textstyle \bE\bigg[\Big(
    \Tilde{Z}_k 
    + \sum_{g\in \hat{\cF}_{k-1}} \bE[Z_k g]\, g \nonumber\\
&\qquad\textstyle
    - \sum_{f\in\hat{\cF}_{r-1}} \bE[\Tilde{Z}_k f]\, f 
    - \sum_{g\in \hat{\cF}_{k-1}} \bE[Z_k g]\, g 
    \Big)^2\bigg] \nonumber\\[0.5em]
&=\textstyle \bE\bigg[\Big(
    \Tilde{Z}_k 
    - \sum_{f\in\hat{\cF}_{r-1}} \bE[\Tilde{Z}_k f]\, f 
    \Big)^2\bigg].
\end{align}
Since~$\Tilde{Z}_k$ is orthogonalized, it follows that 
\begin{align*}
    \bE[\tilde{Z}_kf]=
    \begin{cases}
        0 & \text{if }f\in \hat{\cF}_{r-1}\setminus\{\hat{Z}_k\}\\
        \textstyle\Vert \tilde{Z}_k\Vert & \text{if }f=\hat{Z}_k
    \end{cases},
\end{align*}
and hence
\begin{align*}
    \eqref{equation:alternativemaximizer__}=\bE[(\Tilde{Z}_k-\hat{Z}_k\cdot\Vert\Tilde{Z}_k\Vert)^2]=0,
\end{align*}
It then follows that the maximization problem will not select an index~$i$ that was already selected at a previous iteration.
\end{proof}

\begin{proof}[Proof of Lemma \ref{lemma:alldirectionsFS}]
According to the stopping criterion in GFS, it follows that
\begin{align}\label{equation:conclusionsFromStopping}
    \norm{\boldsigma_i}_{\infty}&>\epsilon^2\mbox{ for all }i\in[m]; \mbox{ and}\nonumber\\
    \norm{\boldsigma_{m+1}}_{\infty}&\le\epsilon^2.
\end{align}
Assume for contradiction that there exists~$i\in\{m+1,\ldots,d\}$ such that~$\norm{\boldsigma_i}_{\infty}>\epsilon^2$, and observe that~$i>m+1$, since otherwise the existence of~$s_{m+1}$ contradicts~\eqref{equation:conclusionsFromStopping}. 
It follows from Lemma~\ref{lemma:GFSaux}.(a) that
\begin{align*}
    \textstyle\boldsigma_i=\boldsigma_{m+1}-\sum_{ f\in\hat{\cF}_{i-1}\setminus\hat{\cF}_{m}}\bE[ fX]^{\otimes 2},
\end{align*}
and therefore
\begin{align*}
    \norm{\boldsigma_i}_{\infty} = \sigma_{i,s_i} &=\textstyle \sigma_{m+1,s_i}-\sum_{ f\in\hat{\cF}_{i-1}\setminus\hat{\cF}_{m}}\bE[ fX_{s_i}]^2 \\
    &\overset{\eqref{equation:conclusionsFromStopping}}{\le}\textstyle\epsilon^2-\sum_{ f\in\hat{\cF}_{i-1}\setminus\hat{\cF}_{m}}\bE[ fX_{s_i}]^2.
\end{align*}
Therefore, since~$\norm{\boldsigma_i}_{\infty}>\epsilon^2$, it follows that 
\begin{align*}
    \textstyle-\sum_{ f\in\hat{\cF}_{i-1}\setminus\hat{\cF}_{m}}\bE[ fX_{s_i}]^2>0,
\end{align*}
which is a contradiction since $\bE[ fX_{s_i}]^2\ge 0$ for every~$ f$.
\end{proof}

\begin{proof}[Proof of Theorem~\ref{theorem:GFS}]
Let~$\{s_{m+1},\ldots,s_d\}=[d]\setminus\{s_1,\ldots,s_m\}$ and recall that~$Z=X_{\cS_\epsilon}$ and~$Z^\bot=(Z_{m+1},\ldots,Z_d)$. Observe that
\begin{align}\label{equation:Hzx-FS}
    H(X\vert Z)&=H(\{Z_{i}\}_{i=1}^d\vert \{Z_{i}\}_{i=1}^m)\nonumber\\
    &=\textstyle \sum_{i=m+1}^d H(Z_{i}\vert (Z_{j})_{j=1}^{i-1}),
\end{align}
where the last equality follows from the chain rule for information entropy. 
Since
\begin{align}\label{equation:ziTilde-FS}
    \Tilde{Z}_i=Z_{i}-\sum_{ f\in \hat{\cF}_{i-1}}\bE[Z_{i} f] f,
\end{align}
it follows that
\begin{align}\label{equation:Hztilde-FS}
    \textstyle H(Z_i\vert (Z_j)_{j=1}^{i-1})&=H(\Tilde{Z}_i+\sum_{ f\in \hat{\cF}_{i-1}}\bE[Z_i f] f\vert (Z_j)_{j=1}^{i-1})\nonumber\\
    &\overset{(a)}{=}H(\Tilde{Z}_i\vert (Z_j)_{j=1}^i)\overset{(b)}{\le}H(\Tilde{Z}_i),
\end{align}
where~$(a)$ follows since the expression $\sum_{ f\in \hat{\cF}_{i-1}}\bE[Z_i f] f$ is uniquely determined by the variables~$Z_1,\ldots,Z_{i-1}$ (every~$ f\in\hat{\cF}_{i-1}$ is a deterministic function of~$Z_1,\ldots,Z_{i-1}$, and the coefficients~$\bE[Z_i f]$ are constants), and~$(b)$ follows since conditioning reduces entropy. Combining~\eqref{equation:Hzx-FS} with~\eqref{equation:Hztilde-FS} we have that~$H(X\vert Z)\le\sum_{i=m+1}^dH(\Tilde{Z}_i)$, and hence it remains to bound~$H(\Tilde{Z}_i)$ for all~$i\in\{m+1,\ldots,d\}$.

To this end let~$i\in\{m+1,\ldots,d\}$, define~$\alpha_i\triangleq\min\{|\Tilde{Z}_i(\boldalpha)|:\boldalpha\in\cX^d, \Tilde{Z}_i(\boldalpha)\ne 0\}$ and~$\alpha_{\min}=\min\{\alpha_i\}_{i=m+1}^d$, and observe that by Markov's inequality we have
\begin{align}\label{equation:PrziTile-FS}
    \Pr(\Tilde{Z}_i\ne 0)=\Pr(|\Tilde{Z}_i|\ge \alpha_i)&=\Pr(\Tilde{Z}_i^2\ge \alpha_i^2)\nonumber\\
    &\le \frac{\bE[\Tilde{Z}_i^2]}{\alpha_i^2}\le \frac{\bE[\Tilde{Z}_i^2]}{\alpha_{\min}^2}.
\end{align}
Moreover, recall that
\begin{align*}
    d_{i-1}(X)&=X-\textstyle\sum_{ f\in\hat{\cF}_{i-1}}\bE[X f] f\mbox{, and}\\
    \boldsigma_i&=\bE[d_{i-1}(X)^{\otimes 2}],
\end{align*}
and therefore
\begin{align*}
    \norm{\boldsigma_i}_{\infty}=\sigma_{i,s_i}&=\bE[d_{i-1}(X)_{s_i}^2]\\
    &= \bE[(X_{s_i}-\textstyle\sum_{ f\in\hat{\cF}_{i-1}}\bE[X_{s_i} f] f)^2] \\
    &= \bE[(Z_i-\textstyle\sum_{ f\in\hat{\cF}_{i-1}}\bE[Z_i f] f)^2] \\
    &\overset{\eqref{equation:ziTilde-FS}}{=}\bE[\Tilde{Z}_i^2].
\end{align*}
Consequently, it follows that
\begin{align*}
    \eqref{equation:PrziTile-FS}=\frac{\norm{\boldsigma_{i}}_{\infty}}{a_{\min}^2}\overset{\text{Lemma}~\ref{lemma:alldirectionsFS}}{\le}\frac{\epsilon^2}{a_{\min}^2}.
\end{align*}
Now, by using the grouping rule \cite[Ex.~2.27]{cover1999elements} we have
\begin{align*}
    H(\Tilde{Z}_i)\le h_b(\epsilon^2/\alpha_{\min}^2)+\frac{\epsilon^2}{\alpha_{\min}^2}\log|\cX|,
\end{align*}
where~$h_b$ is the binary entropy function. Finally, using the bound~$h_b(p)\le2\sqrt{p(1-p)}\le2\sqrt{p}$, it follows that
\begin{align*}
    H(X\vert Z)&\le \textstyle\sum_{i=m+1}^d H(\Tilde{Z}_i)\\
    &\le(d-m)(h_b(\epsilon^2/\alpha_{\min}^2)+(\epsilon^2/\alpha_{\min}^2)\log|\cX|)\\
    &\le(d-m)(2\epsilon/\alpha_{\min}+(\epsilon/\alpha_{\min})\log|\cX|)=dO(\epsilon).\qedhere
\end{align*}
\end{proof}
\section{Complexity analysis}\label{section:complexity_analysis}
The computation bottleneck of our algorithm is the orthogonalization process (Algorithm~\ref{algorithm:orthogonalize}). Theorem~\ref{theorem:complexity} shows the complexity upper bound on our algorithms.
\begin{theorem}\label{theorem:complexity}
Let~$c(\cF)$ be an upper bound on the time complexity of computing any~$f\in\cF$ (once the variables~$z_j$ are known), and let~$N$, $d$, and $m$ be the number of samples, dataset dimension, and the number of extracted features, respectively. Then, the complexity upper bound on GCA and GFR is
\begin{align}
    O(|\hat{\cF}_m|^2Nc(\cF)+|\hat{\cF}_m|(N(c(\cF)+d)+d^2)+d^2(N+d)).
\end{align}
\end{theorem}
\begin{proof}
The orthogonalization process requires to subtract from each~$f_i$ its projections along~$f_1,\ldots,f_{i-1}$, each of which carries a coefficient of the form~$\bE[f_i f_j]$ for~$j<i$ which can be empirically estimated in~$O(Nc(\cF))$. Letting~$m$ be the number of extracted features, this sums up to~$O(Nc(\cF)|\hat{\cF}_m|^2)$.

The computation of the matrices~$\{\Sigma_j\}_{j=1}^m$ can be done recursively using\footnote{The distributions~$d_{j}$ are merely proof artifacts, and need not be computed in practice.} Lemma~\ref{lemma:GFRaux}.(a). Since
$$\Sigma_{j+1}=\Sigma_j-\sum_{f\in\hat{\cF}_j\setminus\hat{\cF}_{j-1}}\bE[fX]\bE[X^\intercal f],$$
at iteration~$j$ the matrix~$\Sigma_{j+1}$ can be approximated in~$O(|\hat{\cF}_j\setminus\hat{\cF}_{j-1}|N(c(\cF)+d)+d^2|\hat{\cF}_j\setminus\hat{\cF}_{j-1}|)$, since each~$\bE[fX]$ can be approximated in~$O(N(c(\cF)+d))$, and then~$O(|\hat{\cF}_j\setminus\hat{\cF}_{j-1}|)$ operations are required per entry. Summing over~$j\in[m]$ yields~$O(|\hat{\cF}_m|(N(c(\cF)+d)+d^2))$. Therefore, all orthogonalization operations in the process of extracting~$m$ features require $O(|\hat{\cF}_m|^2Nc(\cF)+|\hat{\cF}_m|(N(c(\cF)+d)+d^2))$.

Outside the orthogonalization operations, GCA (Algorithm~\ref{algorithm:GCA}) requires algebraic and min/max operations, as well as covariance matrix approximation ($\bE[\bar{X}_j\Bar{X}_j^\intercal]$), all dominated by~$O(d^2(N+d))$. GFR (Algorithm~\ref{algorithm:GFR}) also requires solving the largest eigenvector problem, which can be done in~$O(d^3)$, but iterative methods in~$O(d^2)$ exist. Hence, a rough complexity upper bound on GCA and GFR is
\begin{align*}
    &O(|\hat{\cF}_m|^2Nc(\cF)+|\hat{\cF}_m|(N(c(\cF)+d)+d^2)+d^2(N+d)).\qedhere
\end{align*}
\end{proof}

Similar complexity upper bound on GFS and GFA can also be shown as follows.
\begin{theorem}\label{theorem:complexity_FS}
Let~$c(\cF)$ be an upper bound on the time complexity of computing any~$f\in\cF$ (once the variables~$z_j$ are known), and let~$N$, $d$, and $m$ be the number of samples, dataset dimension, and the number of selected features, respectively. Then, the complexity upper bound on GFA and GFS is
\begin{align}
    O(|\hat{\cF}_m|^2Nc(\cF)+|\hat{\cF}_m|(N(c(\cF)+d)+d)+d(N+d)).
\end{align}
\end{theorem}
\begin{proof}
As shown in the proof of Theorem~\ref{theorem:complexity}, the computational complexity of the orthogonalization process is~$O(Nc(\cF)|\hat{\cF}_m|^2)$. The computation of the variance vectors~$\{\boldsigma_j\}_{j=1}^m$ can be done recursively 
$$\boldsigma_{j+1}=\boldsigma_j-\sum_{f\in\hat{\cF}_j\setminus\hat{\cF}_{j-1}}\bE[fX]^{\otimes 2},$$
at iteration~$j$ the vector~$\boldsigma_{j+1}$ can be approximated in~$O(|\hat{\cF}_j\setminus\hat{\cF}_{j-1}|N(c(\cF)+d)+d|\hat{\cF}_j\setminus\hat{\cF}_{j-1}|)$, since each~$\bE[fX]$ can be approximated in~$O(N(c(\cF)+d))$, and then~$O(|\hat{\cF}_j\setminus\hat{\cF}_{j-1}|)$ operations are required per entry. Summing over~$j\in[m]$ yields~$O(|\hat{\cF}_m|(N(c(\cF)+d)+d))$. Therefore, all orthogonalization operations in the process of extracting~$m$ features require $O(|\hat{\cF}_m|^2Nc(\cF)+|\hat{\cF}_m|(N(c(\cF)+d)+d))$.Outside the orthogonalization operations, GFA (Algorithm~\ref{algorithm:GFA}) finds the highest variant feature with~$O(d(N+d))$ complexity. GFS (Algorithm~\ref{algorithm:GFS}) also requires similar computational complexity. Hence, a rough complexity upper bound on GFA and GFS is
\begin{align*}
    &O(|\hat{\cF}_m|^2Nc(\cF)+|\hat{\cF}_m|(N(c(\cF)+d)+d)+d(N+d)).\qedhere
\end{align*}
\end{proof}

\section{Unsupervised Fourier Feature Selection, a comparison}\label{section:previousFourier}
A fascinating feature selection framework has been recently put forward by~\cite{heidari2022sufficiently} (see also \cite{heidari2021information,heidari2021finding}). In this framework, the random variable~$X=(X_1,\ldots,X_d)$ is viewed from a Fourier-analytic perspective in the following sense. A set of \textit{characters}~$\{\phi_\cS\}_{\cS\subseteq[d]}$ is defined as all multilinear monomials in the (centered and normalized, and possibly dependent) variables~$\{X_i\}_{i=1}^d$, i.e., $\phi_\cS(X)=\prod_{s\in\cS} X_s$. These characters are then orthogonalized as~$\{\Tilde{\psi}_\cS\}_{\cS\subseteq[d]}$ with respect to the data distribution using a Gram-Schmidt process, 
\begin{align}\label{equation:characterorthogonalization}
    \Tilde{\psi}_{\cS_i}&=\phi_{\cS_i}-\sum_{j=1}^{i-1}\bE[\psi_{\cS_j}\phi_{\cS_i}]\psi_{\cS_j},\mbox{ where}\nonumber\\
    \psi_{\cS_i}&=  \begin{cases}
        \frac{\Tilde{\psi}_{\cS_i}}{\norm{\Tilde{\psi}_{\cS_i}}} & \mbox{if }\Vert\Tilde{\psi}_{\cS_i}\Vert\ne 0\\
        0 & \mbox{otherwise},
    \end{cases}
\end{align}
and where the sets~$\cS\subseteq[d]$ are indexed in the following order:
\begin{align}\label{equation:setorder}
    \varnothing,\{1\},\{2\},\{1,2\},\{3\},\{1,3\},\{2,3\},\nonumber\\
    \{1,2,3\},\ldots,\{1,2,\ldots,d\}.
\end{align}

Then, it is shown that the Fourier norm $\Vert\Tilde{\psi}_{j}\Vert=\sqrt{\bE[\Tilde{\psi}_{j}^2]}$, where~$\Tilde{\psi}_{j}$ is a shorthand notation for~$\Tilde{\psi}_{\{j\}}$, can be seen as a measure for the ``informativeness'' of feature~$j$ in the following sense.
\begin{theorem}\cite[Thm.~1]{heidari2022sufficiently}\label{theorem:Heidari}
    Let~$\cS_\epsilon\subseteq[d]$ be the set of indices~$j$ such that $\Vert\Tilde{\psi}_{j}\Vert>\epsilon$. Then~$H(X\vert X^{\cS_\epsilon})=dO(\epsilon)$, where $X^{\cS_\epsilon}=(X_j)_{j\in\cS_\epsilon}$.
\end{theorem}

That is, selecting those features whose respective Fourier norm is high, reduces the information entropy of the data, and therefore those features are ``sufficiently informative.'' In practice, orthogonalization is performed by approximating expectations using empirical means, then the Fourier norms~$\Vert\Tilde{\psi}_{j}\Vert$ are calculated similarly, and the features whose norms are larger than~$\epsilon$ are selected.

\begin{algenvironment}[Unsupervised Fourier Feature Selection (UFFS)]\label{algorithm:Heidari}
    Theorem~\ref{theorem:Heidari} implies a feature selection algorithm, in which orthogonal characters are constructed as in~\eqref{equation:characterorthogonalization},~\eqref{equation:setorder}, the Fourier norms $\Vert\Tilde{\psi}_{j}\Vert$ are computed for every~$j$ (approximating expectations by empirical means), and the~$j$'s whose norms are larger than a chosen threshold~$\epsilon>0$ are selected.
\end{algenvironment}
In what follows we briefly explain that GFS, when choosing~$\cF$ as the set of all multilinear polynomials, specifies to the above framework by~\cite{heidari2022sufficiently}, yet at reduced complexity. A key observation regarding Theorem~\ref{theorem:Heidari}, which enables GFS to subsume Algorithm~\ref{algorithm:Heidari} at reduced complexity, is that the order of sets in~\eqref{equation:setorder} is rather arbitrary, and for any permutation~$\pi$ over~$[d]$ we may set
\begin{align}\label{equation:characterorthogonalizationpi}
    \Tilde{\psi}_{\cS_i^{\pi}}&=\phi_{\cS_i^{\pi}}-\sum_{j=1}^{i-1}\bE[\psi_{\cS_j^{\pi}}\phi_{\cS_i^{\pi}}]\psi_{\cS_j^{\pi}},\mbox{ where}\nonumber\\
    \psi_{\cS_i^{\pi}}&=  \begin{cases}
        \frac{\Tilde{\psi}_{\cS_i^{\pi}}}{\Vert\Tilde{\psi}_{\cS_i^{\pi}}\Vert} & \mbox{if }\Vert\Tilde{\psi}_{\cS_i^{\pi}}\Vert\ne 0\\
        0 & \mbox{otherwise},
    \end{cases}
\end{align}
and where~$\{\cS_{i}^\pi\}_{i=1}^{2^d}$ is the order
\begin{align*}
    &\varnothing,\{\pi(1)\},\{\pi(2)\},\{\pi(1),\pi(2)\},\{\pi(3)\},\{\pi(1),\pi(3)\},\\ 
    &\{\pi(2),\pi(3)\},\{\pi(1),\pi(2),\pi(3)\},\ldots,\{\pi(1),\pi(2),\ldots,\pi(d)\},
\end{align*}
and the guarantees in Theorem~\ref{theorem:Heidari} still hold true. GFS, however, finds such a permutation~$\pi$ on the fly, one element at a time; the next element to be chosen is the one which maximizes~$\{\sigma_{r,i}\}_{i=1}^d$. Now, since
\begin{align}\label{equation:eisigmaei}
    \sigma_{r,i}&=\bE[d_{r-1}(X)_{i}^2],\mbox{ and since}\\\nonumber
    d_{r-1}(X)&=X-\sum_{\ell=1}^{2^{r-1}}\bE[X\psi_{\cS_{\ell}^\pi}]\psi_{\cS_{\ell}^\pi},
\end{align}
it follows that
\begin{align}\label{equation:alternativemaximizer}
    \eqref{equation:eisigmaei}&=\bE\left[\left( X_i-\textstyle\sum_{\ell=1}^{2^{r-1}}\bE[X_i\psi_{\cS_\ell^\pi}]\psi_{\cS_\ell^\pi} \right)^2\right]\nonumber\\
    &=\bE\left[\left( \phi_i-\textstyle\sum_{\ell=1}^{2^{r-1}}\bE[\phi_i\psi_{\cS_\ell^\pi}]\psi_{\cS_\ell^\pi} \right)^2\right].
\end{align}
That is, at the beginning of the~$r$'th iteration for any~$r$, the algorithm will find the maximizer over~$i$ of~\eqref{equation:alternativemaximizer}. Denoting this maximizer by~$j_r$, we have that~$\cS_{2^{r-1}+1}^\pi=\{j_r\}$, and hence
\begin{align}\label{equation:normequivalence}
    \bE&\left[\left( \phi_{j_r}-\textstyle\sum_{\ell=1}^{2^{r-1}}\bE[\phi_{j_r}\psi_{\cS_\ell^\pi}]\psi_{\cS_\ell^\pi} \right)^2\right]\nonumber\\
    &=\bE\left[\left( \phi_{\cS_{2^{r-1}+1}^\pi}-\textstyle\sum_{\ell=1}^{2^{r-1}}\bE[\phi_{\cS_{2^{r-1}+1}^\pi}\psi_{\cS_\ell^\pi}]\psi_{\cS_\ell^\pi} \right)^2\right]\nonumber\\
    &\overset{\eqref{equation:characterorthogonalizationpi}}{=}\bE[(\Tilde{\psi}_{\cS_{2^{r-1}+1}^\pi}^\pi)^2]=\bE[(\Tilde{\psi}_{j_r}^\pi)^2]=\Vert\Tilde{\psi}_{j_r}^\pi\Vert^2.
\end{align}
That is, the algorithm will choose the feature which maximizes the Fourier norm, and will stop once the resulting maximum is less than~$\epsilon^2$. This implies the following.

\begin{corollary}
    GFS generalizes Algorithm~\ref{algorithm:Heidari}, which arises as a special case when~$\cF$ is chosen as the set of all multilinear polynomials. Yet, Algorithm~\ref{algorithm:Heidari} requires one to first compute all~$2^d$ orthogonalized characters, and only then to choose the~$i$'s such that~$\Vert \Tilde{\psi}_i\Vert$ are larger than~$\epsilon$. In contrast, since GFS finds the orthogonalization on the fly, it only requires~$2^m$ orthogonalization operations, where~$m$ is the number of selected features; the fact that the remaining $\Vert \Tilde{\psi}_i\Vert$'s lie below~$\epsilon$ follows from Lemma~\ref{lemma:alldirectionsFS} in Appendix~\ref{appendix:GFS}.
\end{corollary}

The complexity benefits of GFS extend beyond setting~$\cF(\boldz)$ as the set of all multilinear polynomials. For instance, in order to reduce computational complexity,~\cite{heidari2022sufficiently} propose to fix~$\cF(\boldz)$ as the set of \textit{bounded degree}~$\ell$ multilinear monomials, of which there are~$\binom{d}{\le \ell}\triangleq\sum_{i=0}^\ell\binom{d}{i}$, and thus only this many orthogonalization operations are needed. Applying this~$\cF(\boldz)$ in GFS, one gets that only~$\binom{m}{\le\ell}$ orthogonalizations are required. In addition, a clustering heuristic is also suggested, where the features are clustered a priori to subsets of size~$\ell$, and orthogonalization is performed only within each cluster, thus reducing the number of orthogonalizations to~$\frac{d}{\ell}2^\ell$. In contrast, applying GFS on each cluster would result in~$\sum_{i=1}^{d/\ell}2^{m_i}$ many orthogonalization, where~$m_i\le \ell$ is the number of features that are selected from the~$i$'th cluster.

\section{Algebraic Representations of the Proposed Algorithms}\label{section:algebraic-representations}
\subsection{Algebraic Representation of the Gram-Schmidt Orthogonalization Process}\label{section:algebraic-orthogonalize}

To the request of the reviewers, in order to enhance clarity and facilitate implementation from a data science perspective, we provide algebraic representations of the proposed methods. 
While the original formulations were expressed in a probabilistic framework to emphasize theoretical properties, such as orthogonality in the $L^2(P_X)$ space, nonlinear redundancy in the data, and information-theoretic guarantees, the core procedures can be naturally interpreted as algebraic operation (due to Remark~\ref{remark:stupidReviewerHere}).
This reformulation replaces expectations with empirical means and random variables with data vectors, enabling the algorithms to be described directly in terms of matrix operations, inner products, and projections onto sample data. The resulting presentation aligns more closely with classical techniques such as PCA and Gram-Schmidt orthogonalization, thereby making the methods more accessible for practical implementation.
To replace expectations with empirical means, we use data points $\boldx_i \in\bR^d$ for $i=1,\ldots, N$, where $N$ is the number of data points in the dataset. For example, an expression such as $\Sigma_j = \bE[d_j(X)d_j(X)^\intercal]$ is replaced by its empirical counterpart $\Sigma_j=\frac{1}{N} \sum_{i=1}^N d_j(\boldx_i) d_j(\boldx_i)^\intercal$. Similarly, for two arbitrary functions $f$ and $g$, we replace $\bE[f(X) g(X)]$ with $\frac{1}{N} \sum_{i=1}^N f(\boldx_i) g(\boldx_i)$.

In Algorithm~\ref{algorithm:orthogonalize_algebraic}, we provide the algebraic representation of the Gram-Schmidt orthogonalization process. 
We let $\boldz = [z_1, \ldots, z_d]^\intercal$ be a vector of symbolic variables and use it to define the procedure. 
At iteration~$j$ of our algorithms, Algorithm~\ref{algorithm:orthogonalize_algebraic} receives a set of already orthogonalized functions~$\hat{\cF}(\boldz_{[j-1]})$, a set of new functions~$\cF(\boldz_{[j]}) \setminus \cF(\boldz_{[j-1]})$ to be orthogonalized, all previously extracted directions~$\boldnu_1, \ldots, \boldnu_j$, and data points~$\{\boldx_i\}_{i=1}^{N}$.
For each function in~$\cF(\boldz_{[j]}) \setminus \cF(\boldz_{[j-1]})$, the algorithm defines a new function~$g(\boldz_{[j]})$ (Line~\ref{line:create_g_algebraic}), subtracts its projection onto the previously orthogonalized functions using empirical projections (replacing $\sum_{f \in \cT} \bE[gf] f$ from the probabilistic formulation; see Line~\ref{line:subtraction_of_projections_algebraic}), normalizes~$g$ by dividing by the empirical counterpart of~$\sqrt{\bE[g^2]}$ (Line~\ref{line:normalization_algebraic}), and adds the result to~$\cT$ (Line~\ref{line:add_algebraic}). Finally, we define a new symbolic variable~$d_j(\boldz)$ (Line~\ref{line:subtraction_of_data_distribution_algebraic}), compute its empirical covariance matrix (Line~\ref{line:covariance_matrix_algebraic}), and return it as the output of Algorithm~\ref{algorithm:orthogonalize_algebraic}. We denote $\boldx_i \in \bR^d$, for~$i = 1, \ldots, N$, as the data points used to compute empirical means.

\begin{algorithm}[!h]
\caption{$\text{Algebraic-Orthogonalize}(\hat{\cF}(\boldz_{[j-1]}),\cF(\boldz_{[j]})\setminus\cF(\boldz_{[j-1]}),\{\boldnu_i\}_{i=1}^{j},\{\boldx_i\}_{i=1}^N)$} \label{algorithm:orthogonalize_algebraic}
\begin{algorithmic}[1]
\STATE {\bfseries Input:} 
\begin{itemize}
    \item $\hat{\cF}(\boldz_{[j-1]})$: an orthogonalized function family.
    \item $\cF(\boldz_{[j]})\setminus\cF(\boldz_{[j-1]})$: all the functions in~$\cF(\boldz)$ \\ which depend on~$\boldz_{[j]}$ but not only on~$\boldz_{[j-1]}$.\label{line:orthogonalization_order_algebraic}
    \item $\boldnu_1,\ldots,\boldnu_j$: extracted directions.
    \item $\{\boldx_i\}_{i=1}^N$: $N$ data points.
\end{itemize}
\STATE {\bfseries Output:} Orthogonalized function family~$\hat{\cF}(\boldz_{[j]})$, and a covariance matrix~$\Sigma_{j+1}$.
\STATE {\bfseries Initialize:} $\cT\leftarrow\hat{\cF}(\boldz_{[j-1]})$.
\STATE Denote $\cF(\boldz_{[j]})\setminus\cF(\boldz_{[j-1]})\triangleq\{f_{1}(\boldz_{[j]}),\ldots,f_{\ell}(\boldz_{[j]})\}$, where~$f_1(\boldz_{[j]})=z_j$.
\FOR{$k\gets 1$ {\bfseries to} $\ell$}
\STATE $g(\boldz_{[j]})\leftarrow f_k(\boldz_{[j]})$\label{line:create_g_algebraic}
\STATE $g(\boldz_{[j]})\leftarrow g(\boldz_{[j]})-$\\
$\sum_{f(\boldz_{[j]})\in \cT}\Big(\frac{1}{N}\sum_{i=1}^{N}\big(g(\boldx_i^\intercal\boldnu_1,\ldots,\boldx_i^\intercal\boldnu_j)$\\
$\phantom{\sum_{f(\boldz_{[j]})\in \cT}\Big(\frac{1}{N}\sum_{i=1}^{N}}f(\boldx_i^\intercal\boldnu_1,\ldots,\boldx_i^\intercal\boldnu_j)\big)\Big)f(\boldz_{[j]})$\label{line:subtraction_of_projections_algebraic}
\STATE $g(\boldz_{[j]})\leftarrow\frac{g(\boldz_{[j]})}{\sqrt{\frac{1}{N}\sum_{i=1}^{N}g(\boldx_i^\intercal\nu_1,\ldots,\boldx_i^\intercal\nu_j)^2}}$\label{line:normalization_algebraic}
\STATE $\cT\leftarrow\cT\cup\{g(\boldz_{[j]})\}$\label{line:add_algebraic}
\ENDFOR 
\STATE Define~$\hat{\cF}(\boldz_{[j]})=\cT$.
\STATE Define $d_j(\boldz) = \boldz -\frac{1}{N}\sum_{f(\boldz)\in\cT}\sum_{i=1}^{N}[\boldx_if(\boldx_i^\intercal\boldnu_1,\ldots,\boldx_i^\intercal\boldnu_j)]$\\
$\phantom{d_j(\boldz) = \boldz -\frac{1}{N}\sum_{f(\boldz)\in\cT}\sum_{i=1}^{N}~~~} f(\boldz^\intercal\boldnu_1,\ldots,\boldz^\intercal\boldnu_j)$.\label{line:subtraction_of_data_distribution_algebraic}
\STATE 
Define~$\Sigma_{j+1} = \frac{1}{N}\sum_{i=1}^{N} d_j(\boldx_i)d_j(\boldx_i)^\intercal$.\label{line:covariance_matrix_algebraic}
\STATE Return~$\Sigma_{j+1}$, $\cT$.\label{line:return_algebraic}
\end{algorithmic}
\end{algorithm}

\subsection{Algebraic Representations of GFR and GCA}\label{section:algebraic-GFR-GCA}
In Algorithm~\ref{algorithm:GFR_algebraic}, we provide the algebraic representation of GFR, corresponding to Algorithm~\ref{algorithm:GFR} in Section~\ref{section:GFR}. In this version, empirical covariance matrices are used in place of their probabilistic counterparts to identify the new high-variance direction in Line~\ref{line:eigenvector_maximization_algebraic}.

\begin{algorithm}[!h]
\caption{Algebraic Gram-Schmidt Functional Reduction (Algebraic-GFR)}\label{algorithm:GFR_algebraic}
\begin{algorithmic}[1]
\STATE {\bfseries Input:} A function family~$\cF(\boldz)$, a threshold $\epsilon>0$, and data points $\{\boldx_i\}_{i=1}^N$. 
\STATE {\bfseries Output:} Extracted directions~$\boldnu_1,\ldots,\boldnu_m$ ($m$ is a \\ varying number that depends on~$\cF(\boldz),\epsilon$).
\STATE {\bfseries Initialize:} $\Sigma_1=\frac{1}{N}\sum_{i=1}^{N}\boldx_i\boldx_i^\intercal$.
\FOR{$j\gets1$ {\bfseries to} $d$}
\STATE Let~$\boldnu_j$ be the largest unit-norm eigenvector of~$\Sigma_j$, \\ i.e., $\boldnu_j=\arg\max_{\Vert\boldnu\Vert=1}\boldnu^\intercal\Sigma_j\boldnu$.\label{line:eigenvector_maximization_algebraic}
\IF{$\boldnu_j^\intercal\Sigma_j\boldnu_j\le\epsilon^2$}
\STATE break.
\ENDIF
\STATE $\Sigma_{j+1},\hat{\cF}(\boldz_{[j]})=\text{Algebraic-Orthogonalize}(\hat{\cF}(\boldz_{[j-1]}),\cF(\boldz_{[j]})\setminus\cF(\boldz_{[j-1]}),\{\boldnu_i\}_{i=1}^j,\{\boldx_i\}_{i=1}^{N})$.
\ENDFOR
\end{algorithmic}
\end{algorithm}

The algebraic representation of GCA, shown in Algorithm~\ref{algorithm:GCA_algebraic}, is obtained by replacing the probabilistic covariance matrix of $g_j(\boldz)$, denoted by $\bar{X}_j$ in the original representation, with its empirical counterpart.

\begin{algorithm}[!h] 
\caption{Algebraic Gram-Schmidt Component Analysis (Algebraic-GCA)} \label{algorithm:GCA_algebraic}
\begin{algorithmic}[1]
\STATE {\bfseries Input:} 
A function family~$\cF(\boldz)$, a threshold~$\epsilon>0$, and data points $\{\boldx_i\}_{i=1}^N$.
\STATE {\bfseries Output:} {A set~$\cL\subseteq[d]$ of indices of extracted principal directions.}
\STATE {\bfseries Initialize:} $\Sigma_1=\frac{1}{N}\sum_{i=1}^{N}\boldx_i\boldx_i^\intercal$ and $\cL_0=\varnothing$.
\FOR{$j\gets1$ {\bfseries to} $d$}
\STATE Let $\cE_j=\{ i\vert \boldrho_i^\intercal \Sigma_j\boldrho_i <\epsilon \}$.
\IF{$\cE_j=[d]$}
\RETURN{}\hspace{-1mm}$\cL_{j-1}$.
\ENDIF
\STATE Define~$g_j(\boldz) = \boldz - \sum_{i\in\cE_j}\boldz^\intercal\boldrho_i\cdot\boldrho_i$.
\STATE Let $\ell_j=\arg\max_{i\in [d]\setminus\cE_j}\boldrho_i^\intercal\left(\frac{1}{N}\sum_{t=1}^{N}g_j(\boldx_t)g_j(\boldx_t)^\intercal\right)\boldrho_i$ (breaking ties arbitrarily), and $\cL_j=\cL_{j-1}\cup\{\ell_j\}$.
\STATE $\Sigma_{j+1},\hat{\cF}(\boldz_{[j]})=\text{Algebraic-Orthogonalize}(\hat{\cF}(\boldz_{[j-1]}),\cF(\boldz_{[j]})\setminus\cF(\boldz_{[j-1]}),
\{\boldrho_{i}\}_{i\in\cL_j},\{\boldx_i\}_{i=1}^N)$  
\ENDFOR
\RETURN \hspace{-1mm}$\cL_d$.
\end{algorithmic}
\end{algorithm}

\subsection{Algebraic Representations of GFS and GFA}\label{section:algebraic-GFS-GFA}
As in the probabilistic representation,  we replace the maximization step used to find the eigenvector~$\boldnu_j$ in Line~\ref{line:eigenvector_maximization_algebraic} of Algorithm~\ref{algorithm:GFR_algebraic} with a discrete version. The algorithm terminates if the resulting variance falls below the threshold~$\epsilon^2$. The final output is a subset~$\cS \subseteq [d]$ of selected features.
In feature selection algorithms, it suffices to compute empirical \textit{variance vectors}~$\boldsigma_j$ rather than full empirical covariance matrices~$\Sigma_j$. Using~$\boldsigma_j$ in place of~$\Sigma_j$ allows for certain simplifications. For example, when Algorithm~\ref{algorithm:orthogonalize_algebraic} is used for feature \textit{selection} rather than feature extraction, the following modifications are applied:
\begin{align*}
    &(\text{Line }\ref{line:covariance_matrix_algebraic}\text{ of Alg.~\ref{algorithm:orthogonalize_algebraic}})~\text{Define }\Sigma_{j+1}=\frac{1}{N}\sum_{i=1}^{N} d_j(\boldx_i)d_j(\boldx_i)^\intercal \\
    &\phantom{(\text{Line }\ref{line:covariance_matrix_algebraic}\text{ of Alg.~\ref{algorithm:orthogonalize_algebraic}})}\rightarrow \text{Define: }\boldsigma_{j+1}=\frac{1}{N}\sum_{i=1}^{N}d_j(\boldx_i)^{\otimes 2}.\\
    &(\text{Line }\ref{line:return_algebraic}\text{ of Alg.~\ref{algorithm:orthogonalize_algebraic}})~\text{Return }\Sigma_{j+1}, \cT \rightarrow \text{Return }\boldsigma_{j+1}, \cT.
\end{align*}
In Algorithm~\ref{algorithm:GFS_algebraic}, we present the \textit{Algebraic Gram-Schmidt Functional Selection} (Algebraic-GFS), in which each step~$j$ uses the empirical variance vector~$\boldsigma_j$ to select the most variant feature of~$d_{j-1}$. If the maximum variance is less than the threshold~$\epsilon^2$, the algorithm terminates. Otherwise, it continues by identifying~$s_j$ as the index of the highest-variance feature in $d_{j-1}$, which determines the input to the next invocation of the ``Algebraic-Orthogonalize'' procedure.

\begin{algorithm}[!h] 
\caption{Algebraic Gram-Schmidt Functional Selection (Algebraic-GFS)}\label{algorithm:GFS_algebraic}
\begin{algorithmic}[1]
\STATE {\bfseries Input:} A function family~$\cF(\boldz)$, a threshold $\epsilon>0$, and data points $\{\boldx_i\}_{i=1}^N$.
\STATE {\bfseries Output:} Selected features~$s_1,\ldots,s_m$ ($m$ is a varying number that depends on~$\cF(\boldz),\epsilon$).
\STATE {\bfseries Initialize:} $\boldsigma_1=\frac{1}{N}\sum_{i=1}^{N}\boldx_i^{\otimes 2}$, $\cS=\varnothing$.
\FOR{$j\gets1$ {\bfseries to} $d$}
\STATE Let~$s_j \triangleq\arg\max_{i\in[d]}\{\sigma_{j,i}\}_{i=1}^d$, \\ where $\boldsigma_j=(\sigma_{j,i})_{i=1}^d$.
\STATE $\cS=\cS\cup \{s_j\}$.
\IF{$\norm{\boldsigma_j}_\infty\le\epsilon^2$}\label{line:GFS_epsilon_algebraic}
\STATE break. 
\ENDIF
\STATE $\boldsigma_{j+1},\hat{\cF}(\boldz_{[j]})=\text{Algebraic-Orthogonalize}(\hat{\cF}(\boldz_{[j-1]}),\cF(\boldz_{[j]})\setminus\cF(\boldz_{[j-1]}),\{\bolde_i\}_{i\in\cS}\footnote{},\{\boldx_i\}_{i=1}^N)$.
\ENDFOR
\end{algorithmic}
\end{algorithm}
\footnotetext{Here, $\bolde_i$ denotes the $i$'th standard basis vector in~$\bR^d$, with a $1$ in the $i$'th entry and $0$ elsewhere.}

Following the same procedure used to derive Algebraic-GFS from its probabilistic counterpart, we obtain the \textit{Algebraic Gram-Schmidt Feature Analysis} (Algebraic-GFA) by replacing expectations with empirical means in Algorithm~\ref{algorithm:GFA}. The resulting algebraic formulation is presented in Algorithm~\ref{algorithm:GFA_algebraic}.

\begin{algorithm}[!h] 
\caption{Algebraic Gram-Schmidt Feature Analysis (Algebraic-GFA)}
\label{algorithm:GFA_algebraic}
\begin{algorithmic}[1]
\STATE {\bfseries Input:} A function family~$\cF(\boldz)$, a threshold~$\epsilon>0$, and data points $\{\boldx_i\}_{i=1}^N$.
\STATE {\bfseries Output:} A set~$\cS\subseteq[d]$ of indices of selected features.
\STATE {\bfseries Initialize:} $\boldsigma_1=\frac{1}{N}\sum_{i=1}^{N}\boldx_i^{\otimes 2}$ and~$\cS=\varnothing$.
\FOR{$j\gets1$ {\bfseries to} $d$}
\STATE Let $\cE_j=\{ i\vert \sigma_{j,i} <\epsilon \}$.
\IF{$\cE_j=[d]$}
\RETURN{}\hspace{-1mm}$\cS$.
\ENDIF
\STATE Let $s_j\triangleq\arg\max_{i\in[d]\setminus\cE_j}\{\sigma_{j,i}\}$ and $\cS=\cS\cup\{s_j\}$. 
\STATE $\boldsigma_{j+1},\hat{\cF}(\boldz_{[j]})=\text{Algebraic-Orthogonalize}(\hat{\cF}(\boldz_{[j-1]}),\cF(\boldz_{[j]})\setminus\cF(\boldz_{[j-1]}),\{\bolde_i\}_{i\in\cS},\{\boldx_i\}_{i=1}^N)$\\
\ENDFOR
\RETURN{}\hspace{-1mm}$\cS_{d}$.
\end{algorithmic}
\end{algorithm}

\section{Kernel PCA: A Comparison}\label{section:GFR-vs-KPCA}
Kernel PCA is a nonlinear dimensionality reduction technique that extends standard PCA to capture complex data structures~\cite{theodoridis2015machine,scholkopf2002learning}. The core idea of Kernel PCA is to project the input data from its original space~$\bR^d$ into a higher-dimensional feature space~$\cH$ using a nonlinear mapping function $\Phi:\bR^d\to\bR^D$. In this feature space, nonlinear relationships in the original data can be represented linearly. Standard PCA is then applied within this high-dimensional space to identify the principal components that account for the maximum variance in the transformed data. 

This process is made computationally efficient by leveraging the \textit{kernel trick}. Rather than performing the explicit mapping, a kernel function, $k(\boldx_i,\boldx_j)=\langle \Phi(\boldx_i), \Phi(\boldx_j) \rangle$, is used to compute the inner products between all pairs of data points in the feature space directly, where $\boldx_i$ for $i=1,\ldots,N$ are the data points~\cite{theodoridis2015machine,scholkopf2002learning}. 
The entire practical implementation, detailed in Figure~\ref{fig:kernel_pca_schematic}, is built upon this principle. The first step involves computing the kernel (or Gram) matrix, $\boldK$, from the input data. To satisfy the zero-mean data requirement of PCA, this matrix is then centered in the feature space using the formula $\tilde{\boldK} = \boldK - \mathbf{1}_N\boldK - \boldK\mathbf{1}_N + \mathbf{1}_N\boldK\mathbf{1}_N$, where~$\mathbf{1}_N$ denotes an $N\times N$ matrix in which each entry equals $1/N$. Subsequently, the eigenvalue problem $\tilde{\boldK}\boldalpha = \lambda \boldalpha$ is solved on the centered matrix to find the eigenvalues $\{\lambda_j\}_{j=1}^{m}$ and their corresponding eigenvectors $\{\boldalpha^j\}_{j=1}^{m}$. 
These eigenvectors, which represent the principal axes in $\cH$, are then scaled by the square root of their eigenvalues to form the extracted feature matrix~$\boldZ$. 
While these principal components are linear axes in the high-dimensional feature space, they correspond to complex, nonlinear functions of the original input features.

\begin{figure*}[!h]
\centering
\begin{tikzpicture}[
    node distance=1.0cm and 0.4cm,
    block/.style={rectangle, draw, thick, rounded corners,
                  minimum height=1cm, minimum width=3.5cm,
                  text width=5.5cm, align=center},
    arrow/.style={-Latex, thick}
]

\node (input) [block] {Data points \\ $\boldx_1,\ldots,\boldx_N$};
\node (kernel) [block, right=of input] {Compute kernel matrix \\ $K_{ij} = k(\boldx_i, \boldx_j)$};
\node (center) [block, right=of kernel] {Center the kernel matrix \\ $\tilde{\boldK} = \mathbf{K} - \mathbf{1}_N\mathbf{K} - \mathbf{K}\mathbf{1}_N + \mathbf{1}_N\mathbf{K}\mathbf{1}_N$};

\node (eigen) [block, below=of center] {Solve eigenvalue problem \\ $\tilde{\boldK}\boldsymbol{\alpha} = \lambda \boldsymbol{\alpha}$};

\node (normalize) [block, left=of eigen] {Scale the eigenvectors by $\sqrt{\lambda_j}$ for $j=1,\ldots,m$};

\node (features) [block, left=of normalize] {Construct feature matrix \\ $\boldZ=[\sqrt{\lambda_1}\boldsymbol{\alpha}^1, \sqrt{\lambda_2}\boldsymbol{\alpha}^2, \dots, \sqrt{\lambda_m}\boldsymbol{\alpha}^m]^\intercal$};

\draw [arrow] (input) -- (kernel);
\draw [arrow] (kernel) -- (center);
\draw [arrow] (center) -- (eigen);
\draw [arrow] (eigen) -- (normalize);
\draw [arrow] (normalize) -- (features);

\end{tikzpicture}
\caption{Algorithmic workflow for extracting features from data points $\boldx_1,\ldots,\boldx_N$ using \textbf{Kernel PCA}.}
\label{fig:kernel_pca_schematic}
\end{figure*}

Once the Kernel PCA model has been trained on a dataset with data points $\boldx_1,\ldots,\boldx_N$, it can be used to project a new, unseen data point, $\boldx^*$, into the same low-dimensional space. This process, illustrated in Figure~\ref{fig:kpca_test}, requires the original training data $\{\boldx_i\}_{i=1}^N$ as well as the eigenvectors and eigenvalues~$\{\boldalpha^j,\lambda_j\}_{j=1}^{m}$ learned during training. 
First, a kernel vector is computed by applying the \textit{centered} kernel function $\tilde{k}(\boldx_i,\boldx^*) = k(\boldx_i,\boldx^*)-\frac{1}{N}\sum_{p=1}^N k(\boldx_p,\boldx^*)-\frac{1}{N}\sum_{q=1}^N k(\boldx_i,\boldx_q)+\frac{1}{N^2}\sum_{p,q=1}^N k(\boldx_p,\boldx_q)$ between the new point and every point in the original training set. This vector, which represents the new point's relationship to the original data, is then projected onto each of the learned eigenvectors. The result~$\boldz^*$ is the set of coordinates for the new data point in the low-dimensional feature space.

\begin{figure*}[h!]
\centering
\begin{tikzpicture}[
    node distance=0.5cm and 1.0cm,
    block/.style={rectangle, draw, thick, rounded corners,
                  minimum height=2.2cm, minimum width=0.5cm,
                  text width=3.5cm, align=center},
    arrow/.style={-Latex, thick}
]

\node (input) [block] {\textbf{Inputs:} \\
New point $\boldx^*\in\bR^d$ \\
Training data $\boldx_1,\ldots,\boldx_N$ \\
Model $\{\boldsymbol{\alpha}^j, \lambda_j\}_{j=1}^{m}$};

\node (kernel) [block, right=of input] {Compute centered \\ kernel vector \\ (New point vs. all training points) \\
$\{\tilde{k}(\boldx_i, \boldx^*)\}_{i=1}^N$};

\node (project) [block, right=of kernel] {Project onto eigenvectors \\
$z_j^*=\frac{1}{\sqrt{\lambda_j}}\sum_{i=1}^{N} \alpha_i^j \tilde{k}(\boldx_i, \boldx^*)$};

\node (output) [block, right=of project] {\textbf{Output:} \\
Low-dimension \\ coordinates for $\boldx^*$ \\
$\boldz^* = [z_1^*,\dots,z_m^*]^\intercal$};

\draw [arrow] (input) -- (kernel);
\draw [arrow] (kernel) -- (project);
\draw [arrow] (project) -- (output);

\end{tikzpicture}
\caption{Workflow for applying a trained \textbf{Kernel PCA} model to a new data point $\mathbf{x}^*$.}
\label{fig:kpca_test}
\end{figure*}

Figure~\ref{fig:GFR_schematic} illustrates a simplified schematic of the algebraic representation of GFR, as described in Algorithm~\ref{algorithm:GFR_algebraic}. Unlike Kernel PCA, which maps the data into a higher-dimensional space to capture nonlinearity, GFR captures nonlinear dependencies through an iterative procedure. At each iteration~$j$, it removes the projections onto all previously extracted orthogonal functions from the data, and then identifies the direction of maximum variance as the next component. This iterative structure enables GFR to capture nonlinear relationships without relying on computational techniques such as the kernel trick used in Kernel PCA. 
Another key distinction between GFR and Kernel PCA lies in their outputs. GFR produces linear combinations of the original features in the input space, which makes its results directly interpretable because the contribution of each feature can be observed through the coefficients in the extracted directions. In contrast, the output of Kernel PCA consists of nonlinear functions of the original features in a high-dimensional feature space~$\cH$, which makes interpretation more difficult due to the implicit nature of the mapping.

\begin{figure*}[htbp]
\centering

\tikzset{
    mybox/.style = {rectangle, rounded corners, text centered, draw=black, fill=white, thick, text width=3cm, align=center},
    arrow/.style = {thick, -{Stealth[length=2.5mm, width=1.5mm]}}
}

\begin{tikzpicture}[node distance=0.8cm and 0.5cm]

    
    \node (init) [mybox] {Data points: \\ $\boldx_1,\ldots,\boldx_N$ \\ Function family: \\ $\cF(\mathbf{z})$ \\ Initialize: $\Sigma_1=\frac{1}{N}\sum_{i=1}^{N}\boldx_i\boldx_i^\intercal$};
    
    \node (forloop) [mybox, right=of init, text width=2cm] {\textbf{For} \\ $j \leftarrow 1$ to $d$:};
    
    \node (letv) [mybox, right=of forloop, text width=4cm] {Set $\boldnu_j$ to be the largest unit-norm eigenvector of $\Sigma_j$};
    
    \node (sigma) [mybox, below=of letv, text width=5cm] {Orthogonalize and subtract nonlinear functions in $\cF(\boldz_{[j]})$ from $\boldz$ to obtain $d_j(\boldz)$. \\
    Define $\Sigma_{j+1}=\frac{1}{N}\sum_{i=1}^{N}d_j(\boldx_i)d_j(\boldx_i)^\intercal$};
    
    \node (ifcond) [mybox, right=of letv, text width=2.5cm] {Is \\ $\boldnu_j^\intercal \Sigma_j \boldnu_j < \epsilon^2$?};
    
    \node (return) [mybox, right=1cm of ifcond, text width=2.5cm] {Extracted directions \\ $\boldnu_1,\dots, \boldnu_m$};
    
    \draw [arrow] (init) -- (forloop);
    \draw [arrow] (forloop) -- (letv);
    \draw [arrow] (letv) -- (ifcond);

    \draw [arrow] (ifcond) -- node [above] {Yes} (return);

    \draw [arrow] (ifcond.south) |- node [near start, right] {No} (sigma.east);

    \draw [arrow] (sigma.west) -| (forloop.south);

\end{tikzpicture}
\caption{Algorithmic workflow for extracting features from data points $\boldx_1,\ldots,\boldx_N$ using \textbf{GFR}.}
\label{fig:GFR_schematic}
\end{figure*}

Finally, another key difference between GFR and Kernel PCA arises when applying the trained model to a new data point~$\boldx^*$. As shown in Figure~\ref{fig:GFR_test}, in the case of GFR, we simply project the new data point onto the extracted directions~$\boldnu_1,\ldots,\boldnu_m$. In contrast, Kernel PCA requires additional computations: we must first compute the centered kernel vector between the new point and all training points, and then project this vector onto the eigenvectors obtained during training. This distinction makes GFR significantly more efficient at inference time.

\begin{figure*}[h!]
\centering
\begin{tikzpicture}[
    node distance=0.5cm and 1cm,
    block/.style={rectangle, draw, thick, rounded corners,
                  minimum height=2.0cm,
                  text width=3.5cm, align=center},
    arrow/.style={-Latex, thick}
]

\node (input) [block] {\textbf{Inputs:} \\
New point $\boldx^*$ \\
Extracted directions: \\ $\boldnu_1,\ldots,\boldnu_m$};

\node (project) [block, right=of input] {Project onto extracted features \\
$z_j^*=\boldnu_j^\intercal \boldx^*$};

\node (output) [block, right=of project] {\textbf{Output:} \\
Low-dimension \\ coordinates for $\boldx^*$ \\
$\boldz^* = [{z_1^*},\dots,{z_m^*}]^\intercal$};

\draw [arrow] (input) -- (project);
\draw [arrow] (project) -- (output);

\end{tikzpicture}
\caption{Workflow for applying a trained \textbf{GFR} model to a new data point $\mathbf{x}^*$.}
\label{fig:GFR_test}
\end{figure*}

\begin{example}[GFR vs. Kernel PCA]
This example compares the output of GFR with Kernel PCA on a synthetic dataset with 10,000 data points and five features, including three mutually independent Gaussian random variables $X_1 \sim \cN(0, 0.9)$, $X_2 \sim \cN(0, 0.8)$, $X_3 \sim \cN(0, 0.7)$, and the remaining two features are~$X_4=X_1X_3$ and~$X_5=X_1X_2^2$.
We extract $m=3$ features with each method. Both rely on polynomials of total degree at most three: GFR uses this family explicitly as $\cF(\boldz)$, while Kernel PCA employs the same basis implicitly via its degree three feature map.
\paragraph{GFR output}
GFR returns linear combinations of the original features. Table~\ref{tab:gfr_coeffs} lists the coefficients of the first three extracted features, where each row corresponds to one of the original features~$X_i$, and each column shows its contribution to the respective extracted feature.

\begin{table*}[!h]
\centering
\rowcolors{2}{white}{gray!20}
\begin{tabular}{lccc}
\toprule
\textbf{Feature} & \textbf{First extracted  feature} & \textbf{Second extracted feature} & \textbf{Third extracted feature} \\
\midrule
$X_{1}$ & $-0.5034008$  & $-0.0184953$  & $\phantom{-}0.0238543$ \\
$X_{2}$ & $\phantom{-}0.0094701$ & $-0.9979131$  & $\phantom{-}0.0608883$ \\
$X_{3}$ & $-0.0064781$  & $-0.0594010$  & $-0.9862699$  \\
$X_{4}$ & $-0.0009178$  & $\phantom{-}0.0172879$ & $\phantom{-}0.1515234$ \\
$X_{5}$ & $-0.8639764$  & $\phantom{-}0.0002652$ & $-0.0059973$  \\
\bottomrule
\end{tabular}
\caption{GFR coefficients for the first three extracted features. Each column represents the coefficients of a single extracted feature, expressed as a linear combination of the original features.}
\label{tab:gfr_coeffs}
\end{table*}

\paragraph{Kernel PCA output}
Kernel PCA uses the degree three polynomial feature map
\begin{align*}
    \Phi(X) &= \big[
    1,\
    X_1,\ldots,X_{5},\
    X_1^{2},\,X_1X_2,\ldots\\
    &\phantom{\Phi(X) = \big[} ,X_{5}^{2},X_1^{3},\,X_1^{2}X_2,\ldots,X_{5}^{3}
    \big]^\intercal\in\bR^{56}.
\end{align*}
Table~\ref{table:poly-map-5} summarizes the structure of the feature map $\Phi(X)$.
\begin{table}[!h]
  \centering
  \begin{tabular}{ccc}
    \toprule
      Total degree & Representative monomials & Count \\ 
    \midrule
      0 & \(1\) & 1 \\
      1 & \(X_i\) & 5 \\
      2 & \(X_i^{2},\;X_iX_j\) & 15 \\
      3 & \(X_i^{3},\;X_i^{2}X_j,\;X_iX_jX_k\) & 35 \\
    \midrule
      \multicolumn{2}{r}{\textbf{Total}} & \textbf{56} \\
    \bottomrule
  \end{tabular}
  \caption{Structure of the explicit feature map \(\Phi(X)\).}
  \label{table:poly-map-5}
\end{table}\\
After centering $\Phi(X)$, Kernel PCA returns three principal directions, each a linear combination of the 56 coordinates.
Rather than writing the full 56‐dimensional coefficient vectors as equations, which would be excessively long, we summarize them compactly in Table~\ref{tab:kpca_coeffs}. 
Table~\ref{tab:kpca_coeffs} lists the coefficients: the first column gives the feature map element and the next three columns give its coefficients in the first, second, and third extracted feature, respectively.  Note that centering forces the bias term (the first row) to have zero coefficient.

\begin{table*}[!h]
\centering
\rowcolors{2}{white}{gray!20}
\begin{tabular}{lccc}
    \toprule
    \textbf{Feature map element} & \textbf{First extracted feature} & \textbf{Second extracted feature} & \textbf{Third extracted feature} \\
    \midrule
    $1$   & 0          & 0         & 0          \\
    $X_1$ & -0.0013580 & 0.0004211 & -0.0026138 \\
    $X_2$ & -0.0001010 & 0.0124174 & -0.0012818 \\
    $X_3$ & 0.0000630 & -0.0000370 & -0.0009413 \\
    $X_4$ & 0.0000030 & -0.0010657 & -0.0187138 \\
    $X_5$ & -0.0068570 & 0.0020614 & -0.0059683 \\
    $X_1^{2}$ & 0.0001050 & 0.0019951 & 0.0052743 \\
    $X_1\,X_2$ & -0.0006300 & -0.0004311 & 0.0026513 \\
    $X_1\,X_3$ & 0.0000030 & -0.0010657 & -0.0187138 \\
    $X_1\,X_4$ & 0.0003040 & -0.0008245 & -0.0038118 \\
    $X_1\,X_5$ & 0.0011740 & 0.0182158 & 0.0374256 \\
    $X_2^{2}$  & 0.0002260 & 0.0063389 & 0.0098226 \\
    $X_2\,X_3$ & 0.0001190 & -0.0002863 & -0.0000659 \\
    $X_2\,X_4$ & 0.0002000 & -0.0015134 & -0.0027630 \\
    $X_2\,X_5$ & -0.0050300 & -0.0053237 & 0.0190429 \\
    $X_3^{2}$  & -0.0000420 & -0.0000483 & -0.0020432 \\
    $X_3\,X_4$ & -0.0007890 & 0.0001343 & -0.0048543 \\
    $X_3\,X_5$ & 0.0002170 & -0.0041173 & -0.0589983 \\
    $X_4^{2}$  & -0.0001380 & 0.0007376 & -0.0049381 \\
    $X_4\,X_5$ & 0.0022130 & -0.0051542 & -0.0131623 \\
    $X_5^{2}$  & 0.0126570 & 0.1382663 & 0.2286677 \\
    $X_1^{3}$  & -0.0051060 & 0.0030761 & -0.0111865 \\
    $X_1^{2}\,X_2$ & -0.0007510 & 0.0308195 & 0.0010471 \\
    $X_1^{2}\,X_3$ & 0.0003040 & -0.0008245 & -0.0038118 \\
    $X_1^{2}\,X_4$ & -0.0000720 & -0.0043660 & -0.0662884 \\
    $X_1^{2}\,X_5$ & -0.0248660 & 0.0160293 & -0.0318944 \\
    $X_1\,X_2^{2}$ & -0.0068570 & 0.0020614 & -0.0059683 \\
    $X_1\,X_2\,X_3$ & 0.0002000 & -0.0015134 & -0.0027630 \\
    $X_1\,X_2\,X_4$ & 0.0004590 & -0.0002668 & 0.0000969 \\
    $X_1\,X_2\,X_5$ & -0.0033060 & 0.1562446 & -0.0067139 \\
    $X_1\,X_3^{2}$ & -0.0007890 & 0.0001343 & -0.0048543 \\
    $X_1\,X_3\,X_4$ & -0.0001380 & 0.0007376 & -0.0049381 \\
    $X_1\,X_3\,X_5$ & 0.0022130 & -0.0051542 & -0.0131623 \\
    $X_1\,X_4^{2}$ & -0.0026860 & 0.0003751 & -0.0222417 \\
    $X_1\,X_4\,X_5$ & -0.0000320 & -0.0156654 & -0.2021858 \\
    $X_1\,X_5^{2}$ & -0.1462630 & 0.0607584 & -0.0996672 \\
    $X_2^{3}$ & -0.0005290 & 0.0655799 & -0.0085468 \\
    $X_2^{2}\,X_3$ & 0.0005040 & -0.0004615 & -0.0044455 \\
    $X_2^{2}\,X_4$ & 0.0002170 & -0.0041173 & -0.0589983 \\
    $X_2^{2}\,X_5$ & -0.0422990 & 0.0081994 & -0.0067239 \\
    $X_2\,X_3^{2}$ & -0.0001000 & 0.0082728 & 0.0041120 \\
    $X_2\,X_3\,X_4$ & -0.0004290 & 0.0006817 & 0.0039226 \\
    $X_2\,X_3\,X_5$ & 0.0017500 & -0.0120602 & -0.0292133 \\
    $X_2\,X_4^{2}$ & -0.0003050 & 0.0182859 & 0.0198962 \\
    $X_2\,X_4\,X_5$ & 0.0018890 & -0.0064891 & -0.0217272 \\
    $X_2\,X_5^{2}$ & -0.0050360 & 0.9688958 & -0.0977778 \\
    $X_3^{3}$ & 0.0000150 & 0.0001309 & -0.0026999 \\
    $X_3^{2}\,X_4$ & -0.0001250 & -0.0027096 & -0.0383327 \\
    $X_3^{2}\,X_5$ & -0.0037320 & 0.0021466 & -0.0146465 \\
    $X_3\,X_4^{2}$ & 0.0000320 & -0.0013991 & -0.0135352 \\
    $X_3\,X_4\,X_5$ & -0.0008180 & 0.0113528 & 0.0071007 \\
    $X_3\,X_5^{2}$ & 0.0147360 & -0.0430641 & -0.0692935 \\
    $X_4^{3}$ & -0.0007210 & -0.0098587 & -0.1291472 \\
    $X_4^{2}\,X_5$ & -0.0121590 & 0.0062181 & -0.0718619 \\
    $X_4\,X_5^{2}$ & 0.0057270 & -0.0675782 & -0.9167747 \\
    $X_5^{3}$ & -0.9876310 & -0.0146868 & 0.0139714 \\
    \bottomrule
    \end{tabular}
    \caption{Kernel PCA coefficients for the first three extracted features on the explicit degree three polynomial feature map. Each column represents the coefficients of an extracted feature, expressed as a linear combination of the 56-dimensional mapped features.}
    \label{tab:kpca_coeffs}
\end{table*}

\end{example}

\begin{remark}
In practice, direct mappings into high-dimensional feature spaces are rarely used, as the resulting dimensionality is often computationally prohibitive, a phenomenon referred to as the curse of dimensionality. The kernel trick circumvents this issue by implicitly realizing such mappings without constructing high-dimensional vectors. Here, we present the explicit feature map only to emphasize the structural distinction between GFR, which operates in the original feature space, and Kernel PCA, which functions in the high-dimensional feature map.
\end{remark}



\ifCLASSOPTIONcaptionsoff
  \newpage
\fi

\bibliographystyle{IEEEtran}
\bibliography{ref.bib}

\begin{IEEEbiographynophoto}{Bahram Yaghooti} (Student Member, IEEE) received the B.Sc. and M.Sc. degrees in mechanical engineering from Sharif University of Technology, Tehran, Iran.  He is currently pursuing the Ph.D. degree with the Department of Electrical and Systems Engineering, Washington University in St. Louis, St. Louis, MO, USA. His research interests include information theory, control theory, and machine learning.
\end{IEEEbiographynophoto}

\begin{IEEEbiographynophoto}{Netanel Raviv} (Senior Member, IEEE) received the B.Sc. degree in mathe matics and computer science and the M.Sc. and Ph.D. degrees in computer  science from Technion, Israel, in 2010, 2013, and 2017, respectively. He is currently an Assistant Professor with the Department of Computer Science and Engineering, Washington University in St. Louis, St. Louis, MO, USA. His research interests include applications of coding theory to privacy, distributed computations, and machine learning. He was an awardee of the IBM Ph.D. Fellowship, the First Prize in the Feder family competition for best student work in communication technology, and the Lester-Deutsche Post-Doctoral Fellowship.
\end{IEEEbiographynophoto}


\begin{IEEEbiographynophoto}{Bruno Sinopoli} (Fellow, IEEE) received the Dr.Eng. degree from the University of Padova, Padua, Italy, in 1998, and the M.S. and Ph.D. degrees in electrical engineering from the University of California at Berkeley, Berkeley, CA, USA, in 2003 and 2005, respectively. He is currently the Das Family Distinguished Professor with Washington University in St. Louis, St. Louis, MO, USA, where he is also the founding Director of the Center for Trustworthy AI in Cyber-Physical Systems and the Chair of the Electrical and Systems Engineering Department. He was a  Postdoctoral Researcher with Stanford University. Then, from 2007 to 2019, he was a member of the faculty with Carnegie Mellon University, where he was a Professor with the Department of Electrical and Computer Engineering with courtesy appointments in mechanical engineering and with Robotics Institute and a Co-Director of the Smart Infrastructure Institute. His research interests include modeling, analysis, and design of resilient cyber-physical systems with applications to smart interdependent infrastructure systems, such as energy and transportation, Internet of Things, and control of computing systems.
\end{IEEEbiographynophoto}




\end{document}